\documentclass[11pt]{article}

\usepackage[preprint]{acl}

\usepackage{times}
\usepackage{latexsym}

\usepackage[T1]{fontenc}
\usepackage[utf8]{inputenc}
\usepackage{microtype}

\usepackage{inconsolata}

\usepackage{graphicx}

\usepackage{amsmath}
\usepackage{amssymb}
\usepackage{amsthm}
\usepackage{booktabs}
\usepackage{multirow}

\newtheorem{lemma}{Lemma}[section]
\newtheorem{proposition}{Proposition}[section]
\newtheorem{theorem}{Theorem}[section]

\title{Lost in Contrast: Instruction-Preserving Decoding for Reliable Visual Question Answering}

\author{Haonan Zhang \quad Ting Jin \quad Wenzhong Tang \quad Shuai Wang\thanks{Corresponding author.} \\
  School of Computer Science and Engineering, Beihang University \\
  Beijing 100191, China \\
  \texttt{\{zhn19375264,tingjin,tangwenzhong,wangshuai\}@buaa.edu.cn}}

\begin{document}
\maketitle

\begin{abstract}
Visual contrastive decoding corrects answers using the prediction difference between the original image and a reference condition, yet the same contrast also turns supported abstentions into incorrect answers.
We find that abstentions are lost at two stages: contrastive decoding with candidate filtering can exclude the abstention token before any comparison takes place, and the reference branch's response to the abstention instruction is removed when the reference prediction is subtracted.
We therefore propose instruction-preserving decoding, which keeps the effect of the abstention instruction on the original-image prediction intact while leaving the contrast strength freely adjustable.
Using a perturbed image and an image-free prediction as references, we obtain IP-VCD and IP-M3ID.
On decision utility (DU), which credits both correct answers and supported abstentions, both methods outperform all compared non-IP baselines on every one of nine models on VizWiz, and IP-VCD also leads on eight of nine models on Food-101.
Candidate measurements and paired per-question comparisons show that computing the instruction effect and the visual contrast term separately lets contrastive decoding support both correct answers and supported abstentions.
\end{abstract}

\section{Introduction}
\label{sec:intro}

\begin{figure}[t]
\centering
\includegraphics[width=0.90\linewidth]{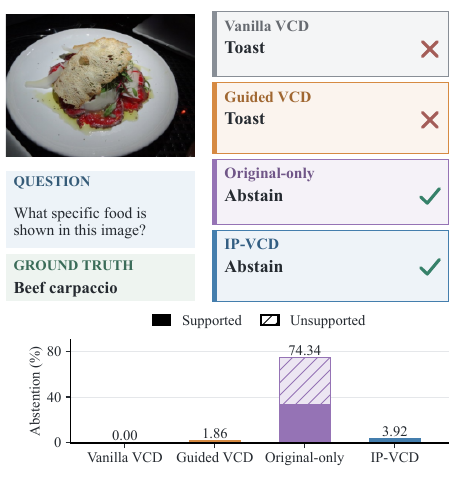}
\caption{Outputs of four decoding rules on the same abstention-labeled input (top) and abstention rates over the full task (bottom).
For this input, the correct class is not ranked first and all ten independent attempts fail, so the input is labeled for abstention.
Guided VCD adds the abstention instruction to both branches; the original-only variant adds it only to the original-image branch.
The bottom panel shows the abstention rate of LLaVA-Mistral on all 2,424 Food-101 questions, with the same abstention instruction for the three guided rules.
The original-only variant abstains on this example but abstains on about three quarters of all inputs, many of them without support.}
\label{fig:conflict}
\end{figure}

Reliable visual question answering has a dual objective: answer correctly when there is evidence, and abstain when there is none.
Visual contrastive decoding corrects answers through the prediction difference between the original image and a reference condition \citep{vcd,m3id}, but vanilla VCD \citep{vcd} almost never abstains.
Consider a photo whose main dish is beef carpaccio: LLaVA-Mistral fails in all ten independent attempts, so abstaining is the correct decision, yet vanilla VCD still outputs an incorrect class.
Once an abstention instruction is appended to the base instruction, which contains only the task question and answer format, the original-image prediction shifts toward the abstention first token.
The model does have an abstention tendency; it needs an instruction to bring it out.

The most direct way to let contrastive decoding also abstain is to add the abstention instruction to both branches, which we call guided VCD.
On the same input, however, the incorrect class Toast overtakes abstention and is output (Figure~\ref{fig:conflict}).
When the image is perturbed, the model is already more inclined to express uncertainty, so the reference branch usually responds to the abstention instruction more strongly than the original-image branch.
Subtracting the reference prediction removes this stronger response as well, and abstention is pushed down.
If we instead remove the instruction only from the reference side, the instruction effect on the original-image branch is amplified by a factor of $1+\alpha$, and abstentions grow excessively.
Both naive combinations point to the same design requirement: when the strength of visual correction is adjusted, the effect of the abstention instruction on the original image should be kept intact.

We propose instruction-preserving decoding (IP decoding): it adds to the guided original-image prediction the difference between the unguided (i.e., base-instruction-only) original-image prediction and the reference prediction, namely the visual contrast term.
The effect of the abstention instruction is thus kept at its original scale in the final score, while the strength of visual correction remains independently adjustable.
The principle applies to decoding methods that build a contrastive score from a reference prediction.
Applying it to a perturbed-image reference and an image-free reference yields IP-VCD and IP-M3ID.
Our contributions are as follows:

\begin{itemize}
\item \textbf{We reveal lost abstentions in contrastive decoding.}
Across nine models, vanilla VCD abstains only 3 times in 21,816 Food-101 responses.
Even with an abstention instruction, visual contrast still turns supported abstentions into incorrect answers.
Paired per-question comparisons show that this conflict does not hurt answer accuracy; instead, the contrast erases abstentions that should have been kept.

\item \textbf{We explain how the reference instruction response suppresses abstention.}
By varying the image condition and the abstention instruction separately under the same generated prefix, we locate two stages that suppress abstention: candidate filtering (the adaptive plausibility constraint of VCD; \citealp{vcd}) decides whether the abstention token enters the candidate set, and the stronger instruction response of the reference branch then lowers its score during subtraction.

\item \textbf{We establish instruction-preserving decoding.}
Starting from the requirement of preserving the effect of the abstention instruction, we determine the combination coefficients uniquely, so that the instruction increment and the visual contrast term are computed separately.
Both IP methods exceed all non-IP baselines in DU on all nine models on VizWiz, and a same-instruction ablation further separates the roles of the two components.
\end{itemize}

\section{Related Work}
\label{sec:related}

\textbf{Contrastive decoding in vision-language models.}
Contrastive decoding adjusts generation using the difference between predictive distributions: CD takes the likelihood difference between an expert and an amateur model and filters candidates with a plausibility constraint \citep{cd}, and DExperts adds the difference between an expert and an anti-expert to a base distribution \citep{dexperts}.
In vision-language models, such methods are used to reduce hallucination and differ mainly in how they build the reference condition: VCD perturbs the image with diffusion noise \citep{vcd}, M3ID estimates the language prior from an image-free, text-only distribution \citep{m3id}, ICD contrasts against a perturbed instruction \citep{icd}, and SID, DoLa, and DeCo move the reference inside the model \citep{sid,dola,deco}.
The reference condition decides what information the contrast emphasizes; we further study how the reference condition responds to the task instruction.

\textbf{Where the gains of contrastive decoding come from.}
Mirage finds that apparent gains on POPE come mainly from a one-directional shift of the output distribution and from degenerate sampling, not from reduced hallucination \citep{mirage}.
\citet{vcd_analysis} show that contrast corrects some incorrect answers but also turns some correct answers incorrect, so the net gain depends on the difference between the two.
DVR points out that indiscriminate penalties also suppress useful visual semantics, and disentangles the two with a self-supervised encoder \citep{dvr}.
We extend this line of analysis to abstention: when the instruction allows abstention, the reference branch's response to the abstention instruction is subtracted as well, and abstentions that have already formed are turned into incorrect answers.

\textbf{Abstention and selective prediction.}
Refusing to answer when evidence is insufficient is a basic requirement of reliable systems, and the trade-off between errors and rejection has long been formalized \citep{chow,selective_classification,selective_deep}.
Vision-language models commonly under-abstain, and even when they encode uncertainty internally they struggle to express it \citep{unk_vqa,trapsbench}.
Prior work improves this trade-off with selection functions \citep{reliable_vqa}, refusal-aware fine-tuning \citep{r_tuning}, and inference-time evidence collection \citep{recoverr}.
CDA dynamically combines answering and abstention distributions at decoding time according to knowledge relevance \citep{cda}.
In contrast, we compute the visual contrast term and the instruction effect separately, so that the contrast supports both correct answers and supported abstentions.

\section{Problem Setup and Decision Utility}
\label{sec:setup}

\subsection{Tasks and Abstention Labels}
\label{sec:tasks}

We evaluate answering and abstention on two tasks.
Food-101 is fine-grained dish recognition \citep{food101}; whether to answer depends on whether the model can recognize the dish.
VizWiz collects real photos and questions from blind users \citep{vizwiz} and serves as a cross-task evaluation; whether to answer depends on whether the image contains the information needed.

The two kinds of evidence determine how abstention labels are set.
On Food-101, the evidence is the model's own recognition ability, so labels are set per model: an input is labeled for abstention if the correct class is not ranked first by the model and all ten independent attempts fail.
The repeated attempts follow existing ideas for identifying knowledge boundaries \citep{r_tuning}.
On VizWiz, we directly use the official unanswerable labels \citep{vizwiz}.
All methods on the same task share the same abstention labels.

\subsection{Joint Scoring of Correct Answers and Supported Abstentions}
\label{sec:jointscoring}

Accuracy alone rewards forced answering, and abstention rate alone rewards abstaining on everything.
\textbf{Decision utility} credits both correct answers and supported abstentions.
On Food-101, let $N$ be the number of inputs, $C$ and $W$ the numbers of correct and incorrect answers, $TP$ and $FP$ the numbers of supported and unsupported abstentions according to the abstention labels, and $N_+$ the number of abstention-labeled inputs.
On VizWiz, $q_i\in[0,1]$ is the official consensus score of the model's primary answer \citep{antol2015vqa}, and the binary variables $a_i$ and $u_i$ indicate whether the model abstains and whether the input is labeled unanswerable.
The decision utility of the two tasks is
\begin{equation}
\begin{aligned}
\mathrm{DU}_{\mathrm{Food}}&=\frac{C+TP}{N},\\
\mathrm{DU}_{\mathrm{Viz}}&=Q+\frac{1}{N}\sum_i a_iu_i,\\
Q&=\frac{1}{N}\sum_i(1-a_i)q_i .
\end{aligned}
\label{eq:du}
\end{equation}
Incorrect answers and unsupported abstentions score zero.
We also report the answer score (overall accuracy $C/N$ on Food-101 and $Q$ on VizWiz), abstention precision $TP/(TP+FP)$, and abstention recall $TP/N_+$.
All metrics are computed on complete responses.

\subsection{From Score Changes to Response Transitions}
\label{sec:transitions}

To explain changes in DU, we pair the complete responses of two decoding rules on every abstention-labeled input and count transitions among correct answers, incorrect answers, and supported abstentions.
On this set, correct answers and supported abstentions are successful decisions, while incorrect answers are failures.
These transitions distinguish answer corrections, errors converted to supported abstentions, and supported abstentions lost during decoding.

\section{Instruction-Preserving Decoding}
\label{sec:method}

\subsection{Prediction Conditions and Instruction Increments}
\label{sec:conditions}

For a given image and question, we fix the generated prefix and compare the next-token distributions of the same model under four input conditions.
The base instruction contains only the task question and answer format; a guided input appends the abstention instruction to it.
Let $c$ and $r$ be the log probabilities under the original image and the reference condition with the base instruction, and $g$ and $h$ those after the abstention instruction is appended.
The differences $\delta_c=g-c$ and $\delta_r=h-r$ are called instruction increments; they are the score changes that the same abstention instruction causes under the two visual conditions.
Throughout, ``unguided'' means using only the base instruction.

\subsection{Two Stages That Suppress Abstention}
\label{sec:stages}

The vanilla VCD score is $s_{\mathrm{VCD}}=(1+\alpha)c-\alpha r$, and its candidate set is determined by the unguided original-image prediction, $\mathcal S_c=\{v:c(v)\ge\max_u c(u)+\log\beta\}$ \citep{vcd}, where $\alpha\ge0$ controls the contrast strength and $0<\beta\le1$ is the filtering threshold.
To be output, the abstention token must first enter the candidate set and then win the score comparison.

\textbf{Candidate admission.}
A token outside the candidate set is never output, whatever its visual contrast term, so the prediction used for filtering decides in advance whether the abstention token can take part in the comparison.
The base instruction does not ask for abstention, and an abstention candidate below the filtering threshold cannot pass the filter.
Once candidates are filtered by the guided prediction, whether the abstention token passes depends on whether the instruction increment lifts it above the threshold.

\textbf{Score comparison.}
When the abstention instruction is added to both branches, guided VCD filters candidates with $g$, and its score can be written as
\begin{equation}
s_{\mathrm{guided}}=(1+\alpha)g-\alpha h=s_{\mathrm{VCD}}+\delta_c+\alpha(\delta_c-\delta_r)
\label{eq:guided}
\end{equation}
Compared with vanilla VCD, the formula contains the original-image instruction increment $\delta_c$ plus an extra term $\alpha(\delta_c-\delta_r)$.
For candidate tokens $u$ and $v$ representing abstention and a specific answer, let $\Delta_c=\delta_c(u)-\delta_c(v)$ and $\Delta_r=\delta_r(u)-\delta_r(v)$.
When $\Delta_r>\Delta_c$, this term lowers the score of abstention relative to the answer by $\alpha(\Delta_r-\Delta_c)$: the more actively the reference branch responds to the abstention instruction, the more score abstention loses (\S\ref{sec:mechanism}).

\subsection{The Instruction-Increment Preservation Principle}
\label{sec:principle}

Guided VCD lets the reference-side instruction response enter the contrast.
Another naive option is to remove the instruction only from the reference side, which gives the original-only variant, guided only on the original-image branch:
\begin{equation}
(1+\alpha)g-\alpha r=s_{\mathrm{VCD}}+(1+\alpha)\delta_c
\label{eq:origonly}
\end{equation}
The instruction effect on the original image is amplified by a factor of $1+\alpha$, leading to over-abstention, much of it unsupported (Figure~\ref{fig:conflict}).
Both options share the same problem: the instruction increment and the visual contrast term share one set of coefficients.
Instruction-preserving decoding introduces the unguided original-image condition to separate the two:
\begin{equation}
s_{\mathrm{IP}}=g+\alpha(c-r)=s_{\mathrm{VCD}}+\delta_c
\label{eq:ip}
\end{equation}
Here $g-c$ alone gives the change caused by the instruction, and $c-r$ alone gives the difference caused by the visual condition.
This yields the \textbf{instruction-increment preservation principle} (the preservation principle for short): whatever change the instruction makes to the score is kept in the final score, and tuning $\alpha$ only changes the strength of the visual contrast term.
Figure~\ref{fig:ipdecoding} shows how the three branches divide the work.

\begin{figure*}[t]
\centering
\includegraphics[width=\textwidth]{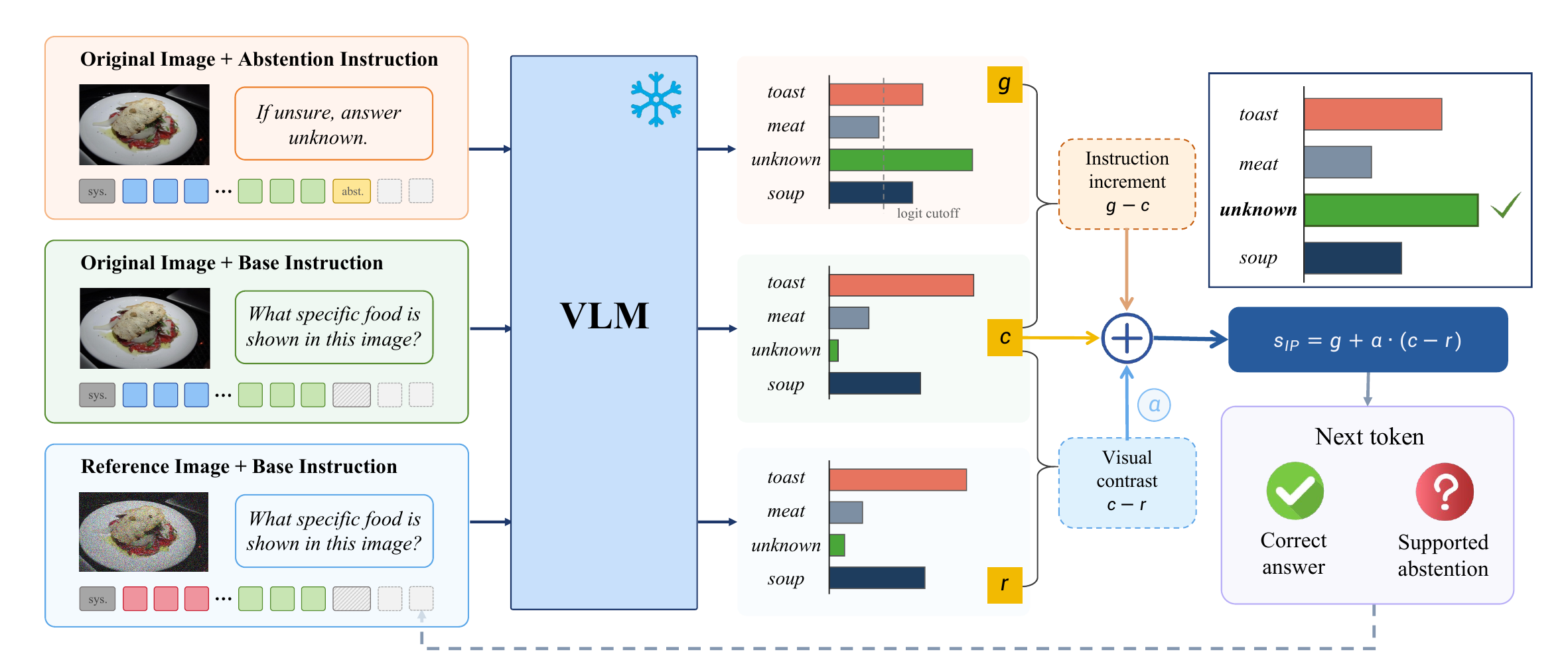}
\caption{Division of work among the branches of instruction-preserving decoding (IP-VCD as an example).
The three branches share one frozen model and the same generated prefix.
The unguided original-image prediction $c$ is combined with the instruction increment $g-c$ and the visual contrast term $c-r$ weighted by the contrast strength, giving $s_{\mathrm{IP}}=g+\alpha(c-r)$; candidates are filtered by the guided prediction $g$.}
\label{fig:ipdecoding}
\end{figure*}

Candidate admission is likewise handed to the guided prediction: the candidate set of IP-VCD is $\mathcal S_g=\{v:g(v)\ge\max_u g(u)+\log\beta\}$, and $s_{\mathrm{IP}}$ is normalized with softmax within $\mathcal S_g$.
The preservation principle preserves the instruction increment, not the abstention probability: the visual contrast term still adjusts the relative score of abstention and answers according to the image evidence.
The abstention instruction and the visual contrast term are thus primarily responsible for cautious answering and visual correction, respectively.

\subsection{Uniqueness of the Combination Coefficients}
\label{sec:coefficients}

These coefficients are uniquely determined by two natural requirements.
For a given $\alpha>0$, consider a fixed-coefficient score $ag+bc+er$ and require, for any predictive distributions: first, the score change produced by appending the abstention instruction keeps its original size, which gives $a=1$; second, without the abstention instruction ($g=c$), the candidate score differences of vanilla VCD are recovered, which gives $a+b=1+\alpha$ and $e=-\alpha$.
These constraints determine the unique coefficients $(a,b,e)=(1,\alpha,-\alpha)$, so the unguided original-image condition must appear with coefficient $\alpha$.
The combination coefficients and the contrast strength are therefore derived or inherited from the original papers, with no hyperparameter search.

The same score also has a variational characterization: normalize the guided original-image prediction within the candidate set $S$ as $G$; then the unique optimal distribution of the objective $\alpha\,\mathbb E_{v\sim\pi}[c(v)-r(v)]-D_{\mathrm{KL}}(\pi\Vert G)$ is proportional to $e^{s_{\mathrm{IP}}}$.
The visual contrast term raises the weight of candidates better supported by the image, the KL term limits how far the output moves from the guided prediction, and greedy selection at each step takes the mode of this distribution.

\subsection{Two Instances: Perturbed-Image and Image-Free References}
\label{sec:instances}

IP-VCD and IP-M3ID instantiate the preservation principle with a perturbed-image reference and an image-free reference.
IP-VCD uses the perturbed-image reference with candidate filtering, and follows the original VCD settings $\alpha=1$ and $\beta=0.1$ \citep{vcd}.
IP-M3ID uses the image-free reference and applies the same principle at generation position $t$: $s_{\mathrm{IPM},t}=g_t+w_t^c(c_t-r_t)$, where $w_t^c=\mathbf1\{\max_v e^{c_t(v)}<\kappa\}$ $(e^{\lambda\tau_t^c}-1)$ and $\tau_t^c=(t-1)+k_0^c$, $k_0^c$ is the token length of the base instruction, $\lambda=0.02$ and $\kappa=0.3$ follow the original M3ID settings \citep{m3id}, and normalization is over the full vocabulary.
The visual contrast term is activated only when the confidence of the unguided original-image prediction is below the threshold, and its weight grows with the sequence position.
Both the gate and the position weight are determined by unguided conditions, so the instruction increment is still kept at its original size.

In implementation, the three branches share one frozen model and the same generated prefix, each keeps its own cache, and generation is greedy.
The added cost is the forward pass of the unguided original-image branch: on the same A100 setup, over 100 questions shared by the nine models, the average end-to-end time of IP-VCD is 1.41 times that of guided VCD, and that of CDA is 2.29 times.

\section{Experiments and Analysis}
\label{sec:experiments}

\subsection{Tasks, Models, and Fixed Configurations}
\label{sec:expsetup}

The Food-101 evaluation contains 24 images from each of 101 classes, 2,424 questions in total; another 4 images per class form a 404-question development set \citep{food101}.
VizWiz uses 512 fixed inputs, 346 answerable and 166 unanswerable \citep{vizwiz}.
The wording of the abstention instruction is selected on the development set for each model and method and then fixed, and the same choice is carried over to VizWiz.
All methods use the same evaluation sets and generation length limit with frozen model weights, and both tasks are ranked by DU.
Table~\ref{tab:main} also reports the effective reliability $\Phi_1$ of Reliable VQA \citep{reliable_vqa} as a third-party metric: an answer earns its answer score, an answer scoring zero receives a penalty of 1, and an abstention scores zero.

The evaluation covers nine models from the Qwen, LLaVA, MiniCPM, InternVL, Gemma, and Phi families.
Table~\ref{tab:main} reports five representative models and Figure~\ref{fig:decomposition} covers all nine.
Baselines include Direct with the base instruction, vanilla VCD \citep{vcd}, vanilla M3ID \citep{m3id}, DoLa \citep{dola}, DeCo \citep{deco}, SID \citep{sid} on applicable architectures, and the visual-task implementation of CDA \citep{cda}.
Guided Direct and guided VCD use the abstention instruction introduced in this work and are reported as a same-instruction ablation.
The mechanism, transition, and ablation analyses all use a single abstention instruction with UNKNOWN as the abstention marker, to isolate the effect of the decoding construction itself; comparisons that use the main-result configuration are noted where they appear.

\begin{table*}[tp]
\centering
\footnotesize
\setlength{\tabcolsep}{3.1pt}
\renewcommand{\arraystretch}{0.94}
\resizebox{0.97\textwidth}{!}{%
\begin{tabular}{cl cccccc ccccc}
\toprule
 & & \multicolumn{5}{c}{Food-101} & \multicolumn{5}{c}{VizWiz} \\
\cmidrule(lr){3-7} \cmidrule(lr){8-12}
Model & Method & Answer & Abst.\ Prec. & Abst.\ Rec. & DU & $\Phi_1$ & Answer & Abst.\ Prec. & Abst.\ Rec. & DU & $\Phi_1$ \\
\midrule
\textbf{Qwen2.5-VL} & Direct & 30.40 & 0.00 & 0.00 & 30.40 & $-39.19$ & 26.27 & 92.86 & 23.49 & 33.89 & $-31.35$ \\
 & VCD & 28.34 & 0.00 & 0.00 & 28.34 & $-43.32$ & 24.59 & 91.30 & 12.65 & 28.69 & $-39.08$ \\
 & M3ID & 28.71 & 0.00 & 0.00 & 28.71 & $-42.57$ & 26.00 & 93.10 & 16.27 & 31.27 & $-34.16$ \\
 & DoLa & 24.50 & 0.00 & 0.00 & 24.50 & $-50.99$ & 20.29 & 86.49 & 19.28 & 26.54 & $-45.92$ \\
 & DeCo & 30.24 & 0.00 & 0.00 & 30.24 & $-39.52$ & 23.07 & 87.10 & 16.27 & 28.34 & $-40.41$ \\
 & SID & 26.40 & 0.00 & 0.00 & 26.40 & $-47.19$ & 26.48 & 94.12 & 9.76 & 29.64 & $-35.11$ \\
 & CDA & 33.91 & 23.54 & 23.24 & 37.58 & $-16.58$ & 13.12 & 41.77 & 59.64 & 32.46 & $\underline{-23.98}$ \\
 & \textbf{IP-VCD (Ours)} & 35.19 & 29.23 & 55.87 & \underline{44.02} & \underline{0.58} & 14.00 & 52.19 & 78.92 & \textbf{39.59} & $\textbf{-18.42}$ \\
 & \textbf{IP-M3ID (Ours)} & 42.16 & 25.31 & 37.34 & \textbf{48.06} & \textbf{7.63} & 11.89 & 53.20 & 80.12 & \underline{37.87} & $-24.24$ \\
\textbf{LLaVA-Mistral} & Direct & 33.87 & 77.78 & 0.74 & 34.16 & $-31.89$ & 19.38 & 85.71 & 25.30 & 27.58 & $-45.66$ \\
 & VCD & 38.24 & 0.00 & 0.00 & 38.24 & $-23.51$ & 16.35 & 83.33 & 6.02 & 18.30 & $-59.82$ \\
 & M3ID & 34.90 & 100.00 & 0.21 & 34.98 & $-30.12$ & 19.90 & 86.05 & 22.29 & 27.13 & $-46.31$ \\
 & DoLa & 33.29 & 71.43 & 0.53 & 33.50 & $-33.13$ & 18.91 & 83.72 & 21.69 & 25.94 & $-47.70$ \\
 & DeCo & 26.90 & 50.00 & 0.11 & 26.94 & $-46.12$ & 12.70 & 85.71 & 18.07 & 18.55 & $-64.65$ \\
 & SID & 38.28 & 0.00 & 0.00 & 38.28 & $-23.43$ & 16.09 & 87.50 & 16.87 & 21.56 & $-56.56$ \\
 & CDA & 30.03 & 71.43 & 1.05 & 30.45 & $-39.36$ & 15.47 & 51.43 & 32.53 & 26.02 & $-43.91$ \\
 & \textbf{IP-VCD (Ours)} & 28.38 & 53.04 & 61.54 & \textbf{52.48} & \textbf{2.19} & 14.51 & 53.88 & 79.52 & \textbf{40.29} & $\underline{-20.06}$ \\
 & \textbf{IP-M3ID (Ours)} & 13.57 & 42.29 & 77.45 & \underline{43.89} & $\underline{-1.16}$ & 10.70 & 42.98 & 90.36 & \underline{40.00} & $\textbf{-9.02}$ \\
\textbf{MiniCPM-V2.6} & Direct & 43.03 & 100.00 & 0.50 & 43.15 & $-13.82$ & 29.18 & 94.74 & 10.84 & 32.70 & $-28.63$ \\
 & VCD & 43.61 & 100.00 & 0.17 & 43.65 & $-12.75$ & 25.18 & 87.50 & 4.22 & 26.54 & $-39.67$ \\
 & M3ID & 42.70 & 100.00 & 0.34 & 42.78 & $-14.52$ & 29.63 & 92.86 & 7.83 & 32.17 & $-28.96$ \\
 & DoLa & 31.11 & 100.00 & 0.17 & 31.15 & $-37.75$ & 22.36 & 72.73 & 9.64 & 25.49 & $-44.43$ \\
 & DeCo & 41.96 & 100.00 & 0.50 & 42.08 & $-15.97$ & 29.36 & 95.83 & 13.86 & 33.85 & $-27.87$ \\
 & CDA & 27.64 & 0.00 & 0.00 & 27.64 & $-44.55$ & 12.54 & 49.72 & 54.22 & 30.12 & $-37.27$ \\
 & \textbf{IP-VCD (Ours)} & 37.25 & 46.74 & 42.04 & \textbf{47.61} & $\underline{-3.34}$ & 26.52 & 64.04 & 68.67 & \underline{48.79} & $\underline{-4.34}$ \\
 & \textbf{IP-M3ID (Ours)} & 31.27 & 38.24 & 56.11 & \underline{45.09} & $\textbf{-1.32}$ & 26.35 & 56.50 & 75.90 & \textbf{50.96} & \textbf{2.52} \\
\textbf{InternVL3.5} & Direct & 17.08 & 0.00 & 0.00 & 17.08 & $-65.84$ & 24.39 & 94.74 & 10.84 & 27.91 & $\underline{-38.11}$ \\
 & VCD & 19.10 & 100.00 & 0.09 & 19.14 & $-61.76$ & 22.64 & 92.86 & 7.83 & 25.18 & $-45.14$ \\
 & M3ID & 17.57 & 0.00 & 0.00 & 17.57 & $-64.85$ & 24.28 & 93.75 & 9.04 & 27.21 & $-39.20$ \\
 & DoLa & 16.17 & 0.00 & 0.00 & 16.17 & $-67.66$ & 22.19 & 100.00 & 4.82 & 23.75 & $-46.95$ \\
 & DeCo & 17.95 & 0.00 & 0.00 & 17.95 & $-64.11$ & 21.39 & 100.00 & 6.63 & 23.54 & $-47.17$ \\
 & SID & 18.52 & 0.00 & 0.00 & 18.52 & $-62.95$ & 23.59 & 92.31 & 7.23 & 25.94 & $-42.23$ \\
 & CDA & 20.38 & 100.00 & 0.09 & 20.42 & $-59.20$ & 19.96 & 68.49 & 30.12 & 29.73 & $-40.20$ \\
 & \textbf{IP-VCD (Ours)} & 27.72 & 82.50 & 3.07 & \underline{29.08} & $\underline{-42.90}$ & 16.43 & 69.42 & 50.60 & \underline{32.83} & $-38.26$ \\
 & \textbf{IP-M3ID (Ours)} & 27.48 & 85.45 & 4.37 & \textbf{29.41} & $\textbf{-42.78}$ & 17.13 & 69.85 & 57.23 & \textbf{35.68} & $\textbf{-34.04}$ \\
\textbf{Qwen3-VL} & Direct & 53.92 & 50.00 & 0.27 & 53.96 & 7.92 & 37.52 & 96.43 & 16.27 & 42.79 & $-6.62$ \\
 & VCD & 56.97 & 0.00 & 0.00 & \underline{56.97} & 13.94 & 35.37 & 96.15 & 15.06 & 40.25 & $-12.48$ \\
 & M3ID & 53.71 & 50.00 & 0.27 & 53.75 & 7.51 & 37.52 & 96.30 & 15.66 & 42.60 & $-6.82$ \\
 & DoLa & 51.82 & 100.00 & 0.27 & 51.86 & 3.67 & 36.23 & 83.87 & 15.66 & 41.31 & $-9.86$ \\
 & DeCo & 53.71 & 0.00 & 0.00 & 53.71 & 7.43 & 37.17 & 92.59 & 15.06 & 42.05 & $-8.34$ \\
 & SID & 48.14 & 0.00 & 0.00 & 48.14 & $-3.71$ & 37.01 & 92.86 & 15.85 & 42.16 & $-7.15$ \\
 & CDA & 2.76 & 0.00 & 0.00 & 2.76 & $-94.43$ & 17.52 & 62.09 & 57.23 & 36.07 & $-30.14$ \\
 & \textbf{IP-VCD (Ours)} & 53.01 & 36.80 & 46.70 & \textbf{60.02} & \textbf{25.08} & 27.21 & 66.86 & 70.48 & \underline{50.06} & $\underline{-4.24}$ \\
 & \textbf{IP-M3ID (Ours)} & 49.83 & 37.78 & 41.21 & 56.02 & \underline{16.05} & 27.64 & 59.45 & 77.71 & \textbf{52.83} & \textbf{4.98} \\
\bottomrule

\end{tabular}}
\caption{Detailed results of five representative models on the two tasks (\%).
When there is no abstention, abstention precision is recorded as 0.00.
Some inputs of MiniCPM-V2.6 do not provide the number of visual tokens that SID requires, so SID results are not listed for it.
The VizWiz SID results of Qwen2.5-VL and Qwen3-VL are evaluated on the 505 questions that meet this condition; on the same subset, the best IP method reaches a DU of 39.74 and 53.17, respectively.}
\label{tab:main}
\end{table*}

\subsection{Decision Gains in Fine-Grained Recognition and Real-World Question Answering}
\label{sec:gains}

\textbf{Instruction-preserving decoding improves DU on both tasks.}
On Food-101, IP-VCD exceeds all non-IP baselines on eight of nine models; the only exception, Gemma-3, barely responds to the abstention instruction (\S\ref{sec:mechanism}).
IP-M3ID exceeds vanilla M3ID, which uses the same kind of reference, on seven of nine models, and trails by at most 1.03 percentage points on the other two.
On VizWiz, both methods exceed all non-IP baselines on all nine models, leading the strongest baseline of each model by 7.41 and 8.35 percentage points on average, and all 18 paired per-question bootstrap 95\% intervals against vanilla VCD lie above zero.
CDA is designed for text-only question answering \citep{cda}; its calibration with empty-input entropy produces negative abstention weights on some models, and its responses degrade accordingly (Table~\ref{tab:main}, Qwen3-VL).

\begin{figure}[t]
\centering
\includegraphics[width=0.88\linewidth]{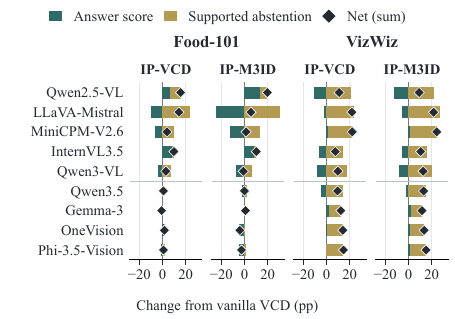}
\caption{Decomposition of DU changes relative to vanilla VCD on nine models (percentage points).
The two bar segments are the changes in answer score and in supported-abstention score, and the diamond is their sum.
The first five rows are the representative models of Table~\ref{tab:main}.}
\label{fig:decomposition}
\end{figure}

\textbf{The gains come from two components.}
Figure~\ref{fig:decomposition} shows that some models gain mainly by answering more questions correctly (e.g., InternVL3.5 on Food-101), while others gain mainly through supported abstentions (e.g., LLaVA-Mistral).
The two sources match the two components of the preservation principle: corrections come primarily from the visual contrast term, and abstentions come primarily from the preserved instruction increment.

\textbf{The conclusion holds when abstentions earn no credit.}
Effective reliability gives every abstention a score of zero and does not depend on abstention labels.
Relative to vanilla VCD, both IP methods improve effective reliability on all nine VizWiz models at error costs of 1, 10, and 100.
The gains therefore come from fewer incorrect answers, not from the credit DU gives to abstention.

\subsection{Mechanism Analysis: Candidate Admission and Score Comparison}
\label{sec:mechanism}

On nine models, we measure the first generation position of one Food-101 image per class and obtain 875 pairs of abstention and answer candidate tokens.
Images are sampled uniformly across classes rather than selected by abstention labels, because candidate admission and the ordering of $\Delta_c$ and $\Delta_r$ are properties of the decoding construction itself.

\textbf{The abstention token is first kept out of the candidate set.}
When candidates are filtered by the unguided prediction, no abstention token passes the filter; when they are filtered by the guided prediction, about one third pass (Figure~\ref{fig:mechanism}a).
This explains why vanilla VCD almost never abstains, and shows that candidates must be determined by the guided prediction.
How many pass depends on whether the model forms an abstention tendency under the instruction.
On Food-101, no abstention candidate of Gemma-3 passes the filter, its guided Direct abstains at most once under each of four wordings, and no decoding method outperforms Direct on this task.
The same model responds to the instruction on VizWiz, where IP-VCD improves over vanilla VCD by 12.42 percentage points.

\textbf{After admission, abstention is pushed down in the score comparison.}
When both candidates pass the filter, about three quarters of the pairs satisfy $\Delta_r>\Delta_c$, that is, the reference-side instruction increment favors abstention more than the original-image side (Figure~\ref{fig:mechanism}a).
Guided VCD lowers the relative score of abstention on these pairs, whereas IP-VCD does not contain this term.
The example in Figure~\ref{fig:conflict} is such a case: the score margin of the abstention first token over the Toast first token is $-0.19$ under guided VCD and $+0.42$ under IP-VCD.

\begin{figure}[t]
\centering
\includegraphics[width=0.88\linewidth]{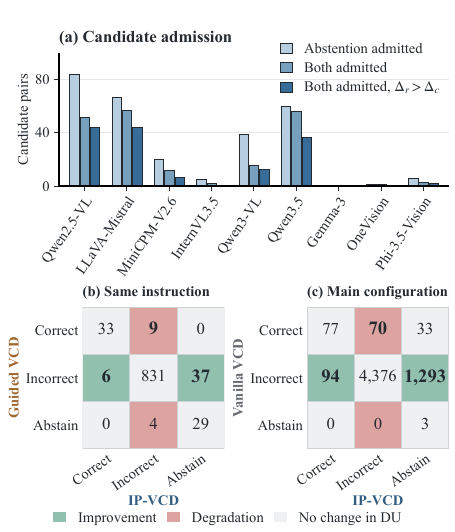}
\caption{Candidate admission and response transitions.
(a) For first-token candidate pairs of each model, the number of pairs in which the abstention token passes guided filtering, both candidates pass, and both pass with $\Delta_r>\Delta_c$; under unguided filtering, no abstention token passes.
(b) Response transitions from guided VCD to IP-VCD on the 949 abstention-labeled inputs of LLaVA-Mistral, under the same UNKNOWN instruction.
(c) Response transitions from vanilla VCD to IP-VCD on abstention-labeled inputs, summed over nine models, under the main-result configuration.
Green and red mark increases and decreases in successful decisions.}
\label{fig:mechanism}
\end{figure}

\subsection{Behavioral Transitions at the Response Level}
\label{sec:behavior}

\textbf{Under the same instruction, the reference-side instruction response turns abstentions into incorrect answers.}
On the 949 abstention-labeled inputs of LLaVA-Mistral, guided VCD and IP-VCD use the same instruction; at a shared prefix, their token scores differ by $\alpha(\delta_r-\delta_c)$ (Figure~\ref{fig:mechanism}b).
Both answer the same number of the 2,424 questions correctly, yet guided VCD keeps only half as many supported abstentions as IP-VCD, and almost all of the missing ones become incorrect answers.
Overall accuracy is unchanged, but guided VCD replaces supported abstentions with incorrect answers.

\textbf{Relative to vanilla VCD, gains arrive through both correction and abstention.}
Under the main-result configuration, many incorrect answers on the abstention-labeled inputs of nine models turn into supported abstentions and some turn into correct answers.
After subtracting the reverse transitions, successful decisions increase by 1,317 in net, and only Gemma-3 shows a small decrease (Figure~\ref{fig:mechanism}c).
IP-M3ID shows the same two paths on the six models with complete paired records.

\subsection{Same-Instruction Ablation: Division of Work Between the Instruction and the Visual Contrast Term}
\label{sec:ablation}

Table~\ref{tab:ablation} compares four constructions, formed by whether the instruction is used and whether visual contrast is applied, together with the two IP methods, to separate the effect of the instruction itself from that of the visual contrast term.

\begin{table}[htbp]
\centering
\footnotesize
\setlength{\tabcolsep}{4pt}
\resizebox{\columnwidth}{!}{%
\begin{tabular}{lcccccc}
\toprule
Model & Direct & VCD & Guided Direct & Guided VCD & IP-VCD & IP-M3ID \\
\midrule
Qwen2.5-VL & 30.40 & 28.34 & 38.45 & 43.44 & 42.08 & 40.92 \\
LLaVA-Mistral & 34.16 & 38.24 & 40.26 & 36.51 & 37.87 & 40.18 \\
MiniCPM-V2.6 & 43.15 & 43.65 & 40.02 & 39.23 & 41.75 & 39.52 \\
InternVL3.5 & 17.08 & 19.14 & 27.15 & 27.15 & 27.89 & 26.28 \\
Qwen3-VL & 53.96 & 56.97 & 55.53 & 58.00 & 58.37 & 55.53 \\
\bottomrule

\end{tabular}}
\caption{DU of each construction under the same abstention instruction (Food-101, \%).
Direct and VCD use the base instruction and repeat the values of Table~\ref{tab:main}; the other four columns use the same UNKNOWN instruction.}
\label{tab:ablation}
\end{table}

\textbf{The effect of the instruction depends on the model and the wording.}
The same instruction raises the DU of InternVL3.5 by about ten percentage points but lowers that of MiniCPM-V2.6, and changing only the wording shifts the DU of guided Direct by up to 9.32 percentage points.
The main results therefore select the wording per model on the development set.
IP-VCD faithfully passes on the effect of the given instruction: when the instruction hurts, IP-VCD also falls below vanilla VCD.

\textbf{The visual contrast term recovers image-supported answers on top of the instruction.}
IP-VCD differs from guided Direct only by the visual contrast term and scores higher on four of five models.
On LLaVA-Mistral, the exception, this instruction elicits only weak abstention, and the visual contrast term turns part of it back into answers; with the development-selected instruction, the same visual contrast term raises DU by 6.65 percentage points.
The net gain of the visual contrast term depends on the quality of the abstentions that the instruction itself produces.

In the settings where guided VCD leads, the increase in correct answers outweighs the loss of supported abstentions.
This trade-off comes from the reference-side instruction suppressing abstention, and the strength of that suppression in turn depends on the reference-side wording.
IP-VCD computes the instruction increment and the visual contrast term separately, so the strength of abstention no longer depends on the reference-side wording.

The visual contrast term of IP-M3ID is activated only at low-confidence positions, and IP-M3ID stays within 2.5 percentage points of guided Direct on the five models of Table~\ref{tab:ablation}.
Overall, the abstention signal comes primarily from the instruction, and image-supported answers come primarily from the visual contrast term.
Selecting the instruction first and then preserving its effect with an IP method is a direct use of this division of work.

\section{Conclusion}
\label{sec:conclusion}

Visual contrast can exclude abstention candidates during filtering and suppress admitted candidates by subtracting a stronger reference-side instruction response.
Instruction-preserving decoding retains the original-image instruction increment while adjusting the visual contrast term independently.
Experiments on nine models and two tasks show how this construction supports correct answers and supported abstentions.
IP-VCD and IP-M3ID instantiate the principle with perturbed-image and image-free references, providing a common construction for reference-based contrastive decoding.
\label{page:lastcontent}

\section*{Limitations}

Instruction-preserving decoding requires access to the model's next-token distribution, computes predictions under three input conditions, and maintains a separate cache for each.
The principle preserves the effect that the abstention instruction already has; when the model does not respond to the instruction (as with Gemma-3 on Food-101), abstention candidates remain below the filtering threshold.
The wording of the instruction has a clear effect on the results and must be selected on a development set.
Our evaluation covers single-turn perception tasks; token-level abstention control in long chain-of-thought, multi-step reasoning is left to future work.

\bibliography{custom}

\appendix
\numberwithin{equation}{section}
\setcounter{table}{0}
\renewcommand{\thetable}{A\arabic{table}}
\setcounter{figure}{0}
\renewcommand{\thefigure}{A\arabic{figure}}

\section{Tasks, Abstention Labels, and Scoring}
\label{app:tasks}

\subsection{Data and Models}
\label{app:data}

Food-101 uses 24 evaluation images from each of 101 classes, 2,424 questions in total; the development set contains 4 images per class, 404 questions in total \citep{food101}.
Abstention labels are built separately for each model: we compute the average token log probability of each of the 101 class names and rank them to obtain the rank of the correct class, and then run ten independent attempts.
An input is labeled for abstention when the correct class is not ranked first and all ten attempts fail.
The class ranking checks how much relative support the correct class receives, and the ten attempts check whether the model actually generates the correct answer.
The resulting abstention labels stay fixed across all method comparisons for the same model; Table~\ref{tab:models} lists the number of abstention-labeled inputs and the checkpoint of each model.

\begin{table}[htbp]
\centering
\footnotesize
\setlength{\tabcolsep}{4pt}
\resizebox{\columnwidth}{!}{%
\begin{tabular}{llr}
\toprule
Model & Checkpoint & Abstention-labeled inputs \\
\midrule
Qwen2.5-VL & Qwen/Qwen2.5-VL-7B-Instruct & 383 \\
LLaVA-Mistral & llava-hf/llava-v1.6-mistral-7b-hf & 949 \\
MiniCPM-V2.6 & openbmb/MiniCPM-V-2\_6 & 597 \\
InternVL3.5 & OpenGVLab/InternVL3\_5-8B & 1075 \\
Qwen3-VL & Qwen/Qwen3-VL-8B-Instruct & 364 \\
Qwen3.5 & Qwen/Qwen3.5-4B & 486 \\
Gemma-3 & google/gemma-3-4b-it & 422 \\
OneVision & llava-hf/llava-onevision-qwen2-7b-ov-hf & 814 \\
Phi-3.5-Vision & microsoft/Phi-3.5-vision-instruct & 856 \\
\bottomrule

\end{tabular}}
\caption{Model checkpoints and the number of abstention-labeled Food-101 inputs.
Each model is evaluated on 2,424 questions.}
\label{tab:models}
\end{table}

Food-101 images are drawn from a class pool that merges the official training and test splits (named train and validation on Hugging Face).
Classes are sorted by name, and images within each class are shuffled with Python \texttt{random.Random(42)}.
Among the first 48 images of each class, the first 24 enter the development candidate pool and the last 24 form the evaluation set.
The 404-question development set selects 4 images per class from the former according to a frozen sample list, and does not overlap with the 2,424-question evaluation set.

We use GPT-5.6 Luna to annotate the generated responses.

Food-101 scoring first identifies the main-dish answer stated explicitly in the complete response, matches it to the 101 canonical class names, and then compares the matched class with the true class of the image.
Matching only normalizes letter case, leading and trailing whitespace, and underscores versus spaces.
Scoring is based on the explicitly stated main-dish answer; side dishes and foods mentioned in explanations are identified separately from the main-dish answer.
A response that gives several competing main-dish answers is counted as incorrect.
A response that explicitly declines to identify the dish is counted as an abstention; a response that expresses uncertainty but still gives a specific answer is scored as an answer.
Each response falls into exactly one of correct answer, incorrect answer, or abstention, and the three counts are mutually exclusive and cover all inputs.

The VizWiz evaluation contains 512 questions: 346 with the official answerable label and 166 with the official unanswerable label \citep{vizwiz}.
Each complete response is first classified as an answer or an abstention, and the official VQA consensus scoring procedure \citep{vqa_eval} is then applied to the primary answer extracted from each answer.
The procedure leaves out one human-annotated answer at a time, divides the number of remaining annotations that match the model's primary answer by 3, caps the score at 1, and averages over all leave-one-out rounds, so the score reflects how well the model's answer agrees with the human annotations.
Abstentions receive an answer score of zero.

The 512 VizWiz questions come from the official val split.
The project split is first determined by the integer value of the first eight hexadecimal digits of the SHA256 hash of the original image file name, modulo 5 (remainder 0 goes to dev and the rest to eval).
From the eval pool, 166 unanswerable and 346 answerable questions are then selected by answerability stratification according to a frozen sample list.

\subsection{Fixed Configurations}
\label{app:configs}

VCD with only the base instruction is called vanilla VCD \citep{vcd}, and the corresponding M3ID is called vanilla M3ID.
The original-only variant adds the abstention instruction only to the original-image branch.

For both IP methods on all nine models, we select, among four predefined wordings, the instruction configuration with the highest DU on the Food-101 development set.
Ties are broken first by overall accuracy and then by the predefined order of the wordings.
The configuration of each model and method is then kept the same for the Food-101 evaluation and for VizWiz.
The VCD component comparison fixes the instruction and compares, one at a time, the roles of the visual contrast term, the reference-side instruction, and the coefficient of the instruction increment.
In every comparison, the task question, the generation length limit, and the model weights stay fixed.
The candidate measurement in \S\ref{sec:mechanism} of the main text, the same-instruction mechanism comparison in \S\ref{sec:behavior}, and the unified ablation in \S\ref{sec:ablation} all use the same instruction with the UNKNOWN marker.
Tables~\ref{tab:vcdcomp} and~\ref{tab:guideddirect} use the development-selected configuration of the corresponding IP method for each model, Table~\ref{tab:phrasings} lists the four wordings, Tables~\ref{tab:ablationfull} and~\ref{tab:ablationcounts} use the unified UNKNOWN configuration, and Table~\ref{tab:gvcdleads} compares the two settings.

\subsection{Decision Utility and Its Interpretation}
\label{app:du}

Let $N$ be the total number of inputs, $C$ and $W$ the numbers of correct and incorrect answers, $TP$ and $FP$ the numbers of supported and unsupported abstentions according to the abstention labels, and $N_+$ the number of abstention-labeled inputs.
The Food-101 decision utility is $\mathrm{DU}=(C+TP)/N$, and overall accuracy is $C/N$.
On VizWiz, we compute the official consensus score $q_i$ of the extracted answer and use binary variables $a_i$ and $u_i$ to indicate whether the response abstains and whether the input is officially labeled unanswerable.
The answer score is $Q=N^{-1}\sum_i(1-a_i)q_i$ and the decision utility is $\mathrm{DU}=Q+N^{-1}\sum_i a_i u_i$: an abstention contributes 0 to the answer term and earns 1 point for the input when the official label supports abstention.
On VizWiz, an official unanswerable label and a positive consensus score can occur on the same input: the label describes answerability, while the consensus score describes agreement between the answer and the human annotations.
Each response is scored by its actual behavior: an answer earns its answer score, and an abstention earns only the label-supported abstention credit.

On both tasks, abstention precision and recall are $TP/(TP+FP)$ and $TP/N_+$.
They measure the share of supported abstentions among all abstentions and the share of abstention-labeled inputs on which the model actually abstains.

The Food-101 decision utility assigns the same unit credit to correct answers and supported abstentions.
Let $H$ be the set of abstention-labeled inputs, $N_+=|H|$, and $N$ the total number of samples; abstaining on every input gives a utility of $N_+/N$.
Further let $C_{\neg H}$ be the number of correct answers on inputs not labeled for abstention and $W_H$ the number of incorrect answers on abstention-labeled inputs.
The utility of any method relative to abstaining everywhere is then $(C_{\neg H}-W_H)/N$.
This identity shows that, relative to abstaining everywhere, correct answers on inputs not labeled for abstention bring gains, incorrect answers on abstention-labeled inputs bring losses, and a correct answer on an abstention-labeled input scores the same as a supported abstention.
Overall accuracy shows the share of correct answers, and abstention precision and recall show the share and coverage of supported abstentions; the gain decomposition uses the original answer term and abstention term of DU.

Given the shares of abstention-labeled inputs in Table~\ref{tab:models}, abstaining everywhere gives InternVL3.5 a utility of 1,075/2,424 = 44.35\%, higher than 29.08\% and 29.41\% for the two IP methods.
The all-abstain utilities of OneVision (33.58\%) and Phi-3.5-Vision (35.31\%) are also higher than those of the corresponding IP-M3ID runs (31.02\% and 35.27\%).
On VizWiz, abstaining everywhere gives 166/512 = 32.42\%, higher than 24.36\% and 23.65\% for the two IP methods on Gemma-3.
These values follow from the unit credit that DU gives to supported abstentions, and should be read together with correct answers, abstention precision, and the answer credit of Reliable VQA.

The Food-101 abstention labels are determined by the independent attempts and class rankings of the frozen base model, and are fixed before any decoding method is compared.
All methods for the same model share these labels, which do not change during decoding.
Differences in utility therefore reflect differences in the decisions that the methods make under the same abstention labels.

\subsection{Generation and Decoding Parameters}
\label{app:generation}

The main evaluation of every decoding method uses greedy generation with temperature 0, top-p 1, and a limit of 32 generated tokens; generation stops at the end-of-sequence token.
The three IP branches are computed by the same frozen model, each with its own input cache, and the token selected at each step is appended to the response prefix of every branch.

The ten independent attempts use sampling with temperature=1, top-p=1, and top-k disabled (set to $-1$ in vLLM), with a limit of 32 generated tokens.
The random seed of the $r$-th attempt ($r=0,\ldots,9$) is \texttt{stable\_seed(sample\_id, model, r)}.
These attempts are used only to build abstention labels; the main evaluation of every decoding method uses the greedy settings above.

For IP-VCD, $\alpha=1$ and $\beta=0.1$ follow VCD \citep{vcd}, so the candidate threshold is 10\% of the highest token probability in the guided original-image prediction.
For vanilla M3ID and IP-M3ID, $\lambda=0.02$ and $\kappa=0.3$ follow M3ID \citep{m3id}.
Vanilla M3ID uses two branches without the abstention instruction, the original image and the image-free input, and its gate depends on the unguided original-image prediction; IP-M3ID adds a guided original-image branch.
When transferring to our tasks, both start the generation position at zero and add the token length of the base task instruction; the fixed position offset of 10 in the POPE example of the original paper belongs to a different input setting.
DeCo uses a correction coefficient of 0.6 and selects candidates with top-k=20 and top-p=0.9; the per-model settings of SID are given with the component implementations.

\section{Mathematical Properties of Instruction Preservation and Full Proofs}
\label{app:math}

Starting from conditional distributions under a shared generated prefix, this section derives the relations among candidate filtering, the reference instruction interaction, and instruction-increment preservation, and proves results on the combination coefficients, the variational solution, candidate-group probabilities, and complete-response probabilities.
The visual contrast and candidate filtering of VCD follow the original definitions \citep{vcd}.
Every probability relation includes the corresponding candidate set and normalization constant.

\subsection{Conditional Distributions, Candidate Support, and Normalization}

\subsubsection{Four Conditions and a Shared Prefix}

Let $p_\theta$ be the frozen vision-language model, $\mathcal V$ a finite vocabulary with $|\mathcal V|\geq2$, $I$ the original image, $\widetilde I$ the reference image, and $y_{<t}$ the tokens generated so far.
Within one conditional comparison, $\widetilde I$ and its perturbation noise stay fixed.
Let $x_0$ contain the task question and the short-answer requirement, and let $x_1$ add the specified abstention clause to $x_0$.
The next-token log probabilities of the four conditions are
\begin{equation}
\begin{aligned}
c_t(v)&=\log p_\theta(v\mid I,x_0,y_{<t}),\\
r_t(v)&=\log p_\theta(v\mid\widetilde I,x_0,y_{<t}),\\
g_t(v)&=\log p_\theta(v\mid I,x_1,y_{<t}),\\
h_t(v)&=\log p_\theta(v\mid\widetilde I,x_1,y_{<t}).
\end{aligned}
\label{eq:B.1}
\end{equation}
The four conditions use the same model, vocabulary, and response prefix, and each keeps the cache of its own input.
Assume that the untruncated logits of every branch are finite on $\mathcal V$, so their softmax probabilities are positive.
Local derivations fix $t$ and drop the time index, writing
\begin{equation}
\begin{gathered}
p_c=e^c,\quad p_r=e^r,\quad p_g=e^g,\quad p_h=e^h,\\
\delta_c=g-c,\quad\delta_r=h-r,\quad d=c-r.
\end{gathered}
\label{eq:B.2}
\end{equation}
$\delta_c$ and $\delta_r$ are the per-token log-probability changes caused by appending the instruction under the two visual conditions.
$d$ compares how strongly the original image and the reference image support the same continuation; both of its conditions keep the task question, the short-answer requirement, and the current response prefix.

\subsubsection{Candidate Constraint and Decoding Distribution}

For a log-probability vector $\ell$ and $0<\beta\leq1$, define
\begin{equation}
\begin{aligned}
\mathcal S_\ell
&=\{v\in\mathcal V:\ell(v)\geq\max_{w\in\mathcal V}\ell(w)+\log\beta\}\\
&=\{v:p_\ell(v)\geq\beta\max_w p_\ell(w)\}.
\end{aligned}
\label{eq:B.3}
\end{equation}
$\mathcal S_\ell$ is the candidate set and also the support of the decoding distribution.
For $\beta=0$, we set $\mathcal S_\ell=\mathcal V$ by convention.
Every maximum-probability token belongs to $\mathcal S_\ell$, so the set is nonempty.
For a nonempty set $S\subseteq\mathcal V$ and a finite score $s$, define the restricted normalization operator
\begin{equation}
\begin{gathered}
\mathcal P_S[s](v)
=\frac{\mathbf1\{v\in S\}\exp s(v)}{Z_S(s)},\\
Z_S(s)=\sum_{w\in S}\exp s(w).
\end{gathered}
\label{eq:B.4}
\end{equation}
For $\alpha\ge0$, the scores and distributions of vanilla VCD, guided VCD, and IP-VCD are
\begin{equation}
\begin{gathered}
s_N=(1+\alpha)c-\alpha r,\quad p_N=\mathcal P_{\mathcal S_c}[s_N],\\
s_G=(1+\alpha)g-\alpha h,\quad p_G=\mathcal P_{\mathcal S_g}[s_G],\\
s_{\mathrm{IP}}=g+\alpha(c-r),\\
\quad p_{\mathrm{IP}}=\mathcal P_{\mathcal S_g}[s_{\mathrm{IP}}],
\end{gathered}
\label{eq:B.5}
\end{equation}
The subscripts $N$, $G$, and $\mathrm{IP}$ denote vanilla unguided VCD, guided VCD, and IP-VCD.
The score functions are defined over the whole vocabulary, and each method normalizes within the candidate set given in Eq.~\eqref{eq:B.5}.
Greedy generation selects the highest-scoring token in the candidate set at each step, breaking ties in a fixed order.
The probability identities below use the distribution defined in Eq.~\eqref{eq:B.4}.

\begin{lemma}[Cancellation of the normalization constant]\label{lem:b1}
For any $u,v\in S$,
\begin{equation}
\log\frac{\mathcal P_S[s](u)}{\mathcal P_S[s](v)}=s(u)-s(v).
\label{eq:B.6}
\end{equation}
\end{lemma}

\begin{proof}
Both probabilities have the denominator $Z_S(s)$, which cancels in the ratio; taking the logarithm gives the result.
Moreover, if the logits of the $j$-th branch are $z_j$, then $\ell_j=z_j-\log\sum_v e^{z_j(v)}$.
For any token-independent coefficients $a_j$, $\sum_j a_jz_j$ and $\sum_j a_j\ell_j$ differ by a token-independent constant.
The restricted softmax, the score differences between candidates, and the argmax all remain the same.
Candidate admission compares $\ell_j(v)-\max_w\ell_j(w)$, which also cancels this constant.
\end{proof}

This property makes the logit implementation match the log-probability derivation, and lets the preservation property be stated through the log odds of normalized distributions.

\subsection{Two Stages: Candidate Admission and the Reference Instruction Interaction}

\subsubsection{Admission Condition for the Abstention Candidate}

For $0<\beta\leq1$, define the margin of candidate $u$ above the admission threshold as
\begin{equation}
K_\ell(u)=\ell(u)-\max_w\ell(w)-\log\beta.
\label{eq:B.7}
\end{equation}
$K_\ell(u)\geq0$ holds exactly when $u\in\mathcal S_\ell$.
Since $g=c+\delta_c$,
\begin{equation}
\begin{gathered}
K_g(u)-K_c(u)\\
=\delta_c(u)-\big[\max_w\{c(w)+\delta_c(w)\}-\max_w c(w)\big].
\end{gathered}
\label{eq:B.8}
\end{equation}

\begin{proof}
Substitute $g=c+\delta_c$ into $K_g$ and $K_c$, and cancel $c(u)$ and $\log\beta$.
Equivalently, candidate admission can be written as a condition over all competing tokens:
\begin{equation}
\begin{gathered}
u\in\mathcal S_g
\iff\\
\quad[c(u)-c(v)]+[\delta_c(u)-\delta_c(v)]\geq\log\beta,\\
\quad\forall v\in\mathcal V,
\end{gathered}
\label{eq:B.9}
\end{equation}
because $g(u)\geq\max_vg(v)+\log\beta$ is equivalent to $g(u)-g(v)\geq\log\beta$ for every $v$.
\end{proof}

Equation~\eqref{eq:B.8} shows that whether the abstention candidate passes the filter depends on how its instruction increment compares with the change in the highest score over the vocabulary.
Only candidates above the threshold take part in later comparisons.
By Eq.~\eqref{eq:B.4}, tokens outside the candidate set have zero probability, and relative probabilities within the set are determined by the score $s$.
IP-VCD filters candidates with $g$ and then adjusts scores with the visual contrast term.

\subsubsection{Full Decomposition of Guided VCD}

\begin{proposition}[Reference instruction interaction]
The four input conditions satisfy
\begin{equation}
\boxed{
\begin{gathered}
s_G=s_N+\delta_c+\alpha(\delta_c-\delta_r),\\
s_{\mathrm{IP}}=s_N+\delta_c.
\end{gathered}}
\label{eq:B.10}
\end{equation}
\end{proposition}

\begin{proof}
Using $g=c+\delta_c$ and $h=r+\delta_r$ and expanding term by term:
\begin{equation}
\begin{aligned}
s_G
&=(1+\alpha)(c+\delta_c)-\alpha(r+\delta_r)\\
&=(1+\alpha)c-\alpha r+(1+\alpha)\delta_c-\alpha\delta_r\\
&=s_N+\delta_c+\alpha(\delta_c-\delta_r),\\
s_{\mathrm{IP}}
&=c+\delta_c+\alpha(c-r)
=s_N+\delta_c.
\end{aligned}
\label{eq:B.11}
\end{equation}
Subtracting the two gives
\begin{equation}
s_{\mathrm{IP}}-s_G=\alpha(\delta_r-\delta_c).
\label{eq:B.12}
\end{equation}
\end{proof}

The decomposition gives, in order, the vanilla visual contrast score, the instruction increment on the original image, and the extra term produced by the different responses under the two visual conditions.

\subsubsection{From Score Changes to Candidate Flips}

Fix an abstention continuation candidate $u$ and an answer continuation candidate $v$, and define score margins as ``abstention minus answer'':
\begin{equation}
\begin{gathered}
m_\ell=\ell(u)-\ell(v),\qquad
\Delta_c=\delta_c(u)-\delta_c(v),\\
\Delta_r=\delta_r(u)-\delta_r(v).
\end{gathered}
\label{eq:B.13}
\end{equation}
Write $m_N=s_N(u)-s_N(v)$ and so on for the combined scores.
Then
\begin{equation}
\begin{aligned}
m_N&=m_c+\alpha(m_c-m_r),\\
m_G&=m_N+\Delta_c+\alpha(\Delta_c-\Delta_r),\\
m_{\mathrm{IP}}&=m_N+\Delta_c
=m_G+\alpha(\Delta_r-\Delta_c).
\end{aligned}
\label{eq:B.14}
\end{equation}
When $\alpha>0$ and $\Delta_r>\Delta_c$, the extra instruction response under the reference condition lowers the abstention-versus-answer margin of guided VCD by $\alpha(\Delta_r-\Delta_c)$.
When $u,v\in\mathcal S_g$, Lemma~\ref{lem:b1} gives the exact relation between normalized probabilities
\begin{equation}
\frac{p_{\mathrm{IP}}(u)/p_{\mathrm{IP}}(v)}{p_G(u)/p_G(v)}
=\exp\{\alpha(\Delta_r-\Delta_c)\}.
\label{eq:B.15}
\end{equation}
Hence the strict condition for the pair to flip from the answer leading to abstention leading is
\begin{equation}
m_G<0<m_G+\alpha(\Delta_r-\Delta_c).
\label{eq:B.16}
\end{equation}
For the greedy winner over the whole candidate set, let $k(w)=\alpha[\delta_r(w)-\delta_c(w)]$.
A candidate $u\in\mathcal S_g$ is the unique highest-scoring token of IP-VCD if and only if
\begin{equation}
\begin{gathered}
s_G(u)-s_G(v)+k(u)-k(v)>0,\\
\qquad\forall v\in\mathcal S_g\setminus\{u\}.
\end{gathered}
\label{eq:B.17}
\end{equation}
Equation~\eqref{eq:B.16} gives the ranking flip between two candidates; Eq.~\eqref{eq:B.17} further requires the abstention candidate to exceed every other token in the set, which corresponds to the actual greedy choice.

The unguided reference can also change the relative support for abstention: by Eq.~\eqref{eq:B.14}, when $m_r>m_c$ and $\alpha>0$, vanilla contrast changes the margin relative to the original distribution by $\alpha(m_c-m_r)<0$.
This visual contrast term remains part of IP-VCD; Eq.~\eqref{eq:B.12} states exactly how IP-VCD handles the extra instruction interaction.

\subsubsection{Removing the Instruction Only from the Reference Side}

Removing the abstention clause only on the reference side gives the score
\begin{equation}
\begin{aligned}
s_{\mathrm{ref\text{-}off}}
&=(1+\alpha)g-\alpha r\\
&=s_N+(1+\alpha)\delta_c
=s_{\mathrm{IP}}+\alpha\delta_c.
\end{aligned}
\label{eq:B.18}
\end{equation}
The original-only variant amplifies the instruction increment on the original image by a factor of $1+\alpha$.
IP-VCD adds the unguided original-image condition $c$, which keeps the coefficient of the instruction increment at 1 and lets $\alpha$ alone control the coefficient of the visual contrast term.

\subsection{The Preservation Property and the Three-Condition Coefficients}

\subsubsection{Exact Preservation of the Relative Instruction Effect}

\begin{theorem}[Instruction-increment preservation]\label{thm:b1}
Fix the generated prefix, and normalize $s_N$ and $s_{\mathrm{IP}}$ on the same nonempty support $S$.
For any $u,v\in S$,
\begin{equation}
\boxed{
\begin{aligned}
&\log\frac{\mathcal P_S[s_{\mathrm{IP}}](u)}{\mathcal P_S[s_{\mathrm{IP}}](v)}
-
\log\frac{\mathcal P_S[s_N](u)}{\mathcal P_S[s_N](v)}\\
&\qquad=\delta_c(u)-\delta_c(v).
\end{aligned}}
\label{eq:B.19}
\end{equation}
\end{theorem}

\begin{proof}
By Lemma~\ref{lem:b1}, the left-hand side equals
\begin{equation*}
\begin{aligned}
&[s_{\mathrm{IP}}(u)-s_{\mathrm{IP}}(v)]-[s_N(u)-s_N(v)]\\
&=[s_N(u)+\delta_c(u)-s_N(v)-\delta_c(v)]\\
&\qquad-[s_N(u)-s_N(v)]
=\delta_c(u)-\delta_c(v).
\end{aligned}
\end{equation*}
\end{proof}
For the actual distributions, which use $\mathcal S_c$ and $\mathcal S_g$ respectively, the same conclusion holds for $u,v\in\mathcal S_c\cap\mathcal S_g$.
For candidates that newly enter $\mathcal S_g$, the instruction effect within the same candidate set can be examined separately by normalizing $s_N$ on $\mathcal S_g$.

Exponentiating Eq.~\eqref{eq:B.19} gives
\begin{equation}
\frac{\mathcal P_S[s_{\mathrm{IP}}](u)}{\mathcal P_S[s_{\mathrm{IP}}](v)}
=
\frac{\mathcal P_S[s_N](u)}{\mathcal P_S[s_N](v)}
\frac{p_g(u)/p_c(u)}{p_g(v)/p_c(v)}.
\label{eq:B.20}
\end{equation}
The change in candidate odds caused by the instruction on the original image is therefore kept in full in the odds after visual contrast.
Correspondingly, relative to the guided original-image base $p_g$, the visual change applied by IP-VCD is
\begin{equation}
\begin{gathered}
\log\frac{p_{\mathrm{IP}}(u)}{p_{\mathrm{IP}}(v)}
-\log\frac{p_g(u)}{p_g(v)}
=\alpha[d(u)-d(v)],\\
u,v\in\mathcal S_g.
\end{gathered}
\label{eq:B.21}
\end{equation}
Equations~\eqref{eq:B.19} and~\eqref{eq:B.21} show that the instruction increment on the original image is kept at its original size, while $\alpha$ sets the strength of the visual contrast term.

\subsubsection{Unique Coefficients Within Three-Condition Linear Combinations}

\begin{theorem}[Characterization of coefficients in the three-condition linear family]
Consider the three-condition linear family with token-independent coefficients
\begin{equation}
s_{a,b,e}=a g+b c+e r.
\label{eq:B.22}
\end{equation}
Require the family to satisfy two properties for positive conditional distributions under any input.
First, when $c,r$ are fixed and $g$ is replaced by $g'$, the score change of any candidate pair equals the change of that pair in $g'-g$.
Second, when $g=c$, the score difference of any candidate pair equals that of vanilla VCD.
Then the coefficients are uniquely
\begin{equation}
(a,b,e)=(1,\alpha,-\alpha).
\label{eq:B.23}
\end{equation}
\end{theorem}

\begin{proof}
The first requirement gives
\begin{equation*}
\begin{gathered}
a\big[(g'(u)-g(u))-(g'(v)-g(v))\big]\\
=(g'(u)-g(u))-(g'(v)-g(v)).
\end{gathered}
\end{equation*}
The log odds of a candidate pair under positive distributions can vary continuously, so this identity requires $a=1$.
The second requirement asks that, for $g=c$ and any $c,r,u,v$,
\begin{equation*}
\begin{aligned}
&(a+b)[c(u)-c(v)]+e[r(u)-r(v)]\\
&=(1+\alpha)[c(u)-c(v)]-\alpha[r(u)-r(v)].
\end{aligned}
\end{equation*}
Varying the pairwise odds of $c$ and of $r$ separately gives $a+b=1+\alpha$ and $e=-\alpha$.
Substituting $a=1$ gives $b=\alpha$.
\end{proof}

Uniqueness refers to the linear family of Eq.~\eqref{eq:B.22} and the two design requirements above.
Adding the same token-independent constant to all scores leaves the decoding distribution unchanged.

\subsubsection{Probability Form and Special Cases}

Substituting the three-condition score into the restricted normalization gives
\begin{equation}
p_{\mathrm{IP}}(v)
=\frac{\mathbf1\{v\in\mathcal S_g\}\,p_g(v)
\left[p_c(v)/p_r(v)\right]^\alpha}
{\sum_{w\in\mathcal S_g}p_g(w)\left[p_c(w)/p_r(w)\right]^\alpha}.
\label{eq:B.24}
\end{equation}
This form, a base distribution multiplied by a likelihood ratio, is a decoding-time product-of-experts construction \citep{dexperts}.
All three distributions are computed by the same frozen model under their respective input conditions.

If $g=c$, then $\mathcal S_g=\mathcal S_c$, and both the score and the support recover vanilla VCD.
If $c=r$ or $\alpha=0$, then
\begin{equation}
p_{\mathrm{IP}}(v)
=\frac{\mathbf1\{v\in\mathcal S_g\}p_g(v)}{p_g(\mathcal S_g)}.
\label{eq:B.25}
\end{equation}
This is the conditional distribution of the guided base on the retained set.
Since $\mathcal S_g$ contains every maximum-probability candidate, greedy selection is the same as guided Direct.
If $\beta=0$, Eq.~\eqref{eq:B.25} further equals the full $p_g$.

\subsection{The Unique Optimal Distribution Under KL Regularization}

\subsubsection{Objective and Closed-Form Solution}

Fix the generated prefix and the support $S=\mathcal S_g$, and write the conditional distribution of the guided base on this support as
\begin{equation}
G(v)=\frac{p_g(v)}{p_g(S)},\qquad v\in S,
\label{eq:B.26}
\end{equation}
and let $\Delta(S)$ be the probability simplex on $S$.
The visual contrast term is $d(v)=c(v)-r(v)$.
Consider the local objective
\begin{equation}
\begin{gathered}
\mathcal L_\alpha(\pi)
=\alpha\sum_{v\in S}\pi(v)d(v)
-\sum_{v\in S}\pi(v)\log\frac{\pi(v)}{G(v)},\\
\qquad\pi\in\Delta(S),
\end{gathered}
\label{eq:B.27}
\end{equation}
with the convention $0\log0=0$.
The first term rewards support from the original image relative to the reference, and the second term uses the KL divergence to limit departures from the guided base.

\begin{theorem}[Variational optimality on a fixed support]\label{thm:b3}
The unique maximizer of objective~\eqref{eq:B.27} is
\begin{equation}
\begin{gathered}
\pi_\alpha^*(v)
=\frac{G(v)e^{\alpha d(v)}}{\overline Z_\alpha},\\
\overline Z_\alpha=\sum_{w\in S}G(w)e^{\alpha d(w)}.
\end{gathered}
\label{eq:B.28}
\end{equation}
This distribution is exactly $p_{\mathrm{IP}}$.
\end{theorem}

\begin{proof}
By Eq.~\eqref{eq:B.28},
\begin{equation*}
\log\pi_\alpha^*(v)=\log G(v)+\alpha d(v)-\log\overline Z_\alpha.
\end{equation*}
For any $\pi\in\Delta(S)$, substituting into the KL divergence gives
\begin{equation*}
\begin{aligned}
D_{\mathrm{KL}}(\pi\|\pi_\alpha^*)
&=\sum_v\pi(v)[\log\pi(v)-\log G(v)\\
&\qquad\qquad-\alpha d(v)+\log\overline Z_\alpha]\\
&=D_{\mathrm{KL}}(\pi\|G)-\alpha\mathbb E_\pi[d]+\log\overline Z_\alpha.
\end{aligned}
\end{equation*}
Rearranging,
\begin{equation}
\boxed{
\mathcal L_\alpha(\pi)
=\log\overline Z_\alpha-D_{\mathrm{KL}}(\pi\|\pi_\alpha^*).
}
\label{eq:B.29}
\end{equation}
Nonnegativity of KL follows directly from $\log x\leq x-1$.
Summing over terms with $\pi(v)>0$,
\begin{equation*}
\begin{gathered}
-\sum_{v:\pi(v)>0}\pi(v)\log\frac{\pi_\alpha^*(v)}{\pi(v)}\\
\geq\sum_{v:\pi(v)>0}[\pi(v)-\pi_\alpha^*(v)]\geq0.
\end{gathered}
\end{equation*}
Equality requires every ratio at positive-probability positions to be 1 and $\pi_\alpha^*$ to put zero total mass elsewhere.
Since $\pi_\alpha^*$ is positive everywhere on $S$, equality holds exactly when $\pi=\pi_\alpha^*$.
The unique maximum of Eq.~\eqref{eq:B.29} is therefore $\log\overline Z_\alpha$, attained by $\pi_\alpha^*$.
Substituting $G=p_g/p_g(S)$ into Eq.~\eqref{eq:B.28}, the common factor $p_g(S)$ cancels and gives Eq.~\eqref{eq:B.24}.
\end{proof}

The log probabilities of the optimal distribution differ from the three-condition score only by a candidate-independent normalization constant, so greedy decoding at each step selects the mode of this distribution.
The objective determines the next-token distribution under a fixed prefix; DU evaluates the final decision from the complete response and the fixed abstention labels.

\subsubsection{How the Contrast Strength Changes the Distribution}

Let $\psi(\alpha)=\log\overline Z_\alpha$.
A finite vocabulary allows term-by-term differentiation, giving
\begin{equation}
\begin{gathered}
\psi'(\alpha)=\mathbb E_{\pi_\alpha^*}[d],\qquad
\psi''(\alpha)=\operatorname{Var}_{\pi_\alpha^*}(d)\geq0.
\end{gathered}
\label{eq:B.30}
\end{equation}

\begin{proof}
Differentiating $\overline Z_\alpha$ and dividing by itself gives the first equation.
Differentiating again, the second equals $\mathbb E[d^2]-(\mathbb E[d])^2$.
Equivalently, for each candidate,
\begin{equation}
\frac{\partial\pi_\alpha^*(v)}{\partial\alpha}
=\pi_\alpha^*(v)\big[d(v)-\mathbb E_{\pi_\alpha^*}[d]\big].
\label{eq:B.31}
\end{equation}
\end{proof}
With the three conditional distributions and the candidate set fixed, increasing $\alpha$ raises the probability of candidates whose visual contrast term is above the current mean, and the expected visual contrast term does not decrease.

For a predefined candidate group $A\subset S$, let $B=S\setminus A$, with both groups nonempty.
Summing gives
\begin{equation}
\begin{aligned}
\frac{\partial\pi_\alpha^*(A)}{\partial\alpha}
=\pi_\alpha^*(A)\pi_\alpha^*(B)
\big[&\mathbb E_{\pi_\alpha^*}[d\mid A]\\
&-\mathbb E_{\pi_\alpha^*}[d\mid B]\big].
\end{aligned}
\label{eq:B.32}
\end{equation}
The growth of a group's probability depends on the difference between the conditional means of the visual contrast term in the two groups, which gives the condition under which the contrast strength changes the probability of a whole group.

\subsection{From Single Candidates to Candidate-Group Probabilities}

\subsubsection{A General Grouping Identity for Exponential Tilting}

Let $p$ be a positive distribution on $\mathcal V$, let $A,B$ be a nonempty partition of $\mathcal V$, and let $A\cap S$ and $B\cap S$ both be nonempty.
For any finite function $f$, define
\begin{equation}
q(v)=\frac{\mathbf1\{v\in S\}p(v)e^{f(v)}}{\sum_{w\in S}p(w)e^{f(w)}}.
\label{eq:B.33}
\end{equation}
Write $O_p(A)=p(A)/p(B)$ and let $p_S=p(\cdot\mid S)$.
Then
\begin{equation}
\boxed{
\begin{aligned}
&\log O_q(A)-\log O_p(A)
=\log\frac{p(S\mid A)}{p(S\mid B)}\\
&\quad+\log\mathbb E_{p_S(\cdot\mid A)}[e^f]
-\log\mathbb E_{p_S(\cdot\mid B)}[e^f].
\end{aligned}}
\label{eq:B.34}
\end{equation}

\begin{proof}
Summing over each group, the normalization constant cancels:
\begin{equation*}
\begin{aligned}
\frac{q(A)}{q(B)}
&=\frac{\sum_{v\in A\cap S}p(v)e^{f(v)}}{\sum_{v\in B\cap S}p(v)e^{f(v)}}\\
&=\frac{p(A)p(S\mid A)}{p(B)p(S\mid B)}
\frac{\mathbb E_{p(\cdot\mid A\cap S)}[e^f]}{\mathbb E_{p(\cdot\mid B\cap S)}[e^f]}.
\end{aligned}
\end{equation*}
Taking logarithms gives Eq.~\eqref{eq:B.34}.
\end{proof}
The first term measures how much candidate filtering changes the between-group odds, and the other two terms measure the reweighting of the retained candidates.
If $A\cap S$ is empty, then $q(A)=0$; if $B\cap S$ is empty, then $q(A)=1$, and the probability follows directly from the support.

For IP-VCD, taking $p=p_g$, $f=\alpha(c-r)$, and $S=\mathcal S_g$, Eq.~\eqref{eq:B.34} gives the full probability change of any prespecified group of abstention tokens relative to the guided original-image distribution.

\subsubsection{Group Probability Change After Removing the Reference Interaction}

Guided VCD and IP-VCD share $\mathcal S_g$.
Let $k=\alpha(\delta_r-\delta_c)$.
By Eq.~\eqref{eq:B.12},
\begin{equation}
p_{\mathrm{IP}}(v)
=\frac{p_G(v)e^{k(v)}}{\mathbb E_{p_G}[e^k]},\qquad v\in\mathcal S_g.
\label{eq:B.35}
\end{equation}
Hence, for a nonempty partition $A,B$ of the candidate set,
\begin{equation}
\begin{aligned}
&\log O_{p_{\mathrm{IP}}}(A)-\log O_{p_G}(A)\\
&\quad=\log\mathbb E_{p_G(\cdot\mid A)}[e^k]
-\log\mathbb E_{p_G(\cdot\mid B)}[e^k].
\end{aligned}
\label{eq:B.36}
\end{equation}

\begin{proof}
Substitute $s_{\mathrm{IP}}=s_G+k$ into the restricted softmax and divide the denominator by $Z_{\mathcal S_g}(s_G)$ to obtain Eq.~\eqref{eq:B.35}; then apply Eq.~\eqref{eq:B.34} on the same support.
\end{proof}

The group probability increases strictly if and only if
\begin{equation}
\begin{gathered}
p_{\mathrm{IP}}(A)>p_G(A)
\iff
\mathbb E_{p_G(\cdot\mid A)}[e^k]\\
>\mathbb E_{p_G(\cdot\mid B)}[e^k].
\end{gathered}
\label{eq:B.37}
\end{equation}
An intuitive sufficient condition is $\min_{u\in A}k(u)>\max_{v\in B}k(v)$.
For a single candidate, the exact condition is $k(u)>\log\mathbb E_{p_G}[e^k]$.
These relations connect candidate score changes with changes in normalized probability: the advantage of one group over the other depends on the weighted increments of the two groups.

\subsubsection{Decomposing VCD into Candidate Support, Between-Group Preference, and Within-Group Allocation}

Let $p,q$ be the original and reference probabilities of a VCD instance, let $S$ be the support, and let the contrastive distribution be
\begin{equation}
m(v)\propto\mathbf1\{v\in S\}p(v)^{1+\alpha}q(v)^{-\alpha}.
\label{eq:B.38}
\end{equation}
Define $p_S=p(\cdot\mid S)$ and $q_S=q(\cdot\mid S)$, and on the two retained groups define $p_A=p_S(\cdot\mid A)$, $q_A=q_S(\cdot\mid A)$, and likewise $p_B,q_B$.
For $\rho>1$, the R\'enyi divergence follows the standard definition for finite distributions \citep{vanerven2014renyi}
\begin{equation}
D_\rho(P\|Q)=\frac{1}{\rho-1}\log\sum_v P(v)^\rho Q(v)^{1-\rho}.
\label{eq:B.39}
\end{equation}

\begin{proposition}[Exact decomposition of the group-level contrast effect]
When $\alpha>0$ and both groups have positive mass on $S$,
\begin{equation}
{
\boxed{
\begin{aligned}
&\log O_m(A)-\log O_p(A)\\
&\quad=\underbrace{\log O_{p_S}(A)-\log O_p(A)}_{\text{candidate support}}\\
&\qquad+\underbrace{\alpha[\log O_{p_S}(A)-\log O_{q_S}(A)]}_{\text{reference between-group preference}}\\
&\qquad+\underbrace{\alpha\big[D_{1+\alpha}(p_A\|q_A)-D_{1+\alpha}(p_B\|q_B)\big]}_{\text{within-group allocation}}.
\end{aligned}}}
\label{eq:B.40}
\end{equation}
\end{proposition}

\begin{proof}
Replacing $p,q$ by $p_S,q_S$ only introduces a proportionality constant shared by the two groups.
The unnormalized mass on group $A$ can be written as
\begin{equation}
\begin{aligned}
F_A
&=\sum_{v\in A\cap S}p_S(v)^{1+\alpha}q_S(v)^{-\alpha}\\
&=p_S(A)^{1+\alpha}q_S(A)^{-\alpha}\\
&\quad\times\sum_{v\in A\cap S}p_A(v)^{1+\alpha}q_A(v)^{-\alpha}\\
&=p_S(A)^{1+\alpha}q_S(A)^{-\alpha}
\exp\{\alpha D_{1+\alpha}(p_A\|q_A)\}.
\end{aligned}
\label{eq:B.41}
\end{equation}
The same holds for $B$.
Since $O_m(A)=F_A/F_B$,
\begin{equation*}
\begin{aligned}
\log O_m(A)
={}&(1+\alpha)\log O_{p_S}(A)\\
&-\alpha\log O_{q_S}(A)\\
&+\alpha\big[D_{1+\alpha}(p_A\|q_A)\\
&\qquad-D_{1+\alpha}(p_B\|q_B)\big].
\end{aligned}
\end{equation*}
Subtracting $\log O_p(A)$ and splitting $(1+\alpha)\log O_{p_S}(A)$ gives Eq.~\eqref{eq:B.40}.
\end{proof}
When $\alpha=0$, $m=p_S$, and the probability change is given by the candidate-support term.

This decomposition grounds the change in group-level abstention in three computable quantities: which candidates pass the threshold, how the reference distribution favors each group, and how each group distributes probability among its own tokens.
When the retained set contains only two candidates, one per group, both within-group divergences are zero and Eq.~\eqref{eq:B.40} recovers the score-margin decomposition of a candidate pair.
The local abstention group is a prespecified set of tokens; the abstention probability of a complete response is computed as an event over terminated sequences.

\subsection{Complete-Response Probabilities in Autoregressive Generation}

\subsubsection{Accumulating Local Normalization Along a Path}

Let $\mathcal Y_T$ be the set of all output token sequences that end at the first end-of-sequence token or at the maximum length $T$.
For any method $M$, the stepwise normalized distributions define the complete sequence probability
\begin{equation}
P_M(y)=\prod_{t=1}^{|y|}p_{M,t}(y_t\mid y_{<t}),\qquad y\in\mathcal Y_T.
\label{eq:B.42}
\end{equation}
Each internal node of the generation tree distributes its probability mass to its children according to the normalized next-token probabilities, so the probabilities of all leaves sum to 1.
For a response ended by the end-of-sequence token, the product includes that token; for a response that reaches the length limit, the product includes the last token actually generated.

For IP-VCD, define
\begin{equation}
\begin{aligned}
Z_{\mathrm{IP},t}(y_{<t})
&=\sum_{v\in\mathcal S_{g,t}(y_{<t})}\\
&\quad\times\exp\{g_t(v)+\alpha[c_t(v)-r_t(v)]\}.
\end{aligned}
\label{eq:B.43}
\end{equation}
For paths that satisfy the support constraint at every step, substituting into Eq.~\eqref{eq:B.42} and taking logarithms gives
\begin{equation}
\begin{aligned}
\log P_{\mathrm{IP}}(y)
={}&\sum_t g_t(y_t)
+\alpha\sum_t[c_t(y_t)-r_t(y_t)]\\
&-\sum_t\log Z_{\mathrm{IP},t}(y_{<t}).
\end{aligned}
\label{eq:B.44}
\end{equation}
Once a path contains a token outside the support, $P_{\mathrm{IP}}(y)=0$.
In Eq.~\eqref{eq:B.44}, all conditional distributions are evaluated along the same $y$, and the local normalization constants are determined by the prefixes of that path.

Further let $P_g(y)=\prod_t e^{g_t(y_t)}$, $P_c(y)=\prod_t e^{c_t(y_t)}$, $P_r(y)=\prod_t e^{r_t(y_t)}$, and $\chi_g(y)=\prod_t\mathbf1\{y_t\in\mathcal S_{g,t}(y_{<t})\}$.
For fixed $\alpha$, the full product form is
\begin{equation}
P_{\mathrm{IP}}(y)
=\frac{\chi_g(y)P_g(y)[P_c(y)/P_r(y)]^\alpha}
{\prod_t Z_{\mathrm{IP},t}(y_{<t})}.
\label{eq:B.45}
\end{equation}
Equation~\eqref{eq:B.45} keeps the normalization constant of every prefix.
Computing these quantities step by step along a given response gives its complete sequence probability.

\subsubsection{Path Difference Between IP-VCD and Guided VCD}

Let $Z_{G,t}$ be the stepwise normalization constant of guided VCD.
On the same given prefix, the two methods use the same $\mathcal S_{g,t}$.
For paths with positive probability under both, Eqs.~\eqref{eq:B.12} and~\eqref{eq:B.44} give
\begin{equation}
\begin{aligned}
\log\frac{P_{\mathrm{IP}}(y)}{P_G(y)}
={}&\alpha\sum_t[\delta_{r,t}(y_t)-\delta_{c,t}(y_t)]\\
&-\sum_t\log\frac{Z_{\mathrm{IP},t}(y_{<t})}{Z_{G,t}(y_{<t})}.
\end{aligned}
\label{eq:B.46}
\end{equation}

\begin{proof}
Write $\log p_{\mathrm{IP},t}-\log p_{G,t}$ for each step, use the score difference~\eqref{eq:B.12}, and sum along the path.
\end{proof}

Equation~\eqref{eq:B.46} accounts both for the removed reference interaction and for the redistribution of probability mass at each step.
Greedy paths form their own prefixes after the first step at which the selected tokens differ; when comparing complete paths, the probabilities above are computed step by step along each given response.

\subsubsection{Probability Change of the Complete Abstention Event}

Let $\mathcal A\subset\mathcal Y_T$ be the abstention event labeled by the semantics of the complete response, and $\mathcal B=\mathcal Y_T\setminus\mathcal A$.
Take a base path distribution $P_0$ that is positive on $\mathcal Y_T$, such as the untruncated guided distribution $P_g$, and another normalized path distribution $P_1$.
Define the log-likelihood change of a path as
\begin{equation}
\begin{gathered}
D(y)=\log\frac{P_1(y)}{P_0(y)},\\
e^{D(y)}=0\ \text{when}\ P_1(y)=0.
\end{gathered}
\label{eq:B.47}
\end{equation}
As long as both events have positive mass under the compared distributions,
\begin{equation}
\boxed{
\begin{aligned}
&\log\frac{P_1(\mathcal A)}{P_1(\mathcal B)}
-\log\frac{P_0(\mathcal A)}{P_0(\mathcal B)}\\
&\quad=\log\mathbb E_{P_0(\cdot\mid\mathcal A)}[e^D]
-\log\mathbb E_{P_0(\cdot\mid\mathcal B)}[e^D].
\end{aligned}}
\label{eq:B.48}
\end{equation}

\begin{proof}
Since $P_1(y)=P_0(y)e^{D(y)}$, summing over the event $\mathcal A$ gives
\begin{equation*}
P_1(\mathcal A)=P_0(\mathcal A)\mathbb E_{P_0(\cdot\mid\mathcal A)}[e^D].
\end{equation*}
The same holds for $\mathcal B$; taking the ratio of the two and then the logarithm gives the result.
\end{proof}

For a prespecified finite set of responses $\mathcal C\subseteq\mathcal Y_T$, constructing $P_j(\cdot\mid\mathcal C)$ and grouping by $\mathcal A\cap\mathcal C$ and $\mathcal B\cap\mathcal C$, Eq.~\eqref{eq:B.48} still holds.
The computation uses the complete token sequence, the way generation ends, and the semantic label of the response, and counts identical token paths once.
Token admission, candidate ranking, and the complete abstention event are thus each computed in their own sample space.

\subsection{Extension to IP-M3ID}

M3ID contrasts a condition with visual input against a condition with the image removed, and controls the strength of this difference through next-token confidence and a time weight \citep{m3id}.
In IP-M3ID, $c_t$ still denotes the unguided original-image log probability, and $r_t$ becomes the image-free log probability that keeps the same text and generated prefix.
Let $\tau_t^c$ be the sequence position index of the unguided condition and $\kappa$ the confidence threshold.
The weight of the current construction can be written as
\begin{equation}
\begin{gathered}
w_t^c
=\mathbf1\{\max_v e^{c_t(v)}<\kappa\}
\big(e^{\lambda\tau_t^c}-1\big),\\
\lambda\geq0,\quad\tau_t^c\geq0.
\end{gathered}
\label{eq:B.49}
\end{equation}
Vanilla M3ID defines the time weight by the generation step \citep{m3id}; IP-M3ID uses the sum of the base-instruction length and the number of generated tokens, $\tau_t^c=(t-1)+k_0^c$, where $k_0^c$ is the number of tokens in the base instruction.
$\lambda=0.02$ and $\kappa=0.3$ follow the original settings.
The weight is determined jointly by $c_t$ and this position index and is the same for all candidates on a fixed prefix.

IP-M3ID and its unguided form with the same weight are
\begin{equation}
\begin{aligned}
s_{\mathrm{IPM},t}&=g_t+w_t^c(c_t-r_t),\\
s_{\mathrm{NM},t}&=c_t+w_t^c(c_t-r_t),\\
s_{\mathrm{IPM},t}-s_{\mathrm{NM},t}&=\delta_{c,t}.
\end{aligned}
\label{eq:B.50}
\end{equation}
IP-M3ID normalizes over the full vocabulary.
The candidate-pair preservation property of Theorem~\ref{thm:b1} therefore holds over the full vocabulary, and Theorem~\ref{thm:b3} holds on each fixed prefix with $\alpha$ replaced by $w_t^c$.
Path probabilities keep $\sum_t w_t^c[c_t(y_t)-r_t(y_t)]$ with the actual weight at each step.

For guided M3ID, let $w_t^g=\mathbf1\{\max_v e^{g_t(v)}<\kappa\}(e^{\lambda\tau_t^g}-1)$, where $\tau_t^g$ is the sequence position of the guided condition.
With the guided image-free distribution $h_t$ as the reference,
\begin{equation*}
s_{\mathrm{GM},t}=g_t+w_t^g(g_t-h_t).
\end{equation*}
Using $g_t-h_t=(c_t-r_t)+(\delta_{c,t}-\delta_{r,t})$ gives the full difference
\begin{equation}
\boxed{
\begin{aligned}
s_{\mathrm{IPM},t}-s_{\mathrm{GM},t}
={}&(w_t^c-w_t^g)(c_t-r_t)\\
&+w_t^g(\delta_{r,t}-\delta_{c,t}).
\end{aligned}}
\label{eq:B.51}
\end{equation}
The first term reflects the difference in gating and time weights, and the second reflects the reference instruction interaction.
When the two weights are equal, the first term vanishes and the difference is determined by the instruction interaction.

\subsection{Worked Substitution for the Beef Carpaccio Example}

In the beef carpaccio example of LLaVA-Mistral, let $u$ and $v$ be the first tokens of abstention and of the answer Toast.
With the input, generated prefix, and noise fixed, the ``abstention minus answer'' score margins under the four conditions are
\begin{equation*}
\begin{gathered}
(m_c,m_r,m_g,m_h)\\
=(-12.5625,-11.546875,1.4375,3.0625).
\end{gathered}
\end{equation*}
By Eq.~\eqref{eq:B.13},
\begin{equation}
\begin{aligned}
\Delta_c&=m_g-m_c=1.4375-(-12.5625)=14,\\
\Delta_r&=m_h-m_r=3.0625-(-11.546875)\\
&=14.609375.
\end{aligned}
\label{eq:B.52}
\end{equation}
With $\alpha=1$, the vanilla visual contrast margin is
\begin{equation}
m_N=2m_c-m_r=-13.578125.
\label{eq:B.53}
\end{equation}
Guided VCD and IP-VCD give
\begin{equation}
\begin{aligned}
m_G&=m_N+\Delta_c+(\Delta_c-\Delta_r)\\
&=-13.578125+14-0.609375\\
&=-0.1875,\\
m_{\mathrm{IP}}&=m_N+\Delta_c=-13.578125+14\\
&=0.421875.
\end{aligned}
\label{eq:B.54}
\end{equation}
Both candidates pass guided filtering.
Relative to guided VCD, IP-VCD multiplies the odds of the abstention candidate over the answer candidate by $\exp(0.609375)$.
The highest-scoring token at this position changes from the answer first token $v$ to the abstention first token $u$, and the complete response changes from Toast to an abstention.

We further insert a diagnostic parameter $\eta\in[0,1]$ on the common support and define
\begin{equation}
\begin{gathered}
s_\eta=s_{\mathrm{IP}}+\eta\alpha(\delta_c-\delta_r),\\
m_\eta=0.421875-0.609375\eta.
\end{gathered}
\label{eq:B.55}
\end{equation}
With the prefix fixed, the exact crossing point of the pair's scores on the common support is
\begin{equation}
\eta_* = \frac{0.421875}{0.609375}=\frac9{13}\approx0.692308.
\label{eq:B.56}
\end{equation}
For $\eta<9/13$, the abstention candidate leads; for $\eta>9/13$, the answer candidate leads.
The existing conditions $\eta=0$, $0.5$, and $1$ select the abstention first token, the abstention first token, and the answer first token, respectively.
The example links the decomposition of the instruction response to changes in candidate ranking and in the complete response.
\section{Full Results on Nine Models}
\label{app:full}

Tables~\ref{tab:foodfull} and~\ref{tab:vizfull} collect the main results of all methods by model, reporting overall accuracy on Food-101 and answer score on VizWiz, together with abstention precision, abstention recall, and DU.
Overall accuracy, answer score, and DU use all inputs as the denominator; abstention precision and recall use the number of abstentions and the number of abstention-labeled inputs.
When there is no abstention, precision is shown as 0.00.
Bold and underline mark the highest and second-highest DU within the same model and the same evaluation set; the 505-question subset marks only the highest value in bold.

\subsection{Food-101}

\begin{table*}[tp]
\centering
\footnotesize
\setlength{\tabcolsep}{4pt}
\resizebox{0.98\textwidth}{!}{%
\begin{tabular}{llrrrrrr}
\toprule
Model & Method & Overall accuracy & Abstention precision & Abstention recall & Decision utility & Supported abstentions & Unsupported abstentions \\
\midrule
Qwen2.5-VL & Direct & 30.40 & 0.00 & 0.00 & 30.40 & 0 & 0 \\
Qwen2.5-VL & Vanilla VCD & 28.34 & 0.00 & 0.00 & 28.34 & 0 & 0 \\
Qwen2.5-VL & Vanilla M3ID & 28.71 & 0.00 & 0.00 & 28.71 & 0 & 0 \\
Qwen2.5-VL & DoLa & 24.50 & 0.00 & 0.00 & 24.50 & 0 & 0 \\
Qwen2.5-VL & DeCo & 30.24 & 0.00 & 0.00 & 30.24 & 0 & 0 \\
Qwen2.5-VL & SID & 26.40 & 0.00 & 0.00 & 26.40 & 0 & 0 \\
Qwen2.5-VL & CDA & 33.91 & 23.54 & 23.24 & 37.58 & 89 & 289 \\
Qwen2.5-VL & \textbf{IP-VCD (Ours)} & 35.19 & 29.23 & 55.87 & \underline{44.02} & 214 & 518 \\
Qwen2.5-VL & \textbf{IP-M3ID (Ours)} & 42.16 & 25.31 & 37.34 & \textbf{48.06} & 143 & 422 \\
LLaVA-Mistral & Direct & 33.87 & 77.78 & 0.74 & 34.16 & 7 & 2 \\
LLaVA-Mistral & Vanilla VCD & 38.24 & 0.00 & 0.00 & 38.24 & 0 & 0 \\
LLaVA-Mistral & Vanilla M3ID & 34.90 & 100.00 & 0.21 & 34.98 & 2 & 0 \\
LLaVA-Mistral & DoLa & 33.29 & 71.43 & 0.53 & 33.50 & 5 & 2 \\
LLaVA-Mistral & DeCo & 26.90 & 50.00 & 0.11 & 26.94 & 1 & 1 \\
LLaVA-Mistral & SID & 38.28 & 0.00 & 0.00 & 38.28 & 0 & 0 \\
LLaVA-Mistral & CDA & 30.03 & 71.43 & 1.05 & 30.45 & 10 & 4 \\
LLaVA-Mistral & \textbf{IP-VCD (Ours)} & 28.38 & 53.04 & 61.54 & \textbf{52.48} & 584 & 517 \\
LLaVA-Mistral & \textbf{IP-M3ID (Ours)} & 13.57 & 42.29 & 77.45 & \underline{43.89} & 735 & 1003 \\
MiniCPM-V2.6 & Direct & 43.03 & 100.00 & 0.50 & 43.15 & 3 & 0 \\
MiniCPM-V2.6 & Vanilla VCD & 43.61 & 100.00 & 0.17 & 43.65 & 1 & 0 \\
MiniCPM-V2.6 & Vanilla M3ID & 42.70 & 100.00 & 0.34 & 42.78 & 2 & 0 \\
MiniCPM-V2.6 & DoLa & 31.11 & 100.00 & 0.17 & 31.15 & 1 & 0 \\
MiniCPM-V2.6 & DeCo & 41.96 & 100.00 & 0.50 & 42.08 & 3 & 0 \\
MiniCPM-V2.6 & CDA & 27.64 & 0.00 & 0.00 & 27.64 & 0 & 4 \\
MiniCPM-V2.6 & \textbf{IP-VCD (Ours)} & 37.25 & 46.74 & 42.04 & \textbf{47.61} & 251 & 286 \\
MiniCPM-V2.6 & \textbf{IP-M3ID (Ours)} & 31.27 & 38.24 & 56.11 & \underline{45.09} & 335 & 541 \\
InternVL3.5 & Direct & 17.08 & 0.00 & 0.00 & 17.08 & 0 & 0 \\
InternVL3.5 & Vanilla VCD & 19.10 & 100.00 & 0.09 & 19.14 & 1 & 0 \\
InternVL3.5 & Vanilla M3ID & 17.57 & 0.00 & 0.00 & 17.57 & 0 & 0 \\
InternVL3.5 & DoLa & 16.17 & 0.00 & 0.00 & 16.17 & 0 & 0 \\
InternVL3.5 & DeCo & 17.95 & 0.00 & 0.00 & 17.95 & 0 & 0 \\
InternVL3.5 & SID & 18.52 & 0.00 & 0.00 & 18.52 & 0 & 0 \\
InternVL3.5 & CDA & 20.38 & 100.00 & 0.09 & 20.42 & 1 & 0 \\
InternVL3.5 & \textbf{IP-VCD (Ours)} & 27.72 & 82.50 & 3.07 & \underline{29.08} & 33 & 7 \\
InternVL3.5 & \textbf{IP-M3ID (Ours)} & 27.48 & 85.45 & 4.37 & \textbf{29.41} & 47 & 8 \\
Qwen3-VL & Direct & 53.92 & 50.00 & 0.27 & 53.96 & 1 & 1 \\
Qwen3-VL & Vanilla VCD & 56.97 & 0.00 & 0.00 & \underline{56.97} & 0 & 0 \\
Qwen3-VL & Vanilla M3ID & 53.71 & 50.00 & 0.27 & 53.75 & 1 & 1 \\
Qwen3-VL & DoLa & 51.82 & 100.00 & 0.27 & 51.86 & 1 & 0 \\
Qwen3-VL & DeCo & 53.71 & 0.00 & 0.00 & 53.71 & 0 & 0 \\
Qwen3-VL & SID & 48.14 & 0.00 & 0.00 & 48.14 & 0 & 0 \\
Qwen3-VL & CDA & 2.76 & 0.00 & 0.00 & 2.76 & 0 & 1 \\
Qwen3-VL & \textbf{IP-VCD (Ours)} & 53.01 & 36.80 & 46.70 & \textbf{60.02} & 170 & 292 \\
Qwen3-VL & \textbf{IP-M3ID (Ours)} & 49.83 & 37.78 & 41.21 & 56.02 & 150 & 247 \\
\bottomrule

\end{tabular}}
\caption{Full Food-101 results on nine models (part 1 of 2).
Each method is evaluated on 2,424 questions; rates are percentages and the two abstention columns are counts.}
\label{tab:foodfull}
\end{table*}

\addtocounter{table}{-1}
\begin{table*}[tp]
\centering
\footnotesize
\setlength{\tabcolsep}{4pt}
\resizebox{0.98\textwidth}{!}{%
\begin{tabular}{llrrrrrr}
\toprule
Model & Method & Overall accuracy & Abstention precision & Abstention recall & Decision utility & Supported abstentions & Unsupported abstentions \\
\midrule
Qwen3.5 & Direct & 51.32 & 0.00 & 0.00 & 51.32 & 0 & 0 \\
Qwen3.5 & Vanilla VCD & 51.90 & 0.00 & 0.00 & \underline{51.90} & 0 & 0 \\
Qwen3.5 & Vanilla M3ID & 50.33 & 0.00 & 0.00 & 50.33 & 0 & 0 \\
Qwen3.5 & DoLa & 51.69 & 0.00 & 0.00 & 51.69 & 0 & 0 \\
Qwen3.5 & DeCo & 51.16 & 0.00 & 0.00 & 51.16 & 0 & 0 \\
Qwen3.5 & CDA & 26.28 & 0.65 & 0.21 & 26.32 & 1 & 152 \\
Qwen3.5 & \textbf{IP-VCD (Ours)} & 51.90 & 46.81 & 4.53 & \textbf{52.81} & 22 & 25 \\
Qwen3.5 & \textbf{IP-M3ID (Ours)} & 49.92 & 45.71 & 9.88 & \underline{51.90} & 48 & 57 \\
Gemma-3 & Direct & 54.79 & 0.00 & 0.00 & \textbf{54.79} & 0 & 0 \\
Gemma-3 & Vanilla VCD & 53.71 & 100.00 & 0.24 & 53.75 & 1 & 0 \\
Gemma-3 & Vanilla M3ID & 54.46 & 0.00 & 0.00 & 54.46 & 0 & 0 \\
Gemma-3 & DoLa & 54.29 & 0.00 & 0.00 & 54.29 & 0 & 0 \\
Gemma-3 & DeCo & 54.58 & 0.00 & 0.00 & 54.58 & 0 & 0 \\
Gemma-3 & SID & 51.94 & 0.00 & 0.00 & 51.94 & 0 & 0 \\
Gemma-3 & CDA & 29.74 & 16.67 & 0.71 & 29.87 & 3 & 15 \\
Gemma-3 & \textbf{IP-VCD (Ours)} & 53.05 & 100.00 & 0.24 & 53.09 & 1 & 0 \\
Gemma-3 & \textbf{IP-M3ID (Ours)} & 54.62 & 0.00 & 0.00 & \underline{54.62} & 0 & 0 \\
OneVision & Direct & 30.57 & 0.00 & 0.00 & 30.57 & 0 & 0 \\
OneVision & Vanilla VCD & 35.02 & 0.00 & 0.00 & \underline{35.02} & 0 & 0 \\
OneVision & Vanilla M3ID & 32.05 & 0.00 & 0.00 & 32.05 & 0 & 0 \\
OneVision & DoLa & 18.94 & 0.00 & 0.00 & 18.94 & 0 & 0 \\
OneVision & DeCo & 31.27 & 0.00 & 0.00 & 31.27 & 0 & 0 \\
OneVision & SID & 33.75 & 0.00 & 0.00 & 33.75 & 0 & 0 \\
OneVision & CDA & 25.87 & 100.00 & 0.12 & 25.91 & 1 & 0 \\
OneVision & \textbf{IP-VCD (Ours)} & 36.39 & 75.00 & 1.11 & \textbf{36.76} & 9 & 3 \\
OneVision & \textbf{IP-M3ID (Ours)} & 30.32 & 41.46 & 2.09 & 31.02 & 17 & 24 \\
Phi-3.5-Vision & Direct & 35.73 & 0.00 & 0.00 & 35.73 & 0 & 0 \\
Phi-3.5-Vision & Vanilla VCD & 37.79 & 0.00 & 0.00 & \underline{37.79} & 0 & 0 \\
Phi-3.5-Vision & Vanilla M3ID & 35.56 & 0.00 & 0.00 & 35.56 & 0 & 0 \\
Phi-3.5-Vision & DoLa & 30.61 & 0.00 & 0.00 & 30.61 & 0 & 0 \\
Phi-3.5-Vision & DeCo & 29.08 & 0.00 & 0.00 & 29.08 & 0 & 0 \\
Phi-3.5-Vision & SID & 35.89 & 0.00 & 0.00 & 35.89 & 0 & 0 \\
Phi-3.5-Vision & CDA & 27.48 & 100.00 & 0.70 & 27.72 & 6 & 0 \\
Phi-3.5-Vision & \textbf{IP-VCD (Ours)} & 36.88 & 69.23 & 5.26 & \textbf{38.74} & 45 & 20 \\
Phi-3.5-Vision & \textbf{IP-M3ID (Ours)} & 33.50 & 48.31 & 5.02 & 35.27 & 43 & 46 \\
\bottomrule

\end{tabular}}
\caption{Full Food-101 results on nine models (part 2 of 2).
Each method is evaluated on 2,424 questions; rates are percentages and the two abstention columns are counts.}
\label{tab:foodfull2}
\end{table*}

\subsection{VizWiz}

\begin{table*}[tp]
\centering
\footnotesize
\setlength{\tabcolsep}{3.4pt}
\resizebox{0.98\textwidth}{!}{%
\begin{tabular}{llrrrrrrr}
\toprule
Model & Method & N & Answer score & Abstention precision & Abstention recall & Decision utility & Supported abstentions & Unsupported abstentions \\
\midrule
Qwen2.5-VL & Direct & 512 & 26.27 & 92.86 & 23.49 & 33.89 & 39 & 3 \\
Qwen2.5-VL & Vanilla VCD & 512 & 24.59 & 91.30 & 12.65 & 28.69 & 21 & 2 \\
Qwen2.5-VL & Vanilla M3ID & 512 & 26.00 & 93.10 & 16.27 & 31.27 & 27 & 2 \\
Qwen2.5-VL & DoLa & 512 & 20.29 & 86.49 & 19.28 & 26.54 & 32 & 5 \\
Qwen2.5-VL & DeCo & 512 & 23.07 & 87.10 & 16.27 & 28.34 & 27 & 4 \\
Qwen2.5-VL & CDA & 512 & 13.12 & 41.77 & 59.64 & 32.46 & 99 & 138 \\
Qwen2.5-VL & \textbf{IP-VCD (Ours)} & 512 & 14.00 & 52.19 & 78.92 & \textbf{39.59} & 131 & 120 \\
Qwen2.5-VL & \textbf{IP-M3ID (Ours)} & 512 & 11.89 & 53.20 & 80.12 & \underline{37.87} & 133 & 117 \\
Qwen2.5-VL & Direct & 505 & 26.46 & 92.68 & 23.17 & 33.98 & - & - \\
Qwen2.5-VL & Vanilla VCD & 505 & 24.55 & 91.30 & 12.80 & 28.71 & - & - \\
Qwen2.5-VL & Vanilla M3ID & 505 & 26.12 & 92.86 & 15.85 & 31.27 & 26 & 2 \\
Qwen2.5-VL & DoLa & 505 & 20.40 & 88.89 & 19.51 & 26.73 & - & - \\
Qwen2.5-VL & DeCo & 505 & 23.21 & 87.10 & 16.46 & 28.55 & - & - \\
Qwen2.5-VL & SID & 505 & 26.48 & 94.12 & 9.76 & 29.64 & - & - \\
Qwen2.5-VL & CDA & 505 & 13.31 & 42.24 & 59.76 & 32.71 & - & - \\
Qwen2.5-VL & IP-VCD & 505 & 14.20 & 52.65 & 78.66 & \textbf{39.74} & - & - \\
Qwen2.5-VL & IP-M3ID & 505 & 12.06 & 53.47 & 79.88 & 38.00 & - & - \\
LLaVA-Mistral & Direct & 512 & 19.38 & 85.71 & 25.30 & 27.58 & 42 & 7 \\
LLaVA-Mistral & Vanilla VCD & 512 & 16.35 & 83.33 & 6.02 & 18.30 & 10 & 2 \\
LLaVA-Mistral & Vanilla M3ID & 512 & 19.90 & 86.05 & 22.29 & 27.13 & 37 & 6 \\
LLaVA-Mistral & DoLa & 512 & 18.91 & 83.72 & 21.69 & 25.94 & 36 & 7 \\
LLaVA-Mistral & DeCo & 512 & 12.70 & 85.71 & 18.07 & 18.55 & 30 & 5 \\
LLaVA-Mistral & SID & 512 & 16.09 & 87.50 & 16.87 & 21.56 & 28 & 4 \\
LLaVA-Mistral & CDA & 512 & 15.47 & 51.43 & 32.53 & 26.02 & 54 & 51 \\
LLaVA-Mistral & \textbf{IP-VCD (Ours)} & 512 & 14.51 & 53.88 & 79.52 & \textbf{40.29} & 132 & 113 \\
LLaVA-Mistral & \textbf{IP-M3ID (Ours)} & 512 & 10.70 & 42.98 & 90.36 & \underline{40.00} & 150 & 199 \\
MiniCPM-V2.6 & Direct & 512 & 29.18 & 94.74 & 10.84 & 32.70 & 18 & 1 \\
MiniCPM-V2.6 & Vanilla VCD & 512 & 25.18 & 87.50 & 4.22 & 26.54 & 7 & 1 \\
MiniCPM-V2.6 & Vanilla M3ID & 512 & 29.63 & 92.86 & 7.83 & 32.17 & 13 & 1 \\
MiniCPM-V2.6 & DoLa & 512 & 22.36 & 72.73 & 9.64 & 25.49 & 16 & 6 \\
MiniCPM-V2.6 & DeCo & 512 & 29.36 & 95.83 & 13.86 & 33.85 & 23 & 1 \\
MiniCPM-V2.6 & CDA & 512 & 12.54 & 49.72 & 54.22 & 30.12 & 90 & 91 \\
MiniCPM-V2.6 & \textbf{IP-VCD (Ours)} & 512 & 26.52 & 64.04 & 68.67 & \underline{48.79} & 114 & 64 \\
MiniCPM-V2.6 & \textbf{IP-M3ID (Ours)} & 512 & 26.35 & 56.50 & 75.90 & \textbf{50.96} & 126 & 97 \\
InternVL3.5 & Direct & 512 & 24.39 & 94.74 & 10.84 & 27.91 & 18 & 1 \\
InternVL3.5 & Vanilla VCD & 512 & 22.64 & 92.86 & 7.83 & 25.18 & 13 & 1 \\
InternVL3.5 & Vanilla M3ID & 512 & 24.28 & 93.75 & 9.04 & 27.21 & 15 & 1 \\
InternVL3.5 & DoLa & 512 & 22.19 & 100.00 & 4.82 & 23.75 & 8 & 0 \\
InternVL3.5 & DeCo & 512 & 21.39 & 100.00 & 6.63 & 23.54 & 11 & 0 \\
InternVL3.5 & SID & 512 & 23.59 & 92.31 & 7.23 & 25.94 & 12 & 1 \\
InternVL3.5 & CDA & 512 & 19.96 & 68.49 & 30.12 & 29.73 & 50 & 23 \\
InternVL3.5 & \textbf{IP-VCD (Ours)} & 512 & 16.43 & 69.42 & 50.60 & \underline{32.83} & 84 & 37 \\
InternVL3.5 & \textbf{IP-M3ID (Ours)} & 512 & 17.13 & 69.85 & 57.23 & \textbf{35.68} & 95 & 41 \\
\bottomrule

\end{tabular}}
\caption{VizWiz results on nine models and the SID subset comparison (part 1 of 2).
$N$ is the number of evaluated questions; rates are percentages and the two abstention columns are counts.
For each model, results on the full 512 questions are listed first, followed by the applicable 505-question common subset; the subset construction is described in \S\ref{app:sid}.}
\label{tab:vizfull}
\end{table*}

\addtocounter{table}{-1}
\begin{table*}[tp]
\centering
\footnotesize
\setlength{\tabcolsep}{3.4pt}
\resizebox{0.98\textwidth}{!}{%
\begin{tabular}{llrrrrrrr}
\toprule
Model & Method & N & Answer score & Abstention precision & Abstention recall & Decision utility & Supported abstentions & Unsupported abstentions \\
\midrule
Qwen3-VL & Direct & 512 & 37.52 & 96.43 & 16.27 & 42.79 & 27 & 1 \\
Qwen3-VL & Vanilla VCD & 512 & 35.37 & 96.15 & 15.06 & 40.25 & 25 & 1 \\
Qwen3-VL & Vanilla M3ID & 512 & 37.52 & 96.30 & 15.66 & 42.60 & 26 & 1 \\
Qwen3-VL & DoLa & 512 & 36.23 & 83.87 & 15.66 & 41.31 & 26 & 5 \\
Qwen3-VL & DeCo & 512 & 37.17 & 92.59 & 15.06 & 42.05 & 25 & 2 \\
Qwen3-VL & CDA & 512 & 17.52 & 62.09 & 57.23 & 36.07 & 95 & 58 \\
Qwen3-VL & \textbf{IP-VCD (Ours)} & 512 & 27.21 & 66.86 & 70.48 & \underline{50.06} & 117 & 58 \\
Qwen3-VL & \textbf{IP-M3ID (Ours)} & 512 & 27.64 & 59.45 & 77.71 & \textbf{52.83} & 129 & 88 \\
Qwen3-VL & Direct & 505 & 37.74 & 96.43 & 16.46 & 43.09 & - & - \\
Qwen3-VL & Vanilla VCD & 505 & 35.60 & 96.15 & 15.24 & 40.55 & - & - \\
Qwen3-VL & Vanilla M3ID & 505 & 37.74 & 96.30 & 15.85 & 42.89 & 26 & 1 \\
Qwen3-VL & DoLa & 505 & 36.50 & 83.87 & 15.85 & 41.64 & - & - \\
Qwen3-VL & DeCo & 505 & 37.39 & 92.59 & 15.24 & 42.34 & - & - \\
Qwen3-VL & SID & 505 & 37.01 & 92.86 & 15.85 & 42.16 & - & - \\
Qwen3-VL & CDA & 505 & 17.70 & 62.91 & 57.93 & 36.51 & - & - \\
Qwen3-VL & IP-VCD & 505 & 27.39 & 66.86 & 70.12 & 50.16 & - & - \\
Qwen3-VL & IP-M3ID & 505 & 28.02 & 60.19 & 77.44 & \textbf{53.17} & - & - \\
Qwen3.5 & Direct & 512 & 37.15 & 100.00 & 0.60 & 37.34 & 1 & 0 \\
Qwen3.5 & Vanilla VCD & 512 & 32.19 & 100.00 & 3.01 & 33.16 & 5 & 0 \\
Qwen3.5 & Vanilla M3ID & 512 & 37.05 & 100.00 & 0.60 & 37.25 & 1 & 0 \\
Qwen3.5 & DoLa & 512 & 36.88 & 100.00 & 0.60 & 37.07 & 1 & 0 \\
Qwen3.5 & DeCo & 512 & 35.51 & 100.00 & 3.01 & 36.48 & 5 & 0 \\
Qwen3.5 & CDA & 512 & 22.21 & 49.02 & 15.06 & 27.09 & 25 & 26 \\
Qwen3.5 & \textbf{IP-VCD (Ours)} & 512 & 27.13 & 71.43 & 48.19 & \underline{42.75} & 80 & 32 \\
Qwen3.5 & \textbf{IP-M3ID (Ours)} & 512 & 30.06 & 66.94 & 50.00 & \textbf{46.27} & 83 & 41 \\
Gemma-3 & Direct & 512 & 14.39 & 62.50 & 3.01 & 15.37 & 5 & 3 \\
Gemma-3 & Vanilla VCD & 512 & 9.39 & 52.00 & 7.83 & 11.93 & 13 & 12 \\
Gemma-3 & Vanilla M3ID & 512 & 12.93 & 72.73 & 4.82 & 14.49 & 8 & 3 \\
Gemma-3 & DoLa & 512 & 8.96 & 84.62 & 6.63 & 11.11 & 11 & 2 \\
Gemma-3 & DeCo & 512 & 8.59 & 80.00 & 4.82 & 10.16 & 8 & 2 \\
Gemma-3 & SID & 512 & 12.64 & 66.67 & 4.82 & 14.20 & 8 & 4 \\
Gemma-3 & CDA & 512 & 4.75 & 53.17 & 40.36 & 17.83 & 67 & 59 \\
Gemma-3 & \textbf{IP-VCD (Ours)} & 512 & 11.66 & 65.66 & 39.16 & \textbf{24.36} & 65 & 34 \\
Gemma-3 & \textbf{IP-M3ID (Ours)} & 512 & 11.54 & 73.81 & 37.35 & \underline{23.65} & 62 & 22 \\
OneVision & Direct & 512 & 30.61 & 100.00 & 7.23 & 32.95 & 12 & 0 \\
OneVision & Vanilla VCD & 512 & 25.53 & 100.00 & 4.82 & 27.09 & 8 & 0 \\
OneVision & Vanilla M3ID & 512 & 30.08 & 100.00 & 10.24 & 33.40 & 17 & 0 \\
OneVision & DoLa & 512 & 17.40 & 70.59 & 7.23 & 19.75 & 12 & 5 \\
OneVision & DeCo & 512 & 29.65 & 100.00 & 4.82 & 31.21 & 8 & 0 \\
OneVision & SID & 512 & 26.37 & 100.00 & 8.43 & 29.10 & 14 & 0 \\
OneVision & CDA & 512 & 26.43 & 79.49 & 18.67 & 32.48 & 31 & 8 \\
OneVision & \textbf{IP-VCD (Ours)} & 512 & 25.84 & 76.70 & 47.59 & \textbf{41.27} & 79 & 24 \\
OneVision & \textbf{IP-M3ID (Ours)} & 512 & 26.80 & 72.45 & 42.77 & \underline{40.66} & 71 & 27 \\
Phi-3.5-Vision & Direct & 512 & 21.88 & 71.64 & 28.92 & 31.25 & 48 & 19 \\
Phi-3.5-Vision & Vanilla VCD & 512 & 17.68 & 73.81 & 18.67 & 23.73 & 31 & 11 \\
Phi-3.5-Vision & Vanilla M3ID & 512 & 21.88 & 71.43 & 21.08 & 28.71 & 35 & 14 \\
Phi-3.5-Vision & DoLa & 512 & 18.26 & 67.14 & 28.31 & 27.44 & 47 & 23 \\
Phi-3.5-Vision & DeCo & 512 & 16.97 & 72.46 & 30.12 & 26.74 & 50 & 19 \\
Phi-3.5-Vision & SID & 512 & 17.89 & 78.95 & 18.07 & 23.75 & 30 & 8 \\
Phi-3.5-Vision & CDA & 512 & 17.70 & 63.38 & 54.22 & 35.27 & 90 & 52 \\
Phi-3.5-Vision & \textbf{IP-VCD (Ours)} & 512 & 17.71 & 66.67 & 63.86 & \underline{38.42} & 106 & 53 \\
Phi-3.5-Vision & \textbf{IP-M3ID (Ours)} & 512 & 19.36 & 60.61 & 60.24 & \textbf{38.89} & 100 & 65 \\
\bottomrule

\end{tabular}}
\caption{VizWiz results on nine models and the SID subset comparison (part 2 of 2).
$N$ is the number of evaluated questions; rates are percentages and the two abstention columns are counts.}
\label{tab:vizfull2}
\end{table*}

\subsection{Gain Decomposition Relative to Vanilla M3ID}

\begin{table*}[tp]
\centering
\footnotesize
\setlength{\tabcolsep}{4pt}
\resizebox{0.85\textwidth}{!}{%
\begin{tabular}{lrrrrrr}
\toprule
Model & IP-VCD $\Delta$answer-score sum & IP-VCD $\Delta$TP & IP-VCD $\Delta$total credit & IP-M3ID $\Delta$answer-score sum & IP-M3ID $\Delta$TP & IP-M3ID $\Delta$total credit \\
\midrule
Qwen2.5-VL & -61.4 & 104 & 42.6 & -72.2 & 106 & 33.8 \\
LLaVA-Mistral & -27.6 & 95 & 67.4 & -47.1 & 113 & 65.9 \\
MiniCPM-V2.6 & -15.9 & 101 & 85.1 & -16.8 & 113 & 96.2 \\
InternVL3.5 & -40.2 & 69 & 28.8 & -36.6 & 80 & 43.4 \\
Qwen3-VL & -52.8 & 91 & 38.2 & -50.6 & 103 & 52.4 \\
Qwen3.5 & -50.8 & 79 & 28.2 & -35.8 & 82 & 46.2 \\
Gemma-3 & -6.5 & 57 & 50.5 & -7.1 & 54 & 46.9 \\
OneVision & -21.7 & 62 & 40.3 & -16.8 & 54 & 37.2 \\
Phi-3.5-Vision & -21.3 & 71 & 49.7 & -12.9 & 65 & 52.1 \\
\bottomrule

\end{tabular}}
\caption{Changes in the utility numerator of the two IP methods relative to vanilla M3ID on VizWiz.
The answer score is the sum of soft consensus scores over the 512 questions; the change in total credit equals the change in answer score plus the change in TP.
The DU difference (percentage points) is the change in total credit multiplied by 100/512.}
\label{tab:m3idnumerators}
\end{table*}

Changes in answer score use the sum of soft consensus scores; per-question response transitions on Food-101 are given in \S\ref{app:d3} and \S\ref{app:d4}.

\section{Candidate Mechanism and Response Changes}
\label{app:mechanism}

\subsection{Candidate Admission}
\label{app:d1}

The candidate analysis selects 101 images for each of the nine models and measures candidate scores before the first token is generated, 909 positions in total; 875 of them are used for pairwise comparisons between candidate tokens representing abstention and a specific answer.
Table~\ref{tab:eligibility} lists, per model, the number of abstention candidates that pass unguided filtering and guided filtering, and the number of pairs in which both candidates pass guided filtering.
For the 199 pairs in which both pass, we compare the score change of the abstention candidate relative to the answer candidate caused by appending the instruction, count the pairs in which this change is larger under the reference condition than under the original image, and compute the proportion with 199 as the denominator.
Abstention in complete responses is judged from the semantics of the response.

\begin{table}[htbp]
\centering
\footnotesize
\setlength{\tabcolsep}{3pt}
\resizebox{\columnwidth}{!}{%
\begin{tabular}{lccccc}
\toprule
Model & Proxy pairs & Unguided admission & Guided admission & Both admitted & Both admitted, $\Delta_r>\Delta_c$ \\
\midrule
Qwen2.5-VL & 96 & 0 & 84 & 52 & 44 \\
LLaVA-Mistral & 94 & 0 & 67 & 57 & 44 \\
MiniCPM-V2.6 & 94 & 0 & 20 & 12 & 7 \\
InternVL3.5 & 101 & 0 & 5 & 2 & 0 \\
Qwen3-VL & 101 & 0 & 39 & 16 & 13 \\
Qwen3.5 & 99 & 0 & 60 & 56 & 37 \\
Gemma-3 & 88 & 0 & 0 & 0 & 0 \\
OneVision & 101 & 0 & 1 & 1 & 0 \\
Phi-3.5-Vision & 101 & 0 & 6 & 3 & 2 \\
\bottomrule

\end{tabular}}
\caption{First-token proxy candidate measurements on nine models.
All columns are counts of candidate pairs.}
\label{tab:eligibility}
\end{table}

\subsection{Score Margins Under the Four Conditions}
\label{app:d2}

The beef carpaccio example fixes the image, generated prefix, and noise, and compares the scores of the abstention first token and the Toast first token.
The ``abstention minus answer'' margins of guided VCD and IP-VCD are $-0.1875$ and $0.421875$.
Equations~\eqref{eq:B.52} to~\eqref{eq:B.56} give the term-by-term substitution of the four condition scores and the analytic crossing point between the two combinations.

\subsection{Paired Changes on the Abstention-Labeled Set}
\label{app:d3}

The mechanism comparison uses one fixed instruction and compares, question by question, the complete responses of guided VCD and IP-VCD on all 949 abstention-labeled inputs of LLaVA-Mistral.
This comparison uses the same instruction with the UNKNOWN marker as \S\ref{sec:expsetup} and \S\ref{sec:ablation} of the main text, while the cross-task main results use the development-selected configuration.
Table~\ref{tab:transitions} lists the seven response transitions observed between the two methods, 949 questions in total; 893 questions keep the same behavior category, and 56 move among incorrect answer, correct answer, and abstention.
IP-VCD raises the number of abstentions on these abstention-labeled inputs from 33 to 66, and both methods answer 852 of all 2,424 questions correctly.

\begin{table}[htbp]
\centering
\footnotesize
\setlength{\tabcolsep}{6pt}
\begin{tabular}{llr}
\toprule
Guided VCD & IP-VCD & Count \\
\midrule
Supported abstention & Supported abstention & 29 \\
Supported abstention & Incorrect answer & 4 \\
Correct answer & Correct answer & 33 \\
Correct answer & Incorrect answer & 9 \\
Incorrect answer & Supported abstention & 37 \\
Incorrect answer & Correct answer & 6 \\
Incorrect answer & Incorrect answer & 831 \\
\bottomrule

\end{tabular}
\caption{Full paired transitions on the 949 inputs.
The first two columns are the response categories of guided VCD and IP-VCD.}
\label{tab:transitions}
\end{table}

\subsection{Correct Answers and Response Transitions on the Abstention-Labeled Set Under the Main-Result Configuration}
\label{app:d4}

On the abstention-labeled set of each model, we pair vanilla VCD and the IP configuration used in the main results question by question.
The set is determined by the frozen abstention labels of the same model, and both methods use the same inputs.
Table~\ref{tab:refposcounts} gives the numbers of correct answers and supported abstentions, and Table~\ref{tab:refpostransitions} lists all nine response transitions.
Both IP methods cover all nine models.

These paired records show two kinds of change: incorrect answers turned into correct answers, and incorrect answers turned into supported abstentions.
For example, IP-VCD on InternVL3.5 turns 43 incorrect answers into correct answers, raising the number of correct answers on the abstention-labeled set from 8 to 44; IP-VCD on LLaVA-Mistral turns 556 incorrect answers into supported abstentions.
The net number of correct answers is determined by transitions in all directions; for LLaVA-Mistral, correct answers on this set fall from 42 to 18.

\begin{table}[htbp]
\centering
\footnotesize
\setlength{\tabcolsep}{3.6pt}
\resizebox{\columnwidth}{!}{%
\begin{tabular}{llrrrrrr}
\toprule
Model & IP method & Abstention-labeled N & Vanilla VCD correct & IP correct & Change in correct & Vanilla VCD abstain & IP abstain \\
\midrule
Qwen2.5-VL & IP-VCD & 383 & 3 & 4 & +1 & 0 & 214 \\
Qwen2.5-VL & IP-M3ID & 383 & 3 & 5 & +2 & 0 & 143 \\
LLaVA-Mistral & IP-VCD & 949 & 42 & 18 & -24 & 0 & 584 \\
LLaVA-Mistral & IP-M3ID & 949 & 42 & 5 & -37 & 0 & 735 \\
MiniCPM-V2.6 & IP-VCD & 597 & 8 & 6 & -2 & 1 & 251 \\
MiniCPM-V2.6 & IP-M3ID & 597 & 8 & 4 & -4 & 1 & 335 \\
InternVL3.5 & IP-VCD & 1075 & 8 & 44 & +36 & 1 & 33 \\
InternVL3.5 & IP-M3ID & 1075 & 8 & 29 & +21 & 1 & 47 \\
Qwen3-VL & IP-VCD & 364 & 4 & 0 & -4 & 0 & 170 \\
Qwen3-VL & IP-M3ID & 364 & 4 & 1 & -3 & 0 & 150 \\
Qwen3.5 & IP-VCD & 486 & 16 & 25 & +9 & 0 & 22 \\
Qwen3.5 & IP-M3ID & 486 & 16 & 9 & -7 & 0 & 48 \\
Gemma-3 & IP-VCD & 422 & 14 & 9 & -5 & 1 & 1 \\
Gemma-3 & IP-M3ID & 422 & 14 & 6 & -8 & 1 & 0 \\
OneVision & IP-VCD & 814 & 36 & 32 & -4 & 0 & 9 \\
OneVision & IP-M3ID & 814 & 36 & 7 & -29 & 0 & 17 \\
Phi-3.5-Vision & IP-VCD & 856 & 49 & 33 & -16 & 0 & 45 \\
Phi-3.5-Vision & IP-M3ID & 856 & 49 & 20 & -29 & 0 & 43 \\
\bottomrule

\end{tabular}}
\caption{Answers and abstentions on the abstention-labeled set.
Units are questions; the change in correct answers is IP minus vanilla VCD.}
\label{tab:refposcounts}
\end{table}

\begin{table*}[tp]
\centering
\footnotesize
\setlength{\tabcolsep}{4.4pt}
\resizebox{\textwidth}{!}{%
\begin{tabular}{llccccccccc}
\toprule
Model & IP method & $C\rightarrow C$ & $C\rightarrow W$ & $C\rightarrow A$ & $W\rightarrow C$ & $W\rightarrow W$ & $W\rightarrow A$ & $A\rightarrow C$ & $A\rightarrow W$ & $A\rightarrow A$ \\
\midrule
Qwen2.5-VL & IP-VCD & 2 & 0 & 1 & 2 & 165 & 213 & 0 & 0 & 0 \\
Qwen2.5-VL & IP-M3ID & 0 & 2 & 1 & 5 & 233 & 142 & 0 & 0 & 0 \\
LLaVA-Mistral & IP-VCD & 10 & 4 & 28 & 8 & 343 & 556 & 0 & 0 & 0 \\
LLaVA-Mistral & IP-M3ID & 4 & 6 & 32 & 1 & 203 & 703 & 0 & 0 & 0 \\
MiniCPM-V2.6 & IP-VCD & 4 & 2 & 2 & 2 & 338 & 248 & 0 & 0 & 1 \\
MiniCPM-V2.6 & IP-M3ID & 1 & 2 & 5 & 3 & 256 & 329 & 0 & 0 & 1 \\
InternVL3.5 & IP-VCD & 1 & 7 & 0 & 43 & 991 & 32 & 0 & 0 & 1 \\
InternVL3.5 & IP-M3ID & 1 & 7 & 0 & 28 & 992 & 46 & 0 & 0 & 1 \\
Qwen3-VL & IP-VCD & 0 & 2 & 2 & 0 & 192 & 168 & 0 & 0 & 0 \\
Qwen3-VL & IP-M3ID & 0 & 3 & 1 & 1 & 210 & 149 & 0 & 0 & 0 \\
Qwen3.5 & IP-VCD & 11 & 5 & 0 & 14 & 434 & 22 & 0 & 0 & 0 \\
Qwen3.5 & IP-M3ID & 3 & 10 & 3 & 6 & 419 & 45 & 0 & 0 & 0 \\
Gemma-3 & IP-VCD & 4 & 10 & 0 & 5 & 402 & 0 & 0 & 0 & 1 \\
Gemma-3 & IP-M3ID & 2 & 12 & 0 & 4 & 403 & 0 & 0 & 1 & 0 \\
OneVision & IP-VCD & 20 & 16 & 0 & 12 & 757 & 9 & 0 & 0 & 0 \\
OneVision & IP-M3ID & 3 & 33 & 0 & 4 & 757 & 17 & 0 & 0 & 0 \\
Phi-3.5-Vision & IP-VCD & 25 & 24 & 0 & 8 & 754 & 45 & 0 & 0 & 0 \\
Phi-3.5-Vision & IP-M3ID & 9 & 38 & 2 & 11 & 755 & 41 & 0 & 0 & 0 \\
\bottomrule

\end{tabular}}
\caption{Full response transitions on the abstention-labeled set.
C, W, and A denote correct answer, incorrect answer, and supported abstention; arrows point from vanilla VCD to the corresponding IP method.
Units are questions, and the nine cells of each row sum to the number of abstention-labeled inputs in Table~\ref{tab:refposcounts}.}
\label{tab:refpostransitions}
\end{table*}

This section uses the main-result configuration; Table~\ref{tab:transitions} is the mechanism comparison under the same UNKNOWN instruction.

These transitions are counted over model-question paired records: the same image can appear for several models, and the abstention-labeled set is defined separately for each model.

The net number of correct answers is determined by transitions in all directions.
Table~\ref{tab:refpostransitions} lists corrections, incorrect-to-abstention transitions, and reverse transitions together.

Within the abstention-labeled set $H$, correct answers and supported abstentions both score 1, and incorrect answers score 0.
Let $n_{X\to Y}$ be the number of transitions from state $X$ under vanilla VCD to state $Y$ under IP.
The contribution of this set to the change in full-task utility is $(n_{W\to C}+n_{W\to A}-n_{C\to W}-n_{A\to W})/N$, where $N=2{,}424$ is the number of questions in the model's full evaluation and the transition counts come from its abstention-labeled set.
Switching between a correct answer and a supported abstention changes answer coverage but keeps this part of the decision credit.
The abstention labels describe the knowledge gaps the base model shows under class ranking and independent attempts; correct responses after decoding can still appear within them, so correction and abstention are two observable paths to gains on the same set.

\section{Component Comparisons and Baseline Implementations}
\label{app:components}

The guided variants serve as mechanism ablations.
The component comparison uses the instruction configuration already fixed for each model and its IP method, and separates two kinds of behavioral change through the numbers of correct answers and supported abstentions.
Guided Direct generates the response directly from the original image with the task and abstention instructions.
Guided VCD uses the instruction responses of both branches in the contrast; the VCD original-only variant amplifies the instruction effect of the original-image branch as the contrast strength grows; IP-VCD keeps the instruction increment on the original image at its original size.

For M3ID, moving from the guided or original-only form to IP-M3ID also changes the distribution on which the confidence gate depends and the text length on which the position weight depends.
The per-model results therefore reflect generation behavior after the instruction response, the gate, and the position weight change together.

\begin{table*}[tp]
\centering
\footnotesize
\setlength{\tabcolsep}{4.6pt}
\resizebox{0.96\textwidth}{!}{%
\begin{tabular}{llrrrrr}
\toprule
Model & Component variant & Correct answers & Supported abstentions & All abstentions & Unsupported abstentions & Decision utility (\%) \\
\midrule
Qwen2.5-VL & Guided VCD & 920 & 192 & 593 & 401 & 45.87 \\
Qwen2.5-VL & VCD original-only variant & 59 & 383 & 2327 & 1944 & 18.23 \\
Qwen2.5-VL & IP-VCD & 853 & 214 & 732 & 518 & 44.02 \\
Qwen2.5-VL & Guided M3ID & 883 & 254 & 962 & 708 & 46.91 \\
Qwen2.5-VL & M3ID original-only variant & 883 & 254 & 962 & 708 & 46.91 \\
Qwen2.5-VL & IP-M3ID & 1022 & 143 & 565 & 422 & 48.06 \\
LLaVA-Mistral & Guided VCD & 711 & 570 & 1054 & 484 & 52.85 \\
LLaVA-Mistral & VCD original-only variant & 33 & 939 & 2373 & 1434 & 40.10 \\
LLaVA-Mistral & IP-VCD & 688 & 584 & 1101 & 517 & 52.48 \\
LLaVA-Mistral & Guided M3ID & 228 & 887 & 2068 & 1181 & 46.00 \\
LLaVA-Mistral & M3ID original-only variant & 227 & 887 & 2068 & 1181 & 45.96 \\
LLaVA-Mistral & IP-M3ID & 329 & 735 & 1738 & 1003 & 43.89 \\
MiniCPM-V2.6 & Guided VCD & 873 & 240 & 506 & 266 & 45.92 \\
MiniCPM-V2.6 & VCD original-only variant & 548 & 465 & 1453 & 988 & 41.79 \\
MiniCPM-V2.6 & IP-VCD & 903 & 251 & 537 & 286 & 47.61 \\
MiniCPM-V2.6 & Guided M3ID & 743 & 361 & 958 & 597 & 45.54 \\
MiniCPM-V2.6 & M3ID original-only variant & 721 & 385 & 1055 & 670 & 45.63 \\
MiniCPM-V2.6 & IP-M3ID & 758 & 335 & 876 & 541 & 45.09 \\
InternVL3.5 & Guided VCD & 654 & 26 & 36 & 10 & 28.05 \\
InternVL3.5 & VCD original-only variant & 645 & 238 & 381 & 143 & 36.43 \\
InternVL3.5 & IP-VCD & 672 & 33 & 40 & 7 & 29.08 \\
InternVL3.5 & Guided M3ID & 647 & 33 & 38 & 5 & 28.05 \\
InternVL3.5 & M3ID original-only variant & 647 & 40 & 53 & 13 & 28.34 \\
InternVL3.5 & IP-M3ID & 666 & 47 & 55 & 8 & 29.41 \\
Qwen3-VL & Guided VCD & 1306 & 147 & 376 & 229 & 59.94 \\
Qwen3-VL & VCD original-only variant & 1023 & 266 & 995 & 729 & 53.18 \\
Qwen3-VL & IP-VCD & 1285 & 170 & 462 & 292 & 60.02 \\
Qwen3-VL & Guided M3ID & 1200 & 156 & 429 & 273 & 55.94 \\
Qwen3-VL & M3ID original-only variant & 1198 & 158 & 444 & 286 & 55.94 \\
Qwen3-VL & IP-M3ID & 1208 & 150 & 397 & 247 & 56.02 \\
Qwen3.5 & Guided VCD & 1249 & 13 & 18 & 5 & 52.06 \\
Qwen3.5 & VCD original-only variant & 1167 & 168 & 351 & 183 & 55.07 \\
Qwen3.5 & IP-VCD & 1258 & 22 & 47 & 25 & 52.81 \\
Qwen3.5 & Guided M3ID & 1188 & 64 & 142 & 78 & 51.65 \\
Qwen3.5 & M3ID original-only variant & 980 & 326 & 855 & 529 & 53.88 \\
Qwen3.5 & IP-M3ID & 1210 & 48 & 105 & 57 & 51.90 \\
Gemma-3 & Guided VCD & 1286 & 1 & 2 & 1 & 53.09 \\
Gemma-3 & VCD original-only variant & 1279 & 1 & 3 & 2 & 52.81 \\
Gemma-3 & IP-VCD & 1286 & 1 & 1 & 0 & 53.09 \\
Gemma-3 & Guided M3ID & 1322 & 0 & 0 & 0 & 54.54 \\
Gemma-3 & M3ID original-only variant & 1322 & 0 & 0 & 0 & 54.54 \\
Gemma-3 & IP-M3ID & 1324 & 0 & 0 & 0 & 54.62 \\
OneVision & Guided VCD & 850 & 5 & 6 & 1 & 35.27 \\
OneVision & VCD original-only variant & 804 & 116 & 202 & 86 & 37.95 \\
OneVision & IP-VCD & 882 & 9 & 12 & 3 & 36.76 \\
OneVision & Guided M3ID & 742 & 13 & 30 & 17 & 31.15 \\
OneVision & M3ID original-only variant & 717 & 117 & 254 & 137 & 34.41 \\
OneVision & IP-M3ID & 735 & 17 & 41 & 24 & 31.02 \\
Phi-3.5-Vision & Guided VCD & 891 & 38 & 52 & 14 & 38.33 \\
Phi-3.5-Vision & VCD original-only variant & 746 & 314 & 650 & 336 & 43.73 \\
Phi-3.5-Vision & IP-VCD & 894 & 45 & 65 & 20 & 38.74 \\
Phi-3.5-Vision & Guided M3ID & 818 & 51 & 91 & 40 & 35.85 \\
Phi-3.5-Vision & M3ID original-only variant & 822 & 53 & 93 & 40 & 36.10 \\
Phi-3.5-Vision & IP-M3ID & 812 & 43 & 89 & 46 & 35.27 \\
\bottomrule

\end{tabular}}
\caption{Behavioral composition of the VCD and M3ID component comparisons.
For each model, the VCD and M3ID components use the development-selected instruction configuration of the corresponding IP method, and each compares guided decoding, the original-only variant, and IP decoding; every construction is evaluated on 2,424 questions.
Correct answers and supported abstentions are counts, and DU is their sum divided by 2,424, in percent.}
\label{tab:vcdcomp}
\end{table*}

The VCD original-only variant amplifies the instruction increment of the original-image condition by a factor of $1+\alpha$ (Eq.~\eqref{eq:B.18}; $\alpha=1$ in this table).
On all nine models, it abstains more often than IP-VCD: 2,327 versus 732 for Qwen2.5-VL and 2,373 versus 1,101 for LLaVA-Mistral, with 1,944 versus 518 and 1,434 versus 517 unsupported abstentions.
This over-abstention lowers correct answers on these two models.
On InternVL3.5, Qwen3.5, OneVision, and Phi-3.5-Vision, the original-only variant reaches a higher DU than IP-VCD, showing that stronger abstention can also earn a higher utility on some abstention-labeled sets.
Instruction preservation sets the coefficient of the original-image instruction increment to 1; the final number of abstentions is also affected by the visual contrast term, the candidate support, and the subsequent generated prefix.

In Table~\ref{tab:vcdcomp}, IP-M3ID exceeds guided M3ID on five models and the M3ID original-only variant on four.
On LLaVA-Mistral, relative to guided M3ID, IP-M3ID answers 101 more questions correctly and makes 152 fewer supported abstentions, a net loss of 51 successful decisions.
On Qwen3.5, relative to the M3ID original-only variant, it answers 230 more questions correctly and makes 278 fewer supported abstentions, a net loss of 48 successful decisions.
The gate, the position weight, and the instruction effect jointly shape this trade-off; the coefficient fixed by instruction preservation and the coefficient that maximizes each model's final utility are two separate questions.

\begin{table}[htbp]
\centering
\footnotesize
\setlength{\tabcolsep}{3pt}
\resizebox{\columnwidth}{!}{%
\begin{tabular}{llrrrr}
\toprule
Model & Method & Overall accuracy & Abstention precision & Abstention recall & Decision utility \\
\midrule
Qwen2.5-VL & Guided Direct & 28.55 & 23.68 & 81.98 & 41.50 \\
Qwen2.5-VL & IP-VCD & 35.19 & 29.23 & 55.87 & 44.02 \\
LLaVA-Mistral & Guided Direct & 8.75 & 42.69 & 94.73 & 45.83 \\
LLaVA-Mistral & IP-VCD & 28.38 & 53.04 & 61.54 & 52.48 \\
MiniCPM-V2.6 & Guided Direct & 30.24 & 37.29 & 62.65 & 45.67 \\
MiniCPM-V2.6 & IP-VCD & 37.25 & 46.74 & 42.04 & 47.61 \\
Qwen3.5 & Guided Direct & 53.38 & 56.76 & 8.64 & 55.12 \\
Qwen3.5 & IP-VCD & 51.90 & 46.81 & 4.53 & 52.81 \\
Gemma-3 & Guided Direct & 54.21 & 100.00 & 0.24 & 54.25 \\
Gemma-3 & IP-VCD & 53.05 & 100.00 & 0.24 & 53.09 \\
\bottomrule

\end{tabular}}
\caption{Guided Direct and IP-VCD under the same instruction.
Five models; all four metrics are percentages.}
\label{tab:guideddirect}
\end{table}

\begin{table}[htbp]
\centering
\footnotesize
\setlength{\tabcolsep}{3pt}
\resizebox{\columnwidth}{!}{%
\begin{tabular}{llrrrr}
\toprule
Model & Abstention marker & Overall accuracy & Abstention precision & Abstention recall & DU \\
\midrule
\textbf{Qwen2.5-VL} & UNKNOWN & 25.12 & 22.78 & 84.33 & 38.45 \\
 & UNCLEAR & 39.65 & 28.84 & 51.44 & 47.77 \\
 & UNSURE & 28.55 & 23.68 & 81.98 & 41.50 \\
 & I cannot identify it & 35.81 & 26.26 & 66.84 & 46.37 \\
\textbf{LLaVA-Mistral} & UNKNOWN & 28.42 & 59.54 & 30.24 & 40.26 \\
 & UNCLEAR & 17.37 & 47.55 & 73.55 & 46.16 \\
 & UNSURE & 17.49 & 46.51 & 74.50 & 46.66 \\
 & I cannot identify it & 8.75 & 42.69 & 94.73 & 45.83 \\
\textbf{MiniCPM-V2.6} & UNKNOWN & 35.93 & 49.50 & 16.58 & 40.02 \\
 & UNCLEAR & 30.24 & 37.29 & 62.65 & 45.67 \\
 & UNSURE & 34.78 & 43.40 & 47.40 & 46.45 \\
 & I cannot identify it & 35.48 & 40.61 & 35.85 & 44.31 \\
\textbf{Qwen3.5} & UNKNOWN & 46.33 & 45.78 & 46.91 & 55.73 \\
 & UNCLEAR & 41.42 & 36.76 & 60.29 & 53.51 \\
 & UNSURE & 45.59 & 41.96 & 49.38 & 55.49 \\
 & I cannot identify it & 53.38 & 56.76 & 8.64 & 55.12 \\
\textbf{Gemma-3} & UNKNOWN & 54.66 & 0.00 & 0.00 & 54.66 \\
 & UNCLEAR & 54.21 & 100.00 & 0.24 & 54.25 \\
 & UNSURE & 54.37 & 100.00 & 0.24 & 54.41 \\
 & I cannot identify it & 54.21 & 100.00 & 0.24 & 54.25 \\
\bottomrule

\end{tabular}}
\caption{Guided Direct with four abstention wordings on five models.
Each wording is evaluated on the same 2,424 questions; the four metrics are percentages.}
\label{tab:phrasings}
\end{table}

Guided Direct has been run on these five models with four single-instruction conditions: UNKNOWN, UNCLEAR, UNSURE, and I cannot identify it.
Direct uses only the original-image prediction and has no reference branch, so there are 4 conditions per model.
The same-instruction comparison matched to the frozen IP configuration is given in Table~\ref{tab:guideddirect}.

When there is no abstention, abstention precision is shown as 0.00.

The four wordings change how the model trades off answering and abstention.
The DU ranges of guided Direct for Qwen2.5-VL, LLaVA-Mistral, MiniCPM-V2.6, Qwen3.5, and Gemma-3 are 9.32, 6.40, 6.43, 2.22, and 0.41 percentage points (computed from the utilities shown in the table with two decimals; from unrounded utilities they are 9.32, 6.39, 6.44, 2.23, and 0.41).
When LLaVA-Mistral switches from UNKNOWN to I cannot identify it, abstention recall rises from 30.24\% to 94.73\%, while the share of correct answers drops from 28.42\% to 8.75\%.
These changes show that adding or changing the abstention instruction alone is enough to change the final decisions, so Table~\ref{tab:guideddirect} uses guided Direct with the same instruction as the comparison, to separate the effect of the prompt from that of the decoding construction.

In the same-instruction comparison, IP-VCD raises DU on the first three models by 2.52, 6.65, and 1.94 percentage points (from the displayed values; computed from raw counts, the LLaVA-Mistral gain is 161/2,424, or 6.64 percentage points), and lowers it on Qwen3.5 and Gemma-3 by 2.31 and 1.16 percentage points.
The same decoding construction yields different net gains on different models, and the four-wording results also show that the models respond to instructions differently.
Gemma-3 abstains 0 or 1 time under each of the four guided Direct conditions; its 100\% abstention precision corresponds to a single abstention, and its recall below 0.25\% reflects limited coverage of abstention-labeled inputs.
The four-wording table shows this sensitivity, and the same-instruction comparison and the main results each use their own frozen configurations.

\begin{table}[htbp]
\centering
\footnotesize
\setlength{\tabcolsep}{3pt}
\resizebox{\columnwidth}{!}{%
\begin{tabular}{lcccccc}
\toprule
Model & Direct & Vanilla VCD & Guided Direct & Guided VCD & IP-VCD & IP-M3ID \\
\midrule
Qwen2.5-VL & 30.40 & 28.34 & 38.45 & 43.44 & 42.08 & 40.92 \\
LLaVA-Mistral & 34.16 & 38.24 & 40.26 & 36.51 & 37.87 & 40.18 \\
MiniCPM-V2.6 & 43.15 & 43.65 & 40.02 & 39.23 & 41.75 & 39.52 \\
InternVL3.5 & 17.08 & 19.14 & 27.15 & 27.15 & 27.89 & 26.28 \\
Qwen3-VL & 53.96 & 56.97 & 55.53 & 58.00 & 58.37 & 55.53 \\
Qwen3.5 & 51.32 & 51.90 & 55.73 & 50.25 & 49.63 & 51.90 \\
Gemma-3 & 54.79 & 53.75 & 54.66 & 52.68 & 52.76 & 54.62 \\
OneVision & 30.57 & 35.02 & 25.87 & 29.46 & 30.94 & 27.06 \\
Phi-3.5-Vision & 35.73 & 37.79 & 29.00 & 32.34 & 34.74 & 29.58 \\
\bottomrule

\end{tabular}}
\caption{Overview of the nine-model ablation under the unified UNKNOWN instruction.
Each method is evaluated on the same 2,424 questions, and values are DU (\%).
Direct and Vanilla VCD use the base instruction; the other four columns all use the same instruction with the UNKNOWN marker.
This table covers all nine models for Table 2 of the main text.}
\label{tab:ablationfull}
\end{table}

\begin{table*}[tp]
\centering
\footnotesize
\setlength{\tabcolsep}{4.4pt}
\resizebox{\textwidth}{!}{%
\begin{tabular}{llrrrrrrrrrr}
\toprule
Model & Method & $N_+$ & C & W & A & TP & FP & Overall accuracy & Abstention precision & Abstention recall & DU \\
\midrule
Qwen2.5-VL & Guided Direct & 383 & 609 & 397 & 1418 & 323 & 1095 & 25.12 & 22.78 & 84.33 & 38.45 \\
Qwen2.5-VL & Guided VCD & 383 & 809 & 846 & 769 & 244 & 525 & 33.37 & 31.73 & 63.71 & 43.44 \\
Qwen2.5-VL & IP-VCD & 383 & 769 & 797 & 858 & 251 & 607 & 31.72 & 29.25 & 65.54 & 42.08 \\
Qwen2.5-VL & IP-M3ID & 383 & 712 & 584 & 1128 & 280 & 848 & 29.37 & 24.82 & 73.11 & 40.92 \\
LLaVA-Mistral & Guided Direct & 949 & 689 & 1253 & 482 & 287 & 195 & 28.42 & 59.54 & 30.24 & 40.26 \\
LLaVA-Mistral & Guided VCD & 949 & 852 & 1527 & 45 & 33 & 12 & 35.15 & 73.33 & 3.48 & 36.51 \\
LLaVA-Mistral & IP-VCD & 949 & 852 & 1477 & 95 & 66 & 29 & 35.15 & 69.47 & 6.95 & 37.87 \\
LLaVA-Mistral & IP-M3ID & 949 & 735 & 1287 & 402 & 239 & 163 & 30.32 & 59.45 & 25.18 & 40.18 \\
MiniCPM-V2.6 & Guided Direct & 597 & 871 & 1353 & 200 & 99 & 101 & 35.93 & 49.50 & 16.58 & 40.02 \\
MiniCPM-V2.6 & Guided VCD & 597 & 925 & 1457 & 42 & 26 & 16 & 38.16 & 61.90 & 4.36 & 39.23 \\
MiniCPM-V2.6 & IP-VCD & 597 & 988 & 1402 & 34 & 24 & 10 & 40.76 & 70.59 & 4.02 & 41.75 \\
MiniCPM-V2.6 & IP-M3ID & 597 & 890 & 1408 & 126 & 68 & 58 & 36.72 & 53.97 & 11.39 & 39.52 \\
InternVL3.5 & Guided Direct & 1075 & 625 & 1762 & 37 & 33 & 4 & 25.78 & 89.19 & 3.07 & 27.15 \\
InternVL3.5 & Guided VCD & 1075 & 640 & 1761 & 23 & 18 & 5 & 26.40 & 78.26 & 1.67 & 27.15 \\
InternVL3.5 & IP-VCD & 1075 & 648 & 1743 & 33 & 28 & 5 & 26.73 & 84.85 & 2.60 & 27.89 \\
InternVL3.5 & IP-M3ID & 1075 & 615 & 1786 & 23 & 22 & 1 & 25.37 & 95.65 & 2.05 & 26.28 \\
Qwen3-VL & Guided Direct & 364 & 1158 & 749 & 517 & 188 & 329 & 47.77 & 36.36 & 51.65 & 55.53 \\
Qwen3-VL & Guided VCD & 364 & 1272 & 809 & 343 & 134 & 209 & 52.48 & 39.07 & 36.81 & 58.00 \\
Qwen3-VL & IP-VCD & 364 & 1278 & 793 & 353 & 137 & 216 & 52.72 & 38.81 & 37.64 & 58.37 \\
Qwen3-VL & IP-M3ID & 364 & 1163 & 773 & 488 & 183 & 305 & 47.98 & 37.50 & 50.27 & 55.53 \\
Qwen3.5 & Guided Direct & 486 & 1123 & 803 & 498 & 228 & 270 & 46.33 & 45.78 & 46.91 & 55.73 \\
Qwen3.5 & Guided VCD & 486 & 1176 & 1161 & 87 & 42 & 45 & 48.51 & 48.28 & 8.64 & 50.25 \\
Qwen3.5 & IP-VCD & 486 & 1142 & 1161 & 121 & 61 & 60 & 47.11 & 50.41 & 12.55 & 49.63 \\
Qwen3.5 & IP-M3ID & 486 & 1210 & 1109 & 105 & 48 & 57 & 49.92 & 45.71 & 9.88 & 51.90 \\
Gemma-3 & Guided Direct & 422 & 1325 & 1099 & 0 & 0 & 0 & 54.66 & 0.00 & 0.00 & 54.66 \\
Gemma-3 & Guided VCD & 422 & 1277 & 1146 & 1 & 0 & 1 & 52.68 & 0.00 & 0.00 & 52.68 \\
Gemma-3 & IP-VCD & 422 & 1278 & 1145 & 1 & 1 & 0 & 52.72 & 100.00 & 0.24 & 52.76 \\
Gemma-3 & IP-M3ID & 422 & 1324 & 1100 & 0 & 0 & 0 & 54.62 & 0.00 & 0.00 & 54.62 \\
OneVision & Guided Direct & 814 & 626 & 1797 & 1 & 1 & 0 & 25.83 & 100.00 & 0.12 & 25.87 \\
OneVision & Guided VCD & 814 & 713 & 1710 & 1 & 1 & 0 & 29.41 & 100.00 & 0.12 & 29.46 \\
OneVision & IP-VCD & 814 & 749 & 1674 & 1 & 1 & 0 & 30.90 & 100.00 & 0.12 & 30.94 \\
OneVision & IP-M3ID & 814 & 655 & 1768 & 1 & 1 & 0 & 27.02 & 100.00 & 0.12 & 27.06 \\
Phi-3.5-Vision & Guided Direct & 856 & 687 & 1715 & 22 & 16 & 6 & 28.34 & 72.73 & 1.87 & 29.00 \\
Phi-3.5-Vision & Guided VCD & 856 & 774 & 1640 & 10 & 10 & 0 & 31.93 & 100.00 & 1.17 & 32.34 \\
Phi-3.5-Vision & IP-VCD & 856 & 831 & 1576 & 17 & 11 & 6 & 34.28 & 64.71 & 1.29 & 34.74 \\
Phi-3.5-Vision & IP-M3ID & 856 & 705 & 1705 & 14 & 12 & 2 & 29.08 & 85.71 & 1.40 & 29.58 \\
\bottomrule

\end{tabular}}
\caption{Full counts and metrics of the four guided constructions under the unified UNKNOWN instruction.
Each row has a denominator of 2,424; $N_+$ is the number of abstention-labeled inputs of the same model.
C, W, and A are correct answers, incorrect answers, and all abstentions; TP and FP are supported and unsupported abstentions.
The last four columns are percentages, and precision is shown as 0.00 when there is no abstention, following the convention used throughout.}
\label{tab:ablationcounts}
\end{table*}

Under the unified UNKNOWN condition, IP-VCD reaches a higher DU than guided Direct on six models, guided VCD exceeds guided Direct on four models, and IP-VCD exceeds guided VCD on seven models.
Under the development-selected configuration, IP-VCD is not lower than guided VCD on seven models, strictly higher on six and tied on Gemma-3 (Table~\ref{tab:vcdcomp}).
Among the five representative models of the main text, the best IP method ranks within the top two among the four guided constructions; this ranking covers only the four guided constructions, with vanilla Direct and VCD listed separately in Table~\ref{tab:ablationfull}.

\begin{table}[htbp]
\centering
\footnotesize
\setlength{\tabcolsep}{3pt}
\resizebox{\columnwidth}{!}{%
\begin{tabular}{llrrr}
\toprule
Model & Setting & Guided Direct precision & Guided VCD DU & IP-VCD DU \\
\midrule
Qwen2.5-VL & UNKNOWN & 22.78 & 43.44 & 42.08 \\
Qwen3.5 & UNKNOWN & 45.78 & 50.25 & 49.63 \\
Qwen2.5-VL & Development-selected & 23.68 & 45.87 & 44.02 \\
LLaVA-Mistral & Development-selected & 42.69 & 52.85 & 52.48 \\
\bottomrule

\end{tabular}}
\caption{The four model and instruction settings in which guided VCD exceeds IP-VCD.
Each of the two settings covers nine models; in all four cases, the abstention precision of guided Direct is below 50\%.
DU and precision are percentages.}
\label{tab:gvcdleads}
\end{table}

Low precision means that most abstentions of guided Direct fall on inputs not labeled for abstention; the visual contrast term applied on top of the guided prediction can then raise utility by restoring correct answers.
This result gives the behavioral background under which the reference-side instruction can help; which per-question changes each branch causes is determined from the complete paired responses.
Under the unified UNKNOWN instruction, guided Direct, guided VCD, and both IP methods on MiniCPM-V2.6 all fall below vanilla VCD, which shows that the effect of the instruction itself also enters the final decisions.

The aggregate differences under the unified UNKNOWN instruction further separate correct answers from supported abstentions.
On LLaVA-Mistral, IP-M3ID answers 46 more questions correctly than guided Direct and makes 48 fewer supported abstentions, a net loss of 2 successful decisions and 0.08 percentage points of utility.
On Gemma-3, neither abstains, and IP-M3ID answers 1 fewer question correctly, 0.04 percentage points lower in utility.
On Qwen3.5, IP-M3ID answers 87 more questions correctly and makes 180 fewer supported abstentions, a net loss of 93 successful decisions and 3.84 percentage points.
On Qwen2.5-VL, IP-VCD answers 40 fewer questions correctly than guided VCD and makes 7 more supported abstentions, a net loss of 33 successful decisions and 1.36 percentage points.
All four comparisons come directly from the unified UNKNOWN condition in Table~\ref{tab:ablationcounts}.
On LLaVA-Mistral, IP-VCD and guided VCD both answer 852 questions correctly, and supported abstentions rise from 33 to 66, raising utility by 1.36 percentage points.

These numerator changes are aggregate differences.
Per-question response transitions are given in Table~\ref{tab:transitions} (the LLaVA-Mistral comparison under the same UNKNOWN instruction) and Table~\ref{tab:refpostransitions} (the main-result configuration on nine models).
An unchanged number of correct answers can coexist with simultaneous corrections and new errors.

The visual-task implementation of CDA dynamically combines three predictions, parametric knowledge, visual context, and abstention, according to knowledge relevance under the current input and response prefix, and computes calibration terms from two empty-input predictions \citep{cda}.
Let $H_p,H_c$ be the entropies of the parametric-knowledge and visual-context branches, and $H_{0p},H_{0c}$ the corresponding empty-input entropies.
Then
\begin{equation*}
\begin{aligned}
r_p&=\frac{\max(H_p-H_{0p},0)}{H_{0p}},\\
r_c&=\frac{\max(H_c-H_{0c},0)}{H_{0c}},
\end{aligned}
\end{equation*}
\begin{equation*}
\begin{gathered}
w_p=\frac{r_p^2}{r_p+r_c},\qquad
w_c=\frac{r_c^2}{r_p+r_c},\\
w_a=1-w_p-w_c.
\end{gathered}
\end{equation*}
When $r_p+r_c=0$, continuous extension gives $(w_p,w_c,w_a)=(0,0,1)$.
The 108 response paths of the nine models provide 803 token positions, and the abstention weight is negative at 145 of them.
This measurement shows how calibration with empty-input entropy changes the branch weights.

\subsection{SID Implementation and Evaluation Subset}
\label{app:sid}

The visual-token selection of SID needs the attention weights of a specified layer and requires the input to contain at least 100 visual tokens.
On Food-101, 857 images of MiniCPM-V2.6 yield only 64 visual tokens, and the corresponding layer of Qwen3.5 uses linear attention; full SID results for these two models are therefore not listed in Table~\ref{tab:foodfull} (Food-101) or Table~\ref{tab:vizfull} (VizWiz).
On VizWiz, SID for Qwen2.5-VL and Qwen3-VL uses the 505-question common subset, and the results are listed in the corresponding sample-size groups of Table~\ref{tab:vizfull} and Table~\ref{tab:vizreliability}.
On Food-101, Gemma-3 uses SID to select 100 visual tokens at the second layer with a contrast coefficient of 0.5 and computes scores over the full vocabulary; the other six Food-101 models that support SID use a contrast coefficient of 1 and filter candidates by the original-image prediction.

Qwen2.5-VL and Qwen3-VL each have 7 of the 512 VizWiz questions with fewer than 100 visual tokens.
After removing these 7 questions, SID completes generation on the remaining 505; the 505 questions are the same for both models, with 341 officially labeled answerable and 164 unanswerable.
SID uses the head-averaged attention of the last-position query at the second decoding layer (index 1, counting from 0), selects the 100 visual tokens with the lowest attention, and masks the remaining visual keys from the next layer (index 2) onward; the contrast coefficient is 0.5, decoding is greedy over the full vocabulary, and responses are limited to 32 tokens.

After the full results of Qwen2.5-VL and Qwen3-VL, Table~\ref{tab:vizfull} lists the baselines and IP methods on the same 505-question subset, and Table~\ref{tab:vizreliability} lists the reliability metrics in the same way.
All methods keep the generation configuration of the full task.

\section{Statistical Intervals and Inference Cost}
\label{app:stats}

\subsection{Paired Bootstrap}
\label{app:f1}

For each VizWiz test question, we pair the DU of an IP method with that of vanilla VCD and compute their difference.
Each round draws 512 questions with replacement from the 512, keeps the paired scores of the two methods on each question, and computes the mean utility difference.
After 2,000 rounds, the 2.5\% and 97.5\% quantiles of the resampled distribution form the 95\% confidence interval; the random seed is 20260929.
Table~\ref{tab:bootstrap} lists all 18 intervals for the two IP methods on nine models.
Resampling uses responses that have already been generated, so the intervals reflect how variation in the test sample affects the mean utility difference.

\begin{table}[htbp]
\centering
\footnotesize
\setlength{\tabcolsep}{4pt}
\resizebox{\columnwidth}{!}{%
\begin{tabular}{llcc}
\toprule
Model & Method & Utility gain & 95\% confidence interval \\
\midrule
Qwen2.5-VL & IP-VCD & +10.90 & [+6.44, +15.47] \\
Qwen2.5-VL & IP-M3ID & +9.18 & [+4.35, +13.87] \\
LLaVA-Mistral & IP-VCD & +21.99 & [+17.89, +26.47] \\
LLaVA-Mistral & IP-M3ID & +21.70 & [+16.80, +26.35] \\
MiniCPM-V2.6 & IP-VCD & +22.25 & [+18.18, +26.43] \\
MiniCPM-V2.6 & IP-M3ID & +24.41 & [+19.84, +29.08] \\
InternVL3.5 & IP-VCD & +7.66 & [+3.67, +11.60] \\
InternVL3.5 & IP-M3ID & +10.51 & [+6.31, +14.67] \\
Qwen3-VL & IP-VCD & +9.80 & [+6.05, +13.48] \\
Qwen3-VL & IP-M3ID & +12.58 & [+8.87, +16.66] \\
Qwen3.5 & IP-VCD & +9.59 & [+5.41, +13.30] \\
Qwen3.5 & IP-M3ID & +13.11 & [+9.12, +17.27] \\
Gemma-3 & IP-VCD & +12.42 & [+9.22, +15.74] \\
Gemma-3 & IP-M3ID & +11.72 & [+8.61, +15.20] \\
OneVision & IP-VCD & +14.18 & [+10.53, +17.79] \\
OneVision & IP-M3ID & +13.57 & [+9.61, +17.46] \\
Phi-3.5-Vision & IP-VCD & +14.69 & [+11.29, +18.05] \\
Phi-3.5-Vision & IP-M3ID & +15.16 & [+11.58, +19.14] \\
\bottomrule

\end{tabular}}
\caption{All paired intervals relative to vanilla VCD on VizWiz.
Differences are in percentage points; the gain column is computed from unrounded utility differences and rounded to two decimals.}
\label{tab:bootstrap}
\end{table}

\subsection{Timing of Paired Runs}
\label{app:f2}

The inference-time comparison covers 14 paired runs on five models; in each run, guided VCD, IP-VCD, and CDA each process 2,424 questions.
End-to-end timing covers image loading and conversion, instruction and conversation construction, response generation, and cleanup; model loading and result writing fall outside the timed range.
Within each run, the GPU, model, numerical precision, software environment, and generation budget are the same.
We first compute the mean time per question of each method and then divide it by the mean time per question of guided VCD in the same run to obtain the relative time.
Over the 14 runs, the mean relative time of IP-VCD is 1.4774 and the median is 1.4655; the mean relative time of CDA is 2.8243.
The times and ratios of each run are given in Table~\ref{tab:pairtime}.

\begin{table}[htbp]
\centering
\footnotesize
\setlength{\tabcolsep}{3pt}
\resizebox{\columnwidth}{!}{%
\begin{tabular}{llccccc}
\toprule
Model & Abstention marker & Guided VCD (s/question) & IP-VCD (s/question) & CDA (s/question) & IP/VCD & CDA/VCD \\
\midrule
Qwen2.5-VL & UNKNOWN & 0.7059 & 1.0507 & 1.8326 & 1.4884 & 2.5961 \\
Qwen2.5-VL & UNCLEAR & 0.8236 & 1.1627 & 2.0523 & 1.4118 & 2.4920 \\
Qwen2.5-VL & UNSURE & 0.8239 & 1.1733 & 2.1607 & 1.4241 & 2.6226 \\
Qwen2.5-VL & I cannot identify it & 1.2681 & 1.7009 & 3.6244 & 1.3413 & 2.8582 \\
LLaVA-Mistral & UNKNOWN & 2.0497 & 3.0585 & 5.2381 & 1.4922 & 2.5556 \\
MiniCPM-V2.6 & UNKNOWN & 0.8475 & 1.3224 & 2.0518 & 1.5603 & 2.4209 \\
MiniCPM-V2.6 & UNCLEAR & 0.9278 & 1.3363 & 2.4492 & 1.4402 & 2.6397 \\
MiniCPM-V2.6 & UNSURE & 0.9738 & 1.4048 & 2.4130 & 1.4426 & 2.4780 \\
MiniCPM-V2.6 & I cannot identify it & 1.0160 & 1.4524 & 2.8516 & 1.4295 & 2.8066 \\
Qwen3.5 & UNKNOWN & 0.5787 & 0.9976 & 2.0645 & 1.7237 & 3.5674 \\
Qwen3.5 & UNCLEAR & 0.5983 & 0.9065 & 1.7713 & 1.5150 & 2.9603 \\
Qwen3.5 & UNSURE & 0.5973 & 0.8900 & 1.8108 & 1.4901 & 3.0317 \\
Qwen3.5 & I cannot identify it & 0.6043 & 0.8662 & 1.8085 & 1.4333 & 2.9926 \\
Gemma-3 & UNKNOWN & 2.8067 & 4.1845 & 9.8754 & 1.4909 & 3.5185 \\
\bottomrule

\end{tabular}}
\caption{Paired inference time on the full task.
In each run, every method processes 2,424 questions; times and ratios keep four decimals, and ratios use guided VCD in the same run as the baseline.}
\label{tab:pairtime}
\end{table}

\subsection{Cost on Nine Models in a Unified A100 Environment}
\label{app:f3}

\subsubsection{Measurement Conditions and Units}

The nine models each run guided VCD, IP-VCD, and CDA on the same 100 Food-101 evaluation inputs, for 2,700 timed responses in total.
The cost subset is drawn from the 2,424 questions without replacement, covers 61 classes, and uses its own sampling seed 20260929.
Another 8 non-overlapping inputs are used for warm-up, once per model and method, for 216 warm-up responses.
Per-question generation and image perturbation follow \texttt{stable\_seed(sample\_id, model, 0)}.

The device is an A100 80GB.
InternVL3.5 uses two GPUs following its original device map, and every other model uses one.
The three methods of the same model keep the same checkpoint, precision, environment, and number of devices.
OneVision uses FP16 and the other eight models use BF16; batch size is 1, temperature 0, top-p 1, and greedy generation is limited to 32 tokens.

Timing starts after model loading: each run synchronizes all assigned GPUs before recording the start time, covers image loading and RGB conversion, instruction and conversation construction, response generation, and cleanup of the conversation and image, and closes the interval with a final synchronization.
Model loading, separate warm-up, result writing, and scheduling waits are counted separately.
Table~\ref{tab:a100cost} reports the mean end-to-end seconds per question, and ratios are ratios of per-condition means within the same model.
All responses that reach the 32-token limit are kept, and the number of generated tokens includes any generated EOS token.

\begin{table}[htbp]
\centering
\footnotesize
\setlength{\tabcolsep}{3pt}
\resizebox{\columnwidth}{!}{%
\begin{tabular}{lcccccc}
\toprule
Model & Guided VCD & IP-VCD & CDA & IP/VCD & CDA/VCD & CDA/IP \\
\midrule
Qwen2.5-VL & 0.603 & 0.824 & 1.449 & $1.37\times$ & $2.40\times$ & $1.758\times$ \\
Qwen3.5-4B & 0.617 & 0.852 & 1.789 & $1.38\times$ & $2.90\times$ & $2.100\times$ \\
LLaVA-Mistral & 1.026 & 1.456 & 3.260 & $1.42\times$ & $3.18\times$ & $2.240\times$ \\
MiniCPM-V2.6 & 0.712 & 0.944 & 1.506 & $1.33\times$ & $2.12\times$ & $1.595\times$ \\
Gemma-3-4B & 1.682 & 2.261 & 6.099 & $1.34\times$ & $3.63\times$ & $2.697\times$ \\
InternVL3.5-8B & 2.066 & 2.762 & 2.443 & $1.34\times$ & $1.18\times$ & $0.884\times$ \\
OneVision & 0.953 & 1.349 & 1.605 & $1.42\times$ & $1.68\times$ & $1.189\times$ \\
Phi-3.5-Vision & 0.885 & 1.438 & 1.498 & $1.63\times$ & $1.69\times$ & $1.042\times$ \\
Qwen3-VL & 0.665 & 1.002 & 1.243 & $1.51\times$ & $1.87\times$ & $1.241\times$ \\
\bottomrule

\end{tabular}}
\caption{A100 inference cost on the 100 questions shared by nine models.
Times are seconds per question; every model and method uses 100 questions, and only means are reported.}
\label{tab:a100cost}
\end{table}

\subsubsection{Generation Length and Branch Overhead}

The model-weighted mean IP/VCD and CDA/VCD ratios over the nine models are 1.4135 and 2.2946, an average increase of 41.35\% and 129.46\% in end-to-end latency.
Across models, IP/VCD ranges from 1.3261 to 1.6258 and CDA/VCD from 1.1821 to 3.6259.
The model-weighted ratio averages the per-model time ratios, so every model contributes equally to the summary.

IP-VCD shows a measurable extra time even when the average response length is nearly the same.
On Qwen2.5-VL, guided VCD and IP-VCD generate 4.18 and 4.19 tokens on average, and IP takes 1.366 times as long; on Qwen3.5-4B, the lengths are 3.93 and 3.92 tokens and the time ratio is 1.381.
These results are consistent with the computational cost of the added unguided original-image branch.
On Phi-3.5, IP/VCD is 1.626 while the output grows from 7.32 to 8.43 tokens, showing that the number of free generation steps also enters the end-to-end cost.

\begin{table}[htbp]
\centering
\footnotesize
\setlength{\tabcolsep}{3pt}
\resizebox{\columnwidth}{!}{%
\begin{tabular}{lccccc}
\toprule
Model & A100 GPUs & Guided VCD mean tokens & IP-VCD mean tokens & CDA mean tokens & Stops at 32-token limit \\
\midrule
Qwen2.5-VL & 1 & 4.18 & 4.19 & 4.84 & 0 / 0 / 3 \\
Qwen3.5-4B & 1 & 3.93 & 3.92 & 5.09 & 0 / 0 / 1 \\
LLaVA-Mistral & 1 & 6.41 & 6.46 & 11.79 & 0 / 0 / 7 \\
MiniCPM-V2.6 & 1 & 5.56 & 5.37 & 5.32 & 0 / 0 / 0 \\
Gemma-3-4B & 1 & 9.58 & 9.19 & 14.72 & 6 / 5 / 12 \\
InternVL3.5-8B & 2 & 11.24 & 10.26 & 4.94 & 10 / 9 / 3 \\
OneVision & 1 & 3.73 & 3.77 & 3.28 & 0 / 0 / 0 \\
Phi-3.5-Vision & 1 & 7.32 & 8.43 & 4.89 & 3 / 7 / 0 \\
Qwen3-VL & 1 & 4.63 & 5.10 & 3.70 & 0 / 2 / 0 \\
\bottomrule

\end{tabular}}
\caption{Generation length, stops at the 32-token limit, and number of devices.
Length is the mean number of generated tokens per question.
Limit stops count only responses without an EOS token; a response whose 32nd token is EOS counts as an EOS stop.
Counts are listed in the order guided VCD / IP-VCD / CDA, with 100 responses per method.}
\label{tab:genlength}
\end{table}

CDA is slower than IP-VCD on eight models; InternVL3.5 is the exception, with a CDA/IP time ratio of 0.884, as CDA outputs 4.94 tokens on average versus 10.26 for IP.
On Gemma-3, CDA outputs 14.72 tokens, stops at the length limit 12 times, and takes 3.626 times as long as guided VCD.
The number of branches, the input conditions of each branch, cache handling, and the final number of generation steps all enter the measured interval, so the ratio of five branches to three cannot directly replace the time ratio.
The end-to-end measurements show the combined cost of these factors.

\subsubsection{GPU Allocation and the Two Kinds of Timing}

Per-question wall-clock time measures response latency, whereas allocated GPU-seconds also account for the number of devices: the seconds of each interval are multiplied by the number of GPUs allocated in that interval and summed over runs.
Each run of InternVL3.5 occupies two GPUs, so the same wall-clock seconds correspond to twice the allocated GPU-seconds.
Table~\ref{tab:gpuseconds} lists timed runs, warm-up, and loading separately; these quantities measure device allocation, while the actual active time of GPU kernels is a separate measurement.

\begin{table}[htbp]
\centering
\footnotesize
\setlength{\tabcolsep}{4pt}
\resizebox{0.9\columnwidth}{!}{%
\begin{tabular}{lrrrr}
\toprule
Model & Timed wall-clock (s) & Timed GPU-s & Warm-up GPU-s & Loading GPU-s \\
\midrule
Qwen2.5-VL & 287.638 & 287.638 & 21.992 & 38.143 \\
Qwen3.5-4B & 325.832 & 325.832 & 28.204 & 33.016 \\
LLaVA-Mistral & 574.220 & 574.220 & 44.940 & 29.712 \\
MiniCPM-V2.6 & 316.092 & 316.092 & 25.404 & 39.541 \\
Gemma-3-4B & 1004.288 & 1004.288 & 86.347 & 22.510 \\
InternVL3.5-8B & 727.117 & 1454.234 & 100.119 & 63.451 \\
OneVision & 390.735 & 390.735 & 36.454 & 17.884 \\
Phi-3.5-Vision & 382.106 & 382.106 & 34.036 & 21.794 \\
Qwen3-VL & 290.942 & 290.942 & 22.096 & 13.168 \\
\bottomrule

\end{tabular}}
\caption{Time and allocated GPU-seconds for 300 timed responses per model.
Timed values are the sums of the timed intervals of the three methods; warm-up and loading are not included in Table~\ref{tab:a100cost}.}
\label{tab:gpuseconds}
\end{table}

The timed runs total 5,026.086 allocated GPU-seconds; warm-up and loading account for 399.591 and 279.218 allocated GPU-seconds.
Counting the running interval of each worker process from admission to completion gives a total of 5,715.523 allocated GPU-seconds, which also includes process overhead outside the timed intervals.
Absolute times across models depend on the checkpoint, precision, and device map, while ratios within the same model give the relative cost under a fixed hardware configuration.

Section~\ref{app:f2} reports 14 full-task runs on five models, and Table~\ref{tab:a100cost} in \S\ref{app:f3} reports the unified A100 cost on 100 questions shared by nine models.
The two panels differ in sample size, model composition, and configuration, so 1.4774/2.8243 and 1.4135/2.2946 each belong to their own timing panel.
The A100 measurements support cost comparisons under this hardware and generation budget; answering, abstention, and utility on the full task follow Tables~\ref{tab:foodfull} and~\ref{tab:vizfull} on their respective evaluation sets.
\section{Same-Question Comparison of Operator Implementations}
\label{app:operators}

On a 101-question set with one image per class, we compare the adapted implementations used in the full-task evaluation with the decoding rules in the authors' code.
For DoLa, the comparison computes the distribution divergence used to select the early layer as in the authors' code; for SID, it follows the authors' greedy decoding rule and contrasts over the full vocabulary with a coefficient of 0.5.
DoLa covers nine models and SID six, and the two implementations process the same 101 questions on each model.
Table~\ref{tab:operator} lists the number of correct answers and the paired response changes per model.

\begin{table}[htbp]
\centering
\footnotesize
\setlength{\tabcolsep}{2.6pt}
\resizebox{\columnwidth}{!}{%
\begin{tabular}{llcccc}
\toprule
Model & Method & Evaluation implementation: correct & Author operator: correct & $Correct \rightarrow Incorrect$ & $Incorrect \rightarrow Correct$ \\
\midrule
Qwen2.5-VL & DoLa & 20 & 20 & 0 & 0 \\
Qwen2.5-VL & SID & 29 & 29 & 4 & 4 \\
LLaVA-Mistral & DoLa & 34 & 34 & 0 & 0 \\
LLaVA-Mistral & SID & 39 & 37 & 3 & 1 \\
MiniCPM-V2.6 & DoLa & 29 & 29 & 0 & 0 \\
InternVL3.5 & DoLa & 20 & 21 & 5 & 6 \\
InternVL3.5 & SID & 20 & 19 & 2 & 1 \\
Qwen3-VL & DoLa & 46 & 45 & 1 & 0 \\
Qwen3-VL & SID & 46 & 44 & 5 & 3 \\
Qwen3.5 & DoLa & 50 & 50 & 0 & 0 \\
Gemma-3 & DoLa & 48 & 53 & 4 & 9 \\
OneVision & DoLa & 24 & 24 & 2 & 2 \\
OneVision & SID & 30 & 28 & 2 & 0 \\
Phi-3.5-Vision & DoLa & 36 & 30 & 8 & 2 \\
Phi-3.5-Vision & SID & 39 & 38 & 5 & 4 \\
\bottomrule

\end{tabular}}
\caption{Operator comparison on the 101-question paired set.
Correct answers and transitions are counted on the same 101 questions.}
\label{tab:operator}
\end{table}

\section{Cross Comparison of Reference-Side Instructions}
\label{app:cross}

On the same 2,424 questions of Qwen2.5-VL, we keep the image and the original-image instruction fixed, change the reference-side instruction, and observe the abstention rate of complete responses.
Table~\ref{tab:cross} uses $\tau_1$ to $\tau_4$ for four instruction conditions with the same abstention intent, whose markers are UNKNOWN, UNCLEAR, UNSURE, and I cannot identify it.
Rows are original-image conditions and columns are reference conditions.
The variation within each row reflects how the reference-side instruction affects the final abstention behavior.

\begin{table}[htbp]
\centering
\footnotesize
\setlength{\tabcolsep}{6pt}
\begin{tabular}{lcccc}
\toprule
Original / Reference & $\tau_1$ & $\tau_2$ & $\tau_3$ & $\tau_4$ \\
\midrule
$\tau_1$ & 31.72 & 84.03 & 84.86 & 87.58 \\
$\tau_2$ & 61.34 & 12.33 & 55.98 & 68.94 \\
$\tau_3$ & 86.30 & 64.89 & 24.46 & 95.79 \\
$\tau_4$ & 81.56 & 80.61 & 79.74 & 29.46 \\
\bottomrule

\end{tabular}
\caption{Cross comparison of original-image and reference-side instructions.
Values are semantic abstention rates (\%).}
\label{tab:cross}
\end{table}

\section{Fixed-Policy Evaluation with Reliable VQA}
\label{app:reliable}

For the already frozen answering and abstention policies, we compute coverage, risk, and effective reliability \citep{reliable_vqa}.
Let $a_i$ indicate abstention, $q_i$ the answer score, and $g_i=1-a_i$.
Coverage is $N^{-1}\sum_i g_i$, and risk is $\sum_i g_i(1-q_i)/\sum_i g_i$.
Effective reliability scores $q_i$ for an answer with $q_i>0$, $-c$ for an answer with $q_i=0$, where $c$ is the error cost (not the log probability $c$ of the main text), and 0 for an abstention, then averages over all inputs; we report $c=1,10,100$.
Food-101 uses binary correctness, and VizWiz keeps the official answer consensus score.
These results evaluate existing generation policies, and no confidence selector is added or trained.

Higher effective reliability and lower risk are better; coverage is the share of answered inputs and should be read together with risk.
Overall accuracy on Food-101 and answer score on VizWiz both equal coverage times the mean answer score over answered inputs, so the answer score reflects both answer coverage and answer quality.
When every input is abstained on, coverage and effective reliability are 0, and risk is undefined because there are no answers.

The tables report metric means multiplied by 100.
Coverage and $\Phi_c$ use all inputs, and risk uses the subset actually answered.
$\Phi_c$ can be negative, and a larger error cost also widens its range.

\begin{table*}[tp]
\centering
\footnotesize
\setlength{\tabcolsep}{4.6pt}
\resizebox{0.96\textwidth}{!}{%
\begin{tabular}{llrrrrr}
\toprule
Model & Method & Coverage (\%) & Risk $\downarrow$ & $\Phi_1\uparrow$ & $\Phi_{10}\uparrow$ & $\Phi_{100}\uparrow$ \\
\midrule
Qwen2.5-VL & Direct & 100.00 & 69.60 & -39.19 & -665.55 & -6929.17 \\
Qwen2.5-VL & Vanilla VCD & 100.00 & 71.66 & -43.32 & -688.24 & -7137.50 \\
Qwen2.5-VL & Vanilla M3ID & 100.00 & 71.29 & -42.57 & -684.16 & -7100.00 \\
Qwen2.5-VL & DoLa & 100.00 & 75.50 & -50.99 & -730.45 & -7525.00 \\
Qwen2.5-VL & DeCo & 100.00 & 69.76 & -39.52 & -667.37 & -6945.83 \\
Qwen2.5-VL & SID & 100.00 & 73.60 & -47.19 & -709.57 & -7333.33 \\
Qwen2.5-VL & CDA & 84.41 & 59.82 & -16.58 & -471.04 & -5015.59 \\
Qwen2.5-VL & IP-VCD & 69.80 & 49.59 & 0.58 & -310.93 & -3426.03 \\
Qwen2.5-VL & IP-M3ID & 76.69 & 45.02 & 7.63 & -303.14 & -3410.81 \\
LLaVA-Mistral & Direct & 99.63 & 66.00 & -31.89 & -623.72 & -6542.04 \\
LLaVA-Mistral & Vanilla VCD & 100.00 & 61.76 & -23.51 & -579.33 & -6137.50 \\
LLaVA-Mistral & Vanilla M3ID & 99.92 & 65.07 & -30.12 & -615.26 & -6466.75 \\
LLaVA-Mistral & DoLa & 99.71 & 66.61 & -33.13 & -630.90 & -6608.62 \\
LLaVA-Mistral & DeCo & 99.92 & 73.08 & -46.12 & -703.30 & -7275.08 \\
LLaVA-Mistral & SID & 100.00 & 61.72 & -23.43 & -578.88 & -6133.33 \\
LLaVA-Mistral & CDA & 99.42 & 69.79 & -39.36 & -663.86 & -6908.91 \\
LLaVA-Mistral & IP-VCD & 54.58 & 48.00 & 2.19 & -233.58 & -2591.25 \\
LLaVA-Mistral & IP-M3ID & 28.30 & 52.04 & -1.16 & -133.70 & -1459.20 \\
MiniCPM-V2.6 & Direct & 99.88 & 56.92 & -13.82 & -525.45 & -5641.79 \\
MiniCPM-V2.6 & Vanilla VCD & 99.96 & 56.38 & -12.75 & -519.93 & -5591.71 \\
MiniCPM-V2.6 & Vanilla M3ID & 99.92 & 57.27 & -14.52 & -529.50 & -5679.25 \\
MiniCPM-V2.6 & DoLa & 99.96 & 68.88 & -37.75 & -657.43 & -6854.21 \\
MiniCPM-V2.6 & DeCo & 99.88 & 57.99 & -15.97 & -537.25 & -5750.12 \\
MiniCPM-V2.6 & CDA & 99.83 & 72.31 & -44.55 & -694.31 & -7191.83 \\
MiniCPM-V2.6 & IP-VCD & 77.85 & 52.15 & -3.34 & -368.69 & -4022.15 \\
MiniCPM-V2.6 & IP-M3ID & 63.86 & 51.03 & -1.32 & -294.64 & -3227.81 \\
InternVL3.5 & Direct & 100.00 & 82.92 & -65.84 & -812.13 & -8275.00 \\
InternVL3.5 & Vanilla VCD & 99.96 & 80.89 & -61.76 & -789.48 & -8066.71 \\
InternVL3.5 & Vanilla M3ID & 100.00 & 82.43 & -64.85 & -806.68 & -8225.00 \\
InternVL3.5 & DoLa & 100.00 & 83.83 & -67.66 & -822.11 & -8366.67 \\
InternVL3.5 & DeCo & 100.00 & 82.05 & -64.11 & -802.60 & -8187.50 \\
InternVL3.5 & SID & 100.00 & 81.48 & -62.95 & -796.25 & -8129.17 \\
InternVL3.5 & CDA & 99.96 & 79.61 & -59.20 & -775.41 & -7937.54 \\
InternVL3.5 & IP-VCD & 98.35 & 71.81 & -42.90 & -678.55 & -7034.98 \\
InternVL3.5 & IP-M3ID & 97.73 & 71.89 & -42.78 & -675.08 & -6998.10 \\
\bottomrule

\end{tabular}}
\caption{Fixed policies on Food-101 for nine models (part 1 of 2; 2,424 questions per operating point).
Coverage, risk, and the three reliability scores are percentages; reliability can be negative and can fall below $-$100.}
\label{tab:foodreliability}
\end{table*}

\addtocounter{table}{-1}
\begin{table*}[tp]
\centering
\footnotesize
\setlength{\tabcolsep}{4.6pt}
\resizebox{0.96\textwidth}{!}{%
\begin{tabular}{llrrrrr}
\toprule
Model & Method & Coverage (\%) & Risk $\downarrow$ & $\Phi_1\uparrow$ & $\Phi_{10}\uparrow$ & $\Phi_{100}\uparrow$ \\
\midrule
Qwen3-VL & Direct & 99.92 & 46.04 & 7.92 & -406.06 & -4545.92 \\
Qwen3-VL & Vanilla VCD & 100.00 & 43.03 & 13.94 & -373.31 & -4245.83 \\
Qwen3-VL & Vanilla M3ID & 99.92 & 46.24 & 7.51 & -408.33 & -4566.75 \\
Qwen3-VL & DoLa & 99.96 & 48.16 & 3.67 & -429.62 & -4762.54 \\
Qwen3-VL & DeCo & 100.00 & 46.29 & 7.43 & -409.16 & -4575.00 \\
Qwen3-VL & SID & 100.00 & 51.86 & -3.71 & -470.42 & -5137.50 \\
Qwen3-VL & CDA & 99.96 & 97.23 & -94.43 & -969.18 & -9716.71 \\
Qwen3-VL & IP-VCD & 80.94 & 34.51 & 25.08 & -226.28 & -2739.89 \\
Qwen3-VL & IP-M3ID & 83.62 & 40.40 & 16.05 & -288.04 & -3328.88 \\
Qwen3.5 & Direct & 100.00 & 48.68 & 2.64 & -435.48 & -4816.67 \\
Qwen3.5 & Vanilla VCD & 100.00 & 48.10 & 3.80 & -429.13 & -4758.33 \\
Qwen3.5 & Vanilla M3ID & 100.00 & 49.67 & 0.66 & -446.37 & -4916.67 \\
Qwen3.5 & DoLa & 100.00 & 48.31 & 3.38 & -431.39 & -4779.17 \\
Qwen3.5 & DeCo & 100.00 & 48.84 & 2.31 & -437.29 & -4833.33 \\
Qwen3.5 & CDA & 93.69 & 71.95 & -41.13 & -647.81 & -6714.65 \\
Qwen3.5 & IP-VCD & 98.06 & 47.08 & 5.73 & -409.74 & -4564.44 \\
Qwen3.5 & IP-M3ID & 95.67 & 47.82 & 4.17 & -407.59 & -4525.17 \\
Gemma-3 & Direct & 100.00 & 45.21 & 9.57 & -397.36 & -4466.67 \\
Gemma-3 & Vanilla VCD & 99.96 & 46.26 & 7.47 & -408.75 & -4570.87 \\
Gemma-3 & Vanilla M3ID & 100.00 & 45.54 & 8.91 & -400.99 & -4500.00 \\
Gemma-3 & DoLa & 100.00 & 45.71 & 8.58 & -402.81 & -4516.67 \\
Gemma-3 & DeCo & 100.00 & 45.42 & 9.16 & -399.63 & -4487.50 \\
Gemma-3 & SID & 100.00 & 48.06 & 3.88 & -428.67 & -4754.17 \\
Gemma-3 & CDA & 99.26 & 70.03 & -39.77 & -665.39 & -6921.58 \\
Gemma-3 & IP-VCD & 99.96 & 46.93 & 6.15 & -416.01 & -4637.54 \\
Gemma-3 & IP-M3ID & 100.00 & 45.38 & 9.24 & -399.17 & -4483.33 \\
OneVision & Direct & 100.00 & 69.43 & -38.86 & -663.74 & -6912.50 \\
OneVision & Vanilla VCD & 100.00 & 64.98 & -29.95 & -614.73 & -6462.50 \\
OneVision & Vanilla M3ID & 100.00 & 67.95 & -35.89 & -647.40 & -6762.50 \\
OneVision & DoLa & 100.00 & 81.06 & -62.13 & -791.71 & -8087.50 \\
OneVision & DeCo & 100.00 & 68.73 & -37.46 & -656.02 & -6841.67 \\
OneVision & SID & 100.00 & 66.25 & -32.51 & -628.80 & -6591.67 \\
OneVision & CDA & 99.96 & 74.12 & -48.23 & -715.06 & -7383.37 \\
OneVision & IP-VCD & 99.50 & 63.43 & -26.73 & -594.80 & -6275.50 \\
OneVision & IP-M3ID & 98.31 & 69.16 & -37.67 & -649.55 & -6768.36 \\
Phi-3.5-Vision & Direct & 100.00 & 64.27 & -28.55 & -607.01 & -6391.67 \\
Phi-3.5-Vision & Vanilla VCD & 100.00 & 62.21 & -24.42 & -584.32 & -6183.33 \\
Phi-3.5-Vision & Vanilla M3ID & 100.00 & 64.44 & -28.88 & -608.83 & -6408.33 \\
Phi-3.5-Vision & DoLa & 100.00 & 69.39 & -38.78 & -663.28 & -6908.33 \\
Phi-3.5-Vision & DeCo & 100.00 & 70.92 & -41.83 & -680.07 & -7062.50 \\
Phi-3.5-Vision & SID & 100.00 & 64.11 & -28.22 & -605.20 & -6375.00 \\
Phi-3.5-Vision & CDA & 99.75 & 72.46 & -44.80 & -695.30 & -7200.25 \\
Phi-3.5-Vision & IP-VCD & 97.32 & 62.10 & -23.56 & -567.49 & -6006.85 \\
Phi-3.5-Vision & IP-M3ID & 96.33 & 65.22 & -29.33 & -594.80 & -6249.50 \\
\bottomrule

\end{tabular}}
\caption{Fixed policies on Food-101 for nine models (part 2 of 2; 2,424 questions per operating point).
Coverage, risk, and the three reliability scores are percentages; reliability can be negative and can fall below $-$100.}
\label{tab:foodreliability2}
\end{table*}

\begin{table*}[tp]
\centering
\footnotesize
\setlength{\tabcolsep}{3.6pt}
\resizebox{0.98\textwidth}{!}{%
\begin{tabular}{llrrrrrr}
\toprule
Model & Method & N & Coverage (\%) & Risk $\downarrow$ & $\Phi_1\uparrow$ & $\Phi_{10}\uparrow$ & $\Phi_{100}\uparrow$ \\
\midrule
Qwen2.5-VL & Direct & 512 & 91.80 & 71.38 & -31.35 & -549.90 & -5735.45 \\
Qwen2.5-VL & Vanilla VCD & 512 & 95.51 & 74.25 & -39.08 & -612.13 & -6342.60 \\
Qwen2.5-VL & Vanilla M3ID & 512 & 94.34 & 72.44 & -34.16 & -575.57 & -5989.63 \\
Qwen2.5-VL & DoLa & 512 & 92.77 & 78.13 & -45.92 & -641.82 & -6600.80 \\
Qwen2.5-VL & DeCo & 512 & 93.95 & 75.45 & -40.41 & -611.70 & -6324.59 \\
Qwen2.5-VL & CDA & 512 & 53.71 & 75.56 & -23.98 & -357.97 & -3697.81 \\
Qwen2.5-VL & IP-VCD & 512 & 50.98 & 72.53 & -18.42 & -310.21 & -3228.18 \\
Qwen2.5-VL & IP-M3ID & 512 & 51.17 & 76.76 & -24.24 & -349.43 & -3601.39 \\
Qwen2.5-VL & Direct & 505 & 91.88 & 71.21 & -30.97 & -547.80 & -5716.12 \\
Qwen2.5-VL & Vanilla VCD & 505 & 95.45 & 74.27 & -39.01 & -611.09 & -6331.88 \\
Qwen2.5-VL & Vanilla M3ID & 505 & 94.46 & 72.35 & -34.08 & -575.86 & -5993.68 \\
Qwen2.5-VL & DoLa & 505 & 92.87 & 78.04 & -45.74 & -640.99 & -6593.47 \\
Qwen2.5-VL & DeCo & 505 & 93.86 & 75.27 & -39.96 & -608.48 & -6293.62 \\
Qwen2.5-VL & SID & 505 & 96.63 & 72.60 & -35.11 & -589.37 & -6131.94 \\
Qwen2.5-VL & CDA & 505 & 54.06 & 75.38 & -23.92 & -358.97 & -3709.47 \\
Qwen2.5-VL & IP-VCD & 505 & 51.49 & 72.42 & -18.48 & -312.53 & -3253.13 \\
Qwen2.5-VL & IP-M3ID & 505 & 51.49 & 76.58 & -24.18 & -350.32 & -3611.70 \\
LLaVA-Mistral & Direct & 512 & 90.43 & 78.57 & -45.66 & -631.02 & -6484.53 \\
LLaVA-Mistral & Vanilla VCD & 512 & 97.66 & 83.26 & -59.82 & -745.37 & -7600.84 \\
LLaVA-Mistral & Vanilla M3ID & 512 & 91.60 & 78.27 & -46.31 & -642.21 & -6601.19 \\
LLaVA-Mistral & DoLa & 512 & 91.60 & 79.36 & -47.70 & -647.11 & -6641.25 \\
LLaVA-Mistral & DeCo & 512 & 93.16 & 86.37 & -64.65 & -760.74 & -7721.68 \\
LLaVA-Mistral & SID & 512 & 93.75 & 82.83 & -56.56 & -710.47 & -7249.53 \\
LLaVA-Mistral & CDA & 512 & 79.49 & 80.54 & -43.91 & -578.28 & -5922.03 \\
LLaVA-Mistral & IP-VCD & 512 & 52.15 & 72.17 & -20.06 & -331.19 & -3442.52 \\
LLaVA-Mistral & IP-M3ID & 512 & 31.84 & 66.38 & -9.02 & -186.56 & -1961.95 \\
MiniCPM-V2.6 & Direct & 512 & 96.29 & 69.70 & -28.63 & -548.95 & -5752.07 \\
MiniCPM-V2.6 & Vanilla VCD & 512 & 98.44 & 74.42 & -39.67 & -623.26 & -6459.20 \\
MiniCPM-V2.6 & Vanilla M3ID & 512 & 97.27 & 69.54 & -28.96 & -556.31 & -5829.75 \\
MiniCPM-V2.6 & DoLa & 512 & 95.70 & 76.63 & -44.43 & -645.61 & -6657.32 \\
MiniCPM-V2.6 & DeCo & 512 & 95.31 & 69.20 & -27.87 & -542.91 & -5693.30 \\
MiniCPM-V2.6 & CDA & 512 & 64.65 & 80.60 & -37.27 & -485.51 & -4967.93 \\
MiniCPM-V2.6 & IP-VCD & 512 & 65.23 & 59.34 & -4.34 & -282.07 & -3059.41 \\
MiniCPM-V2.6 & IP-M3ID & 512 & 56.45 & 53.32 & 2.52 & -211.93 & -2356.46 \\
InternVL3.5 & Direct & 512 & 96.29 & 74.67 & -38.11 & -600.61 & -6225.61 \\
InternVL3.5 & Vanilla VCD & 512 & 97.27 & 76.73 & -45.14 & -655.10 & -6754.71 \\
InternVL3.5 & Vanilla M3ID & 512 & 96.88 & 74.94 & -39.20 & -610.49 & -6323.38 \\
InternVL3.5 & DoLa & 512 & 98.44 & 77.46 & -46.95 & -669.22 & -6891.88 \\
InternVL3.5 & DeCo & 512 & 97.85 & 78.14 & -47.17 & -664.16 & -6834.08 \\
InternVL3.5 & SID & 512 & 97.46 & 75.79 & -42.23 & -634.61 & -6558.44 \\
InternVL3.5 & CDA & 512 & 85.74 & 76.72 & -40.20 & -581.60 & -5995.66 \\
InternVL3.5 & IP-VCD & 512 & 76.37 & 78.49 & -38.26 & -530.45 & -5452.32 \\
InternVL3.5 & IP-M3ID & 512 & 73.44 & 76.68 & -34.04 & -494.59 & -5100.06 \\
\bottomrule

\end{tabular}}
\caption{Fixed policies on VizWiz for nine models and the SID subset comparison (part 1 of 3).
All methods are listed for each model, grouped by the number of evaluated questions; metrics are multiplied by 100, and risk is undefined when there are no answers.
512 corresponds to the full evaluation set and 505 to the common subset in Table~\ref{tab:vizfull}.}
\label{tab:vizreliability}
\end{table*}

\addtocounter{table}{-1}
\begin{table*}[tp]
\centering
\footnotesize
\setlength{\tabcolsep}{3.6pt}
\resizebox{0.98\textwidth}{!}{%
\begin{tabular}{llrrrrrr}
\toprule
Model & Method & N & Coverage (\%) & Risk $\downarrow$ & $\Phi_1\uparrow$ & $\Phi_{10}\uparrow$ & $\Phi_{100}\uparrow$ \\
\midrule
Qwen3-VL & Direct & 512 & 94.53 & 60.31 & -6.62 & -403.89 & -4376.54 \\
Qwen3-VL & Vanilla VCD & 512 & 94.92 & 62.74 & -12.48 & -443.14 & -4749.79 \\
Qwen3-VL & Vanilla M3ID & 512 & 94.73 & 60.39 & -6.82 & -405.84 & -4396.07 \\
Qwen3-VL & DoLa & 512 & 93.95 & 61.43 & -9.86 & -424.71 & -4573.14 \\
Qwen3-VL & DeCo & 512 & 94.73 & 60.76 & -8.34 & -417.91 & -4513.61 \\
Qwen3-VL & CDA & 512 & 70.12 & 75.01 & -30.14 & -459.04 & -4748.11 \\
Qwen3-VL & IP-VCD & 512 & 65.82 & 58.66 & -4.24 & -287.25 & -3117.32 \\
Qwen3-VL & IP-M3ID & 512 & 57.62 & 52.03 & 4.98 & -198.93 & -2237.99 \\
Qwen3-VL & Direct & 505 & 94.46 & 60.04 & -6.22 & -401.86 & -4358.30 \\
Qwen3-VL & Vanilla VCD & 505 & 94.85 & 62.46 & -11.92 & -439.64 & -4716.87 \\
Qwen3-VL & Vanilla M3ID & 505 & 94.65 & 60.13 & -6.42 & -403.84 & -4378.10 \\
Qwen3-VL & DoLa & 505 & 93.86 & 61.12 & -9.25 & -420.93 & -4537.76 \\
Qwen3-VL & DeCo & 505 & 94.65 & 60.50 & -7.96 & -416.08 & -4497.27 \\
Qwen3-VL & SID & 505 & 94.46 & 60.82 & -7.15 & -404.57 & -4378.83 \\
Qwen3-VL & CDA & 505 & 70.10 & 74.75 & -29.82 & -457.54 & -4734.77 \\
Qwen3-VL & IP-VCD & 505 & 65.94 & 58.47 & -3.90 & -285.49 & -3101.33 \\
Qwen3-VL & IP-M3ID & 505 & 58.22 & 51.87 & 5.25 & -199.70 & -2249.21 \\
Qwen3.5 & Direct & 512 & 99.80 & 62.78 & -15.00 & -484.34 & -5177.70 \\
Qwen3.5 & Vanilla VCD & 512 & 99.02 & 67.50 & -25.04 & -540.08 & -5690.47 \\
Qwen3.5 & Vanilla M3ID & 512 & 99.80 & 62.88 & -15.49 & -488.34 & -5216.86 \\
Qwen3.5 & DoLa & 512 & 99.80 & 63.05 & -15.47 & -486.56 & -5197.50 \\
Qwen3.5 & DeCo & 512 & 99.02 & 64.14 & -16.25 & -482.07 & -5140.27 \\
Qwen3.5 & CDA & 512 & 90.04 & 75.34 & -39.12 & -591.07 & -6110.61 \\
Qwen3.5 & IP-VCD & 512 & 78.12 & 65.28 & -16.43 & -408.42 & -4328.34 \\
Qwen3.5 & IP-M3ID & 512 & 75.78 & 60.34 & -8.03 & -350.80 & -3778.54 \\
Gemma-3 & Direct & 512 & 98.44 & 85.38 & -65.29 & -782.48 & -7954.36 \\
Gemma-3 & Vanilla VCD & 512 & 95.12 & 90.12 & -74.00 & -824.59 & -8330.45 \\
Gemma-3 & Vanilla M3ID & 512 & 97.85 & 86.79 & -68.32 & -799.57 & -8112.07 \\
Gemma-3 & DoLa & 512 & 97.46 & 90.80 & -76.39 & -844.55 & -8526.19 \\
Gemma-3 & DeCo & 512 & 98.05 & 91.24 & -78.52 & -862.50 & -8702.34 \\
Gemma-3 & SID & 512 & 97.66 & 87.06 & -68.42 & -797.91 & -8092.83 \\
Gemma-3 & CDA & 512 & 75.39 & 93.70 & -64.59 & -688.61 & -6928.85 \\
Gemma-3 & IP-VCD & 512 & 80.66 & 85.54 & -54.16 & -646.54 & -6570.37 \\
Gemma-3 & IP-M3ID & 512 & 83.59 & 86.19 & -56.82 & -672.05 & -6824.39 \\
\bottomrule
\end{tabular}}
\caption{Fixed policies on VizWiz for nine models and the SID subset comparison (part 2 of 3).
All methods are listed for each model, grouped by the number of evaluated questions; metrics are multiplied by 100, and risk is undefined when there are no answers.}
\label{tab:vizreliability2}
\end{table*}

\addtocounter{table}{-1}
\begin{table*}[tp]
\centering
\footnotesize
\setlength{\tabcolsep}{3.6pt}
\resizebox{0.98\textwidth}{!}{%
\begin{tabular}{llrrrrrr}
\toprule
Model & Method & N & Coverage (\%) & Risk $\downarrow$ & $\Phi_1\uparrow$ & $\Phi_{10}\uparrow$ & $\Phi_{100}\uparrow$ \\
\midrule
OneVision & Direct & 512 & 97.66 & 68.66 & -26.23 & -537.75 & -5652.99 \\
OneVision & Vanilla VCD & 512 & 98.44 & 74.07 & -40.10 & -630.72 & -6536.97 \\
OneVision & Vanilla M3ID & 512 & 96.68 & 68.89 & -26.95 & -540.23 & -5673.05 \\
OneVision & DoLa & 512 & 96.68 & 82.00 & -56.23 & -718.93 & -7345.88 \\
OneVision & DeCo & 512 & 98.44 & 69.88 & -30.12 & -568.01 & -5946.91 \\
OneVision & SID & 512 & 97.27 & 72.89 & -36.33 & -600.59 & -6243.16 \\
OneVision & CDA & 512 & 92.38 & 71.40 & -30.02 & -538.03 & -5618.11 \\
OneVision & IP-VCD & 512 & 79.88 & 67.65 & -19.86 & -431.19 & -4544.47 \\
OneVision & IP-M3ID & 512 & 80.86 & 66.86 & -17.34 & -414.61 & -4387.27 \\
Phi-3.5-Vision & Direct & 512 & 86.91 & 74.83 & -36.33 & -560.16 & -5798.44 \\
Phi-3.5-Vision & Vanilla VCD & 512 & 91.80 & 80.74 & -50.29 & -662.01 & -6779.20 \\
Phi-3.5-Vision & Vanilla M3ID & 512 & 90.43 & 75.81 & -39.06 & -587.50 & -6071.88 \\
Phi-3.5-Vision & DoLa & 512 & 86.33 & 78.85 & -42.09 & -585.25 & -6016.89 \\
Phi-3.5-Vision & DeCo & 512 & 86.52 & 80.38 & -46.31 & -615.84 & -6311.15 \\
Phi-3.5-Vision & SID & 512 & 92.58 & 80.68 & -50.08 & -661.80 & -6778.98 \\
Phi-3.5-Vision & CDA & 512 & 72.27 & 75.51 & -30.35 & -462.77 & -4786.99 \\
Phi-3.5-Vision & IP-VCD & 512 & 68.95 & 74.31 & -28.18 & -441.27 & -4572.13 \\
Phi-3.5-Vision & IP-M3ID & 512 & 67.77 & 71.44 & -23.22 & -406.43 & -4238.46 \\
\bottomrule
\end{tabular}}
\caption{Fixed policies on VizWiz for nine models and the SID subset comparison (part 3 of 3).
All methods are listed for each model; metrics are multiplied by 100, and each method evaluates 512 questions.}
\label{tab:vizreliability3}
\end{table*}

Effective reliability follows the original definition of Whitehead et al.
On Food-101, $\Phi_c=(C-cW)/N$; on VizWiz, $\Phi_c=Q-cW_0/N$, where $W_0$ is the number of answered questions whose answer consensus score is zero.

These definitions follow \S\ref{sec:jointscoring} and Equations (3) to (6) of \citet{reliable_vqa}.

\subsection{Complementary Readings of Reliability and Decision Utility}
\label{app:reliable-reading}

These fixed-policy metrics can be computed directly from existing responses and scores.
Coverage measures how much is answered, risk measures the average error of the answered part, and effective reliability combines the credit of correct answers and the penalty of zero-score answers over all inputs.
Relative to vanilla VCD, both IP methods improve effective reliability under all three cost settings on all nine models in the full 512-question VizWiz evaluation; on Food-101, IP-VCD improves it on eight models and IP-M3ID on seven.

DU and effective reliability credit abstention differently, so the full ranking of methods can also differ.
For example, on Food-101 for MiniCPM-V2.6, IP-VCD has a DU of 47.61\%, higher than 45.09\% for IP-M3ID, while their $\Phi_1$ values are $-$3.34 and $-$1.32, higher for IP-M3ID.
On VizWiz for LLaVA-Mistral, the two have DU of 40.29\% and 40.00\% and $\Phi_1$ of $-$20.06 and $-$9.02, again with the order reversed.
The direction of improvement relative to vanilla VCD and the ranking between the two IP methods are judged separately.

An increase in effective reliability can come from fewer incorrect answers and does not require the conditional answering risk to fall on every model.
On VizWiz, IP-VCD lowers risk on eight models, while risk on InternVL3.5 rises by 1.76 percentage points; IP-M3ID also lowers risk on eight models, while risk on Qwen2.5-VL rises by 2.50 percentage points.
In these exceptions, coverage and the number of zero-score answers change together, and reliability over all inputs still improves.
The Food-101 result of LLaVA-Mistral gives a direct example: IP-VCD reduces zero-score answers from 1,497 to 635, and $\Phi_1$ moves from $-$23.51 to 2.19, showing how fewer incorrect answers contribute to this metric.

For binary correctness on Food-101, $\Phi_c=(C-cW)/N=\mathrm{coverage}\,[1-(1+c)\,\mathrm{risk}]$; on VizWiz, partially correct answers earn their soft consensus score, and the penalty applies only to zero-score answers.
DU gives 1 point to abstentions supported by the abstention labels, whereas effective reliability gives 0 to every abstention, so the two describe different credit and loss settings.
The all-abstain policy scores exactly 0 in effective reliability; at $c=10$ or 100, all fixed policies for Food-101 and VizWiz in Tables~\ref{tab:foodreliability} and~\ref{tab:vizreliability} still score below zero, showing that under high error costs the remaining incorrect answers still outweigh the credit of answering.
The relative improvement of IP holds alongside this absolute result.

\clearpage
\section{Case Studies}
\label{app:cases}

Each example shows the original image and the reference image, together with the responses and first-token distributions of Vanilla VCD, Guided VCD, Original-only, and IP-VCD.
Responses come from the formal evaluation, and probabilities come from a separate replay of the first step; for all four methods shown, the formal first token matches the highest-probability token in the replay.
The Original-only distribution is computed from the branch scores of the same replay, and $\times$ marks tokens excluded by candidate filtering.
The reference image is the perturbed input used in the replay; for multi-crop models, the first view is shown.
Ellipses mark continuous excerpts of long responses.

\subsection{Qwen2.5-VL}

\begin{figure*}[tp]
\centering
\includegraphics[width=160mm]{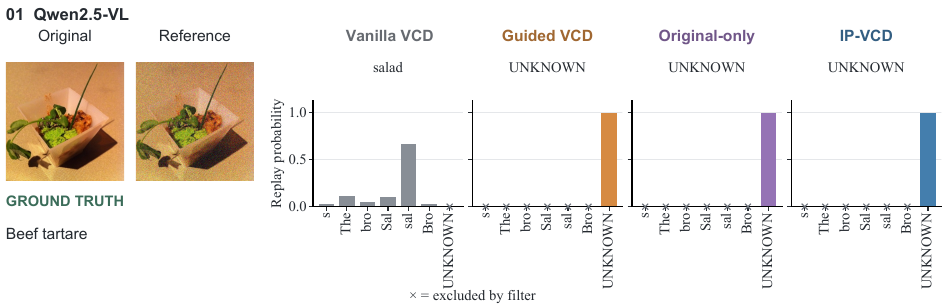}
\caption{Qwen2.5-VL, beef tartare.}
\label{fig:case01}
\end{figure*}

\begin{figure*}[tp]
\centering
\includegraphics[width=160mm]{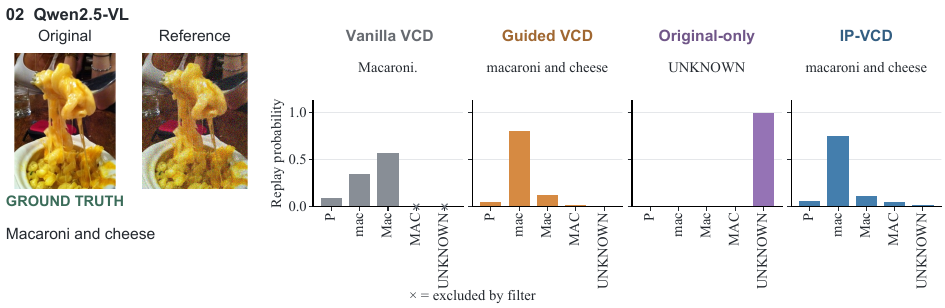}
\caption{Qwen2.5-VL, macaroni and cheese.}
\label{fig:case02}
\end{figure*}

\begin{figure*}[tp]
\centering
\includegraphics[width=160mm]{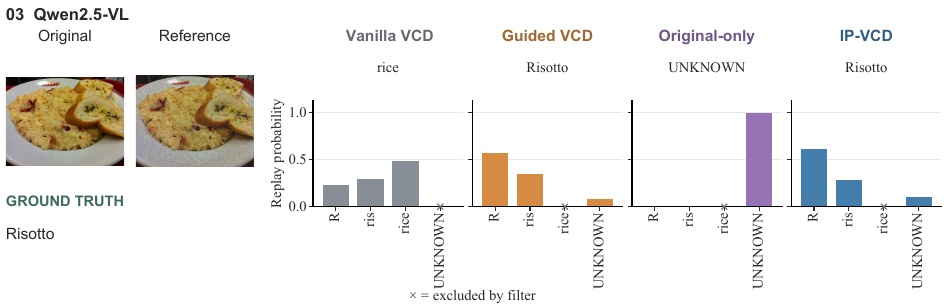}
\caption{Qwen2.5-VL, risotto.}
\label{fig:case03}
\end{figure*}

\subsection{LLaVA-Mistral}

\begin{figure*}[tp]
\centering
\includegraphics[width=160mm]{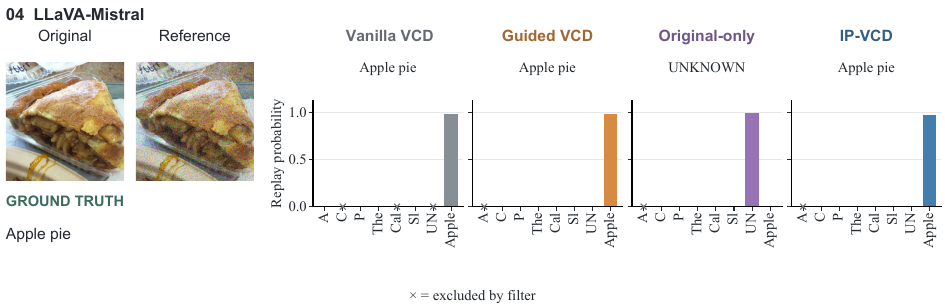}
\caption{LLaVA-Mistral, apple pie.}
\label{fig:case04}
\end{figure*}

\begin{figure*}[tp]
\centering
\includegraphics[width=160mm]{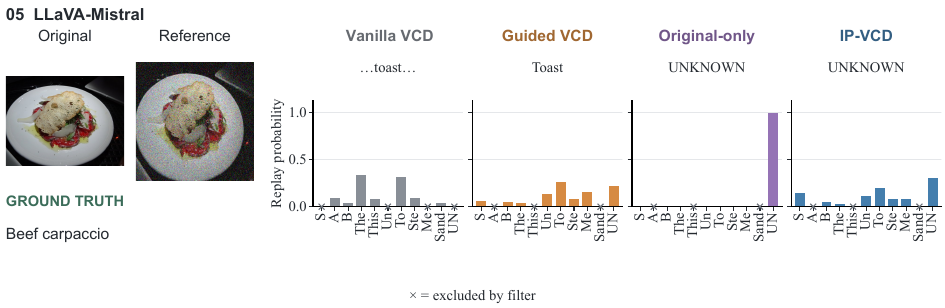}
\caption{LLaVA-Mistral, beef carpaccio.}
\label{fig:case05}
\end{figure*}

\begin{figure*}[tp]
\centering
\includegraphics[width=160mm]{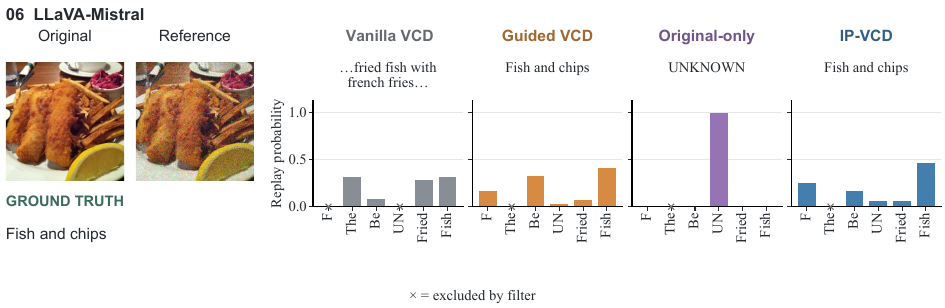}
\caption{LLaVA-Mistral, fish and chips.}
\label{fig:case06}
\end{figure*}

\begin{figure*}[tp]
\centering
\includegraphics[width=160mm]{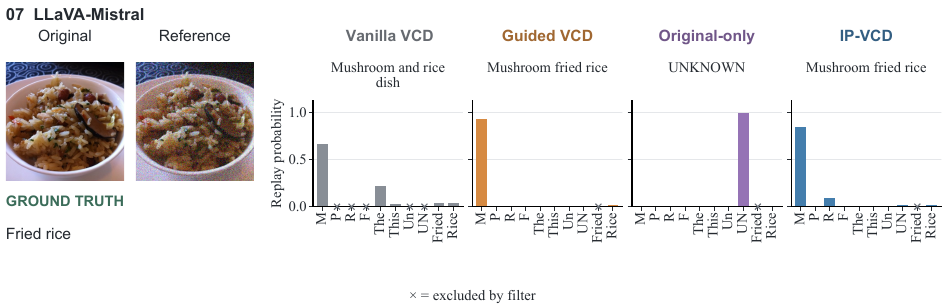}
\caption{LLaVA-Mistral, fried rice.}
\label{fig:case07}
\end{figure*}

\subsection{MiniCPM-V2.6}

\begin{figure*}[tp]
\centering
\includegraphics[width=160mm]{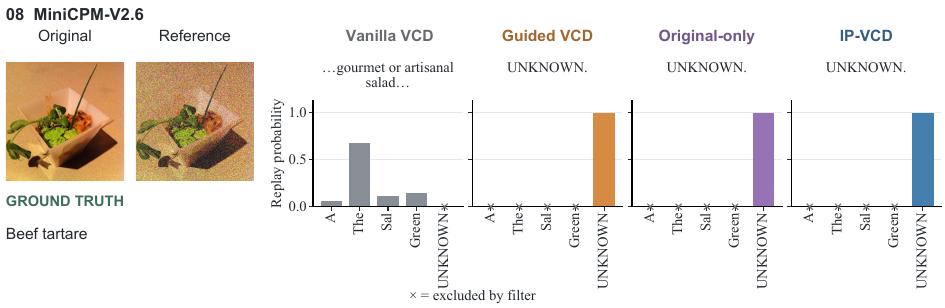}
\caption{MiniCPM-V2.6, beef tartare.}
\label{fig:case08}
\end{figure*}

\begin{figure*}[tp]
\centering
\includegraphics[width=160mm]{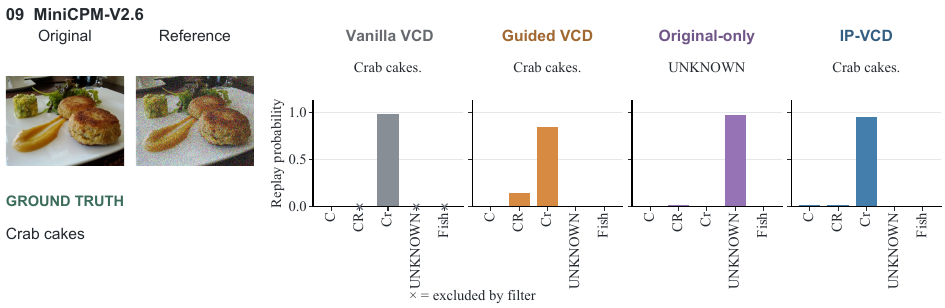}
\caption{MiniCPM-V2.6, crab cakes.}
\label{fig:case09}
\end{figure*}

\begin{figure*}[tp]
\centering
\includegraphics[width=160mm]{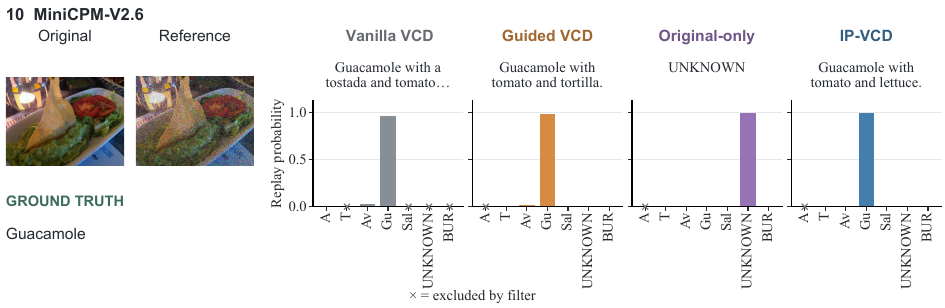}
\caption{MiniCPM-V2.6, guacamole.}
\label{fig:case10}
\end{figure*}

\subsection{InternVL3.5}

\begin{figure*}[tp]
\centering
\includegraphics[width=160mm]{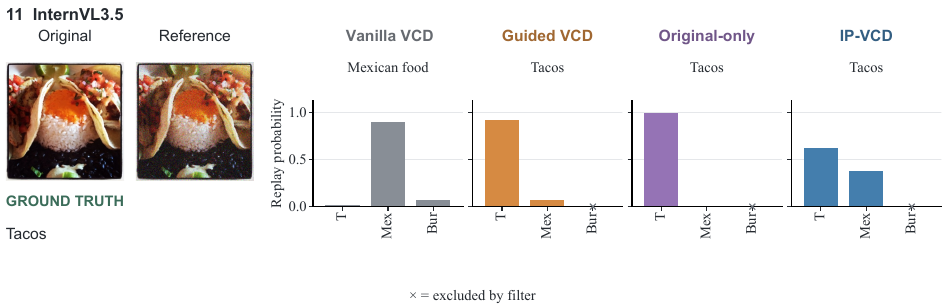}
\caption{InternVL3.5, tacos.}
\label{fig:case11}
\end{figure*}

\subsection{Qwen3-VL}

\begin{figure*}[tp]
\centering
\includegraphics[width=160mm]{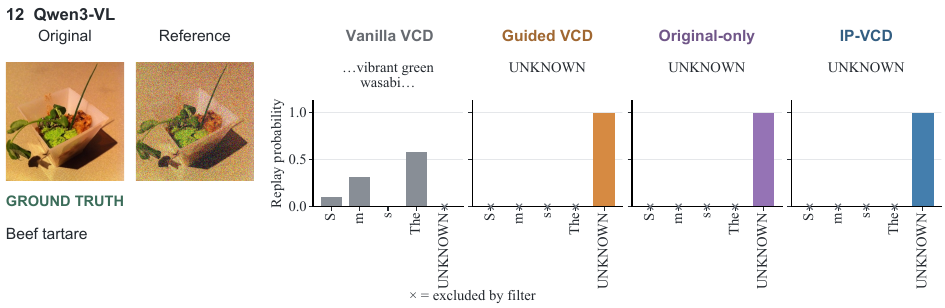}
\caption{Qwen3-VL, beef tartare.}
\label{fig:case12}
\end{figure*}

\begin{figure*}[tp]
\centering
\includegraphics[width=160mm]{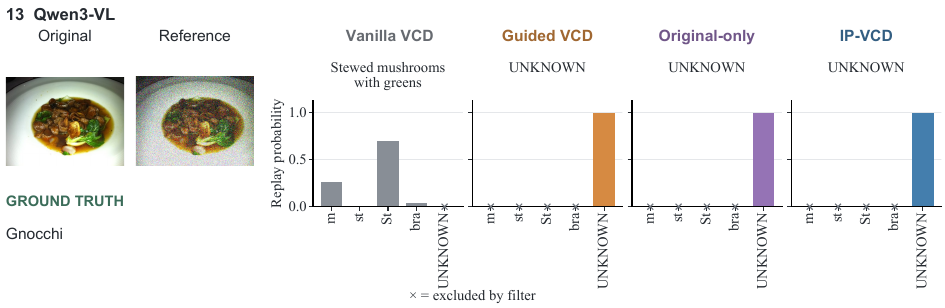}
\caption{Qwen3-VL, gnocchi.}
\label{fig:case13}
\end{figure*}

\subsection{Qwen3.5}

\begin{figure*}[tp]
\centering
\includegraphics[width=160mm]{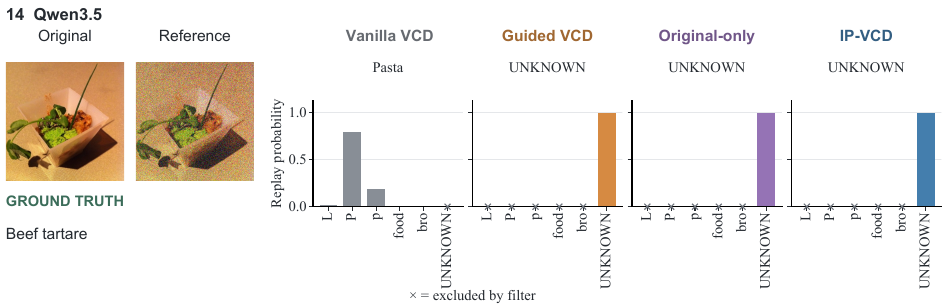}
\caption{Qwen3.5, beef tartare.}
\label{fig:case14}
\end{figure*}

\begin{figure*}[tp]
\centering
\includegraphics[width=160mm]{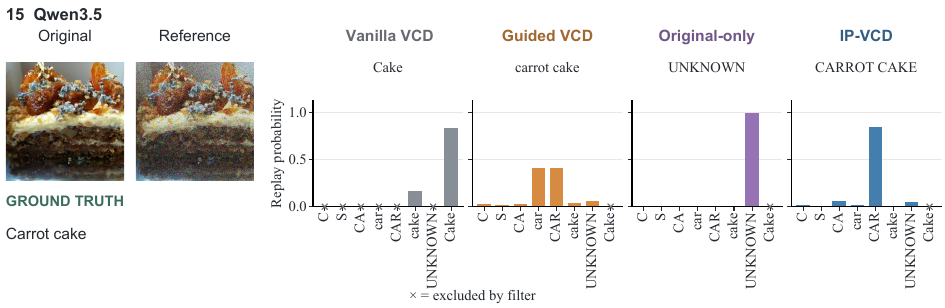}
\caption{Qwen3.5, carrot cake.}
\label{fig:case15}
\end{figure*}

\begin{figure*}[tp]
\centering
\includegraphics[width=160mm]{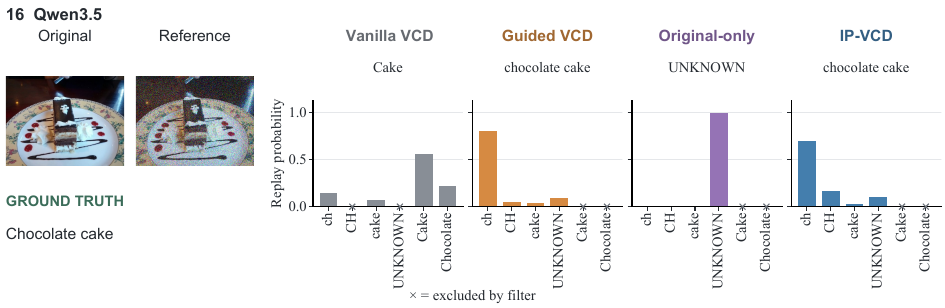}
\caption{Qwen3.5, chocolate cake.}
\label{fig:case16}
\end{figure*}

\subsection{Gemma-3}

\begin{figure*}[tp]
\centering
\includegraphics[width=160mm]{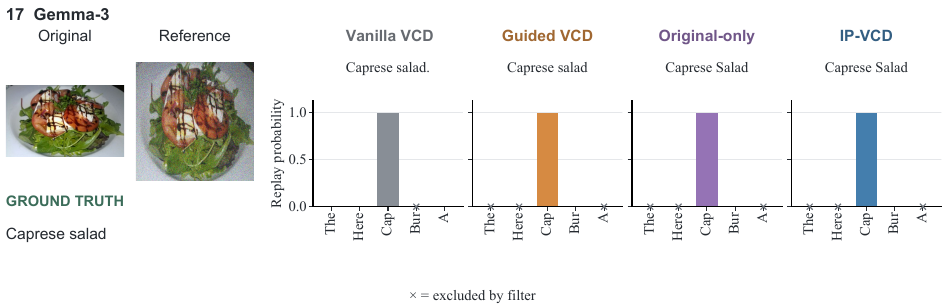}
\caption{Gemma-3, caprese salad.}
\label{fig:case17}
\end{figure*}

\begin{figure*}[tp]
\centering
\includegraphics[width=160mm]{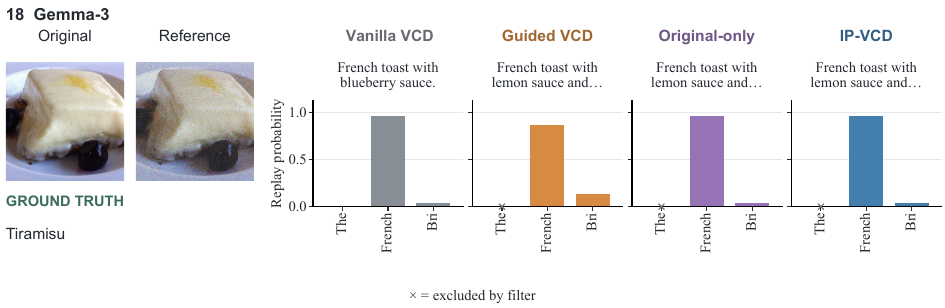}
\caption{Gemma-3, tiramisu.}
\label{fig:case18}
\end{figure*}

\subsection{OneVision}

\begin{figure*}[tp]
\centering
\includegraphics[width=160mm]{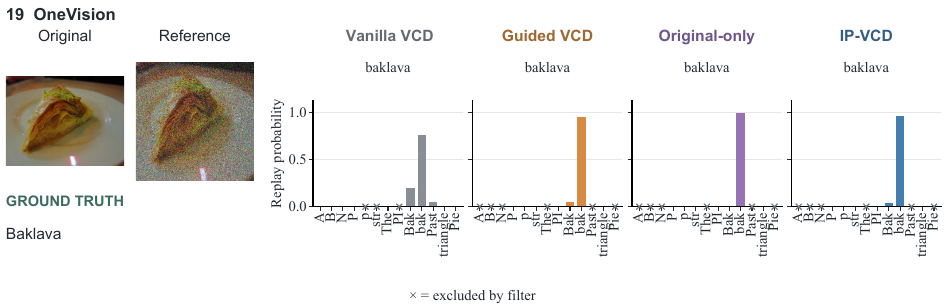}
\caption{OneVision, baklava.}
\label{fig:case19}
\end{figure*}

\begin{figure*}[tp]
\centering
\includegraphics[width=160mm]{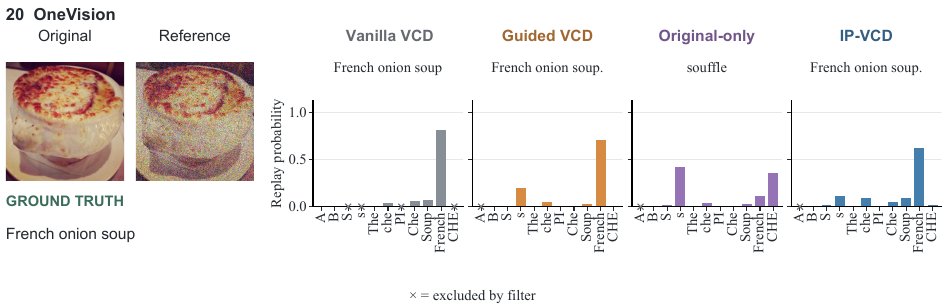}
\caption{OneVision, french onion soup.}
\label{fig:case20}
\end{figure*}

\subsection{Phi-3.5-Vision}

\begin{figure*}[tp]
\centering
\includegraphics[width=160mm]{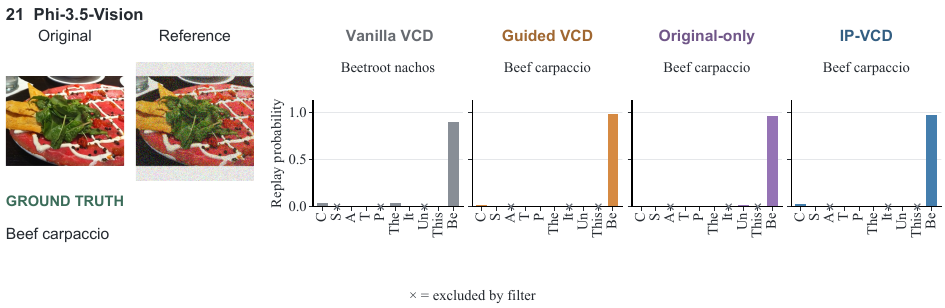}
\caption{Phi-3.5-Vision, beef carpaccio.}
\label{fig:case21}
\end{figure*}

\clearpage

\end{document}